\documentclass[nohyperref]{article}

\usepackage{microtype}
\usepackage{graphicx}
\usepackage{subfigure}
\usepackage{booktabs} 
\usepackage{bm}
\usepackage{hyperref}

\usepackage[accepted]{icml2022}

\usepackage{amsmath}
\usepackage{amssymb}
\usepackage{mathtools}
\usepackage{amsthm}

\usepackage[capitalize,noabbrev]{cleveref}

\theoremstyle{plain}
\newtheorem{theorem}{Theorem}[section]
\newtheorem{prop}[theorem]{Proposition}

\theoremstyle{definition}

\newtheorem{assumption}[theorem]{Assumption}
\theoremstyle{remark}
\newtheorem{remark}[theorem]{Remark}

\usepackage[textsize=tiny]{todonotes}
\pdfoutput=1

\icmltitlerunning{An Asymptotic Test for Conditional Independence using Analytic Kernel Embeddings}

\begin{document}

\twocolumn[
\icmltitle{An Asymptotic Test for Conditional Independence \\
using Analytic Kernel Embeddings}



\icmlsetsymbol{equal}{*}

\begin{icmlauthorlist}
\icmlauthor{Meyer Scetbon}{equal,crest}
\icmlauthor{Laurent Meunier}{equal,fb,dauph}
\icmlauthor{Yaniv Romano}{tech}
\end{icmlauthorlist}

\icmlaffiliation{crest}{CREST, ENSAE, France}
\icmlaffiliation{fb}{Facebook AI Research, Paris, France}
\icmlaffiliation{dauph}{Université Paris-Dauphine, France}
\icmlaffiliation{tech}{Departments of Electrical and Computer Engineering and of Computer Science, Technion, Israel}

\icmlcorrespondingauthor{Meyer Scetbon}{meyer.scetbon@ensae.fr}

\icmlkeywords{Machine Learning, ICML}

\vskip 0.3in
]



\printAffiliationsAndNotice{\icmlEqualContribution} 

\begin{abstract}
We propose a new conditional dependence measure and a statistical test for conditional independence. The measure is based on the difference between analytic kernel embeddings of two well-suited distributions evaluated at a finite set of locations. We obtain its asymptotic distribution under the null hypothesis of conditional independence and design a consistent statistical test from it. We conduct a series of experiments showing that our new test outperforms state-of-the-art methods both in terms of type-I and type-II errors even in the high dimensional setting.
\end{abstract}

\section{Introduction}
We consider the problem of testing whether two variables $X$ and $Y$ are independent given a set of confounding variables $Z$, which can be formulated as a hypothesis testing problem of the form:
\begin{align*}
    H_0:~ X\perp Y | Z\qquad\text{vs.}\qquad H_1:X\not\perp Y | Z.
\end{align*}
Testing for conditional independence (CI) is central in a wide variety of statistical learning problems. For example, it is at the core of graphical modeling~\citep{lauritzen1996graphical,koller2009probabilistic}, causal discovery~\citep{pearl2009causal,glymour2019review}, variable selection \citep{candes2018panning}, dimensionality reduction~\citep{li2018sufficient}, and biomedical studies~\citep{richardson1993bayesian,dobra2004sparse,markowetz2007inferring}. 
 
Testing for $H_0$ in such applications is known to be a highly challenging task~\citep{gcm2020,neykov2021minimax}. A large line of work has focused on the design of measures for conditional dependence based for example on kernel methods~\cite{fukumizu2008kernel,Sheng2019OnDA,NEURIPS2020_f340f1b1,huang2020kernel} and rank statistics~\cite{azadkia2021simple,shi2021azadkia}. Testing for conditional independence is even more difficult as it requires both designing a test statistic which measures the conditional dependencies and controlling its quantiles. Indeed, existing tests may fail to control the type-I error, especially when the confounding set of variables is high-dimensional with a complex dependency structure~\cite{bergsma2004testing}. Furthermore, even if the test is valid, the availability of limited data makes the problem of discriminating between the null and alternative hypotheses extremely difficult, resulting in a test of low power. These challenges has motivated the development of a series of practical methods attempting to reliably test for conditional independence. These include tests based on kernels ~\citep{zhang2012kernel,doran2014permutation,strobl2019approximate,Zhang2017FeaturetoFeatureRF}, ranks~\cite{runge2018conditional,mittag2018nonparametric}, models~\citep{sen2017modelpowered,sen2018mimic,chalupka2018fast,gcm2020}, permutations and  samplings~\citep{berrett2020conditional,candes2018panning,BellotS19, shi2021double,javanmard2021pearson},
and optimal transport~\cite{warren2021wasserstein}.

Another line of work aims at building statistical tests for different problems by computing difference of analytic kernel embeddings evaluated at a finite set of locations. Two main strategies are adopted in the literature: either the locations are chosen randomly or are learned in order to maximize the power of the test. In~\cite{epps1986omnibus,chwialkowski2015fast}, the authors propose two-sample tests where locations are chosen randomly. In~\cite{zhang2018large}, they adopt a similar method for independence testing. In~\cite{jitkrittum2016interpretable,NEURIPS2019_0e2db0cb}, the authors propose two-sample tests where the location are leaned instead.~\citet{jitkrittum2017adaptive} learned the location for independence testing and~\cite{jitkrittum2017linear} learned them also to test for goodness-of-fit.

In this paper, we propose a new kernel-based test for conditional independence with asymptotic theoretical guarantees.
Taking inspiration from~\cite{chwialkowski2015fast,jitkrittum2017adaptive,NEURIPS2019_0e2db0cb}, we use the $\ell^p$ distance between two well-chosen analytic kernel mean embeddings evaluated at a finite set of locations. To the best of our knowledge, it is the first time that this strategy is employed for conditional independence testing. We show that this measure encodes the conditional dependence relation of the random
variables under study. 
Under common assumptions on the richness of the RKHS, we derive the asymptotic null distribution of our measure, and design a simple nonparametric test that is distribution-free under the null hypothesis. Furthermore, we show that our test is consistent. Lastly, we validate our theoretical claims and study the performance of the proposed approach using simulated conditionally (in)dependent data and show that our testing procedure outperforms state-of-the-art methods.

\subsection{Related Work}

\citet{zhang2012kernel} propose a kernel based-test (KCIT), by leveraging the characterization of conditional independence derived in \citep{daudin1980partial} to form a test statistic. The authors of this work obtain the asymptotic null distribution of the proposed statistic and derived a practical procedure from it to test for $H_0$. However, one main practical issue of the proposed test is that the asymptotic null distribution of their statistic cannot be computed directly as it involved unknown quantities. To address this problem, the authors propose to approximate it either with Monte Carlo simulations or by fitting a Gamma distribution. In our work, we propose a new kernel-based statistic to test for conditional independence and show that its asymptotic null distribution is simply the standard normal distribution. In addition \citet{zhang2012kernel} extended the Gaussian process (GP) regression framework to the multi-output case, which allowed them to find the hyperparameters involved in the test statistic, maximizing the marginal likelihood. We also deploy a similar optimization procedure to that of \citet{zhang2012kernel}, however, in our case the output of the GP regression is univariate and therefore more computationally efficient. Note also that in~\cite{strobl2019approximate}, the authors propose a relaxed version of KCIT which approximates it using random Fourier features and offer a new method to deal with the tradeoff between the computational cost and the power of the test.

\citet{doran2014permutation} propose an MMD-based test for conditional independence using a well chosen permutation matrix. The role of this permutation is to simulate samples from the factorized distribution. Once such permutation is obtained, the authors propose to apply an MMD-based two-sample test~\cite{gretton2012kernel} to detect conditional dependencies between the simulated distribution and the joint one. However the test proposed there can only be applied for small sample sizes as it requires to solve a linear program using the simplex algorithm to compute the permutation matrix. Note also that the authors do not have access directly to the quantiles of the asymptotic null distribution and therefore a bootstrap procedure is required to compute them. In addition, the consistency of their test holds only under some non-trivial conditions on the permutation matrix obtained. In contrast, our test can be applied for large sample sizes, admits a simple asymptotic null distributions from which the quantiles can be directly obtained and is consistent under some mild assumptions on the distributions.


Other CI tests proposed in the literature suggest testing relaxed forms of conditional independence. For instance,~\citet{gcm2020} propose the generalised covariance measure (GCM) which only characterises weak conditional dependence~\cite{daudin1980partial} and \citet{Zhang2017FeaturetoFeatureRF} propose a kernel-based test which focuses only on individual effects of the conditioning variable $Z$ on $X$ and $Y$. Some other tests are based on the knowledge of the conditional distributions in order to measure conditional dependencies. For example \citet{candes2018panning} assume that one has access to the exact conditional distributions,~\citet{BellotS19,shi2021double} approximate them using generative models and \citet{sen2017modelpowered} consider model-based methods to generate samples from the conditional distributions. In our work, we design a test statistic which characterizes the exact conditional independence of random variables and obtain its asymptotic null distribution without assuming any knowledge on the conditional distributions. Under some mild assumptions on the RKHSs considered, we also derive an approximate test statistic which admits the same asymptotic distribution and obtain a simple testing procedure from it.

\section{Background and Notations}
We first recall some notions on kernels and mean embeddings which will be useful in the derivation of our conditional independence test.
Let $(\mathcal{D},\mathcal{A})$ be a Borel measurable space and denote $\mathcal{M}_{1}^{+}(\mathcal{D})$ the space of Borel probability measures on $\mathcal{D}$. Let also $(H,k)$ be a measurable RKHS  on $\mathcal{D}$, i.e. a functional Hilbert space satisfying the reproducing property: for all $f\in H$, $x\in\mathcal{D}$, $f(x)=\langle f,k_x\rangle_H$. Let $\nu\in\mathcal{M}_{1}^{+}(\mathcal{D})$. If $\mathbb{E}_{x\sim\nu}[\sqrt{k(x,x)}]$ is finite, we define for all $t\in\mathcal{D}$ the \emph{mean embedding} as $\mu_{\nu,k}(t):=\int_{x\in\mathcal{D}}k(x,t)d\nu(x)$.
Note that $\mu_{\nu,k}$ is the unique element in $H$ satisfying for all $f\in H$,  $\mathbb{E}_{x\sim\nu}(f(x))=\langle \mu_{\nu,k},f\rangle_{H}$. If  $\nu\mapsto\mu_{\nu,k}$ is injective, then  the kernel $k$ is said to be \emph{characteristic}. This property is essential for the separation property to be verified when defining a kernel metric between distributions, such as the MMD~\citep{gretton2012kernel}, or the $\ell^p$ distance~\citep{NEURIPS2019_0e2db0cb}. 

\textbf{$\ell^p$-distance between mean embeddings.} Let $k$ be a definite positive, characteristic, continuous, and bounded kernel on $\mathbb{R}^d$ and $p\geq 1$ an integer. \citet{NEURIPS2019_0e2db0cb} showed that given an absolutely continuous Borel probability measure $\Gamma$ on $\mathbb{R}^d$, the following function defined for any  $(P,Q)\in\mathcal{M}_1^{+}(\mathbb{R}^d)\times\mathcal{M}_1^{+}(\mathbb{R}^d)$ as 
\begin{align}
\label{eq-dlp}
    d_{p}(P,Q):=\left[\int_{\mathbb{R}^d}|\mu_{P,k}(\mathbf{t})-\mu_{Q,k}(\mathbf{t})|^p d\Gamma(\mathbf{t})\right]^{\frac{1}{p}}
\end{align}
is a metric on $\mathcal{M}_1^{+}(\mathbb{R}^d)$.
When the kernel $k$ is analytic\footnote{An \emph{analytic kernel} on $\mathbb{R}^d$ is a positive definite kernel such that for all $x\in\mathbb{R}^d$, $k(x,\cdot)$ is an analytic function, i.e., a function defined locally by a convergent power series.}, \citet{NEURIPS2019_0e2db0cb} also showed that for any $J\geq 1$, 
\begin{align}
\label{eq-dlp_J}
    d_{p,J}(P,Q):=\left[\frac{1}{J}\sum_{j=1}^J |\mu_{P,k}(\mathbf{t}_j)-\mu_{Q,k}(\mathbf{t}_j)|^p\right]^{\frac{1}{p}},
\end{align}
where $(\mathbf{t}_j)_{j=1}^J$ are sampled independently from the $\Gamma$ distribution, is a random metric\footnote{A random metric is a random process which satisfies all the conditions for a metric almost-surely.} on $\mathcal{M}_1^{+}(\mathbb{R}^d)$.

In what follows, we consider distributions on Euclidean spaces. More precisely, let $d_x,d_y,d_z\geq 1$, $\mathcal{X}:=\mathbb{R}^{d_x}$, $\mathcal{Y}:=\mathbb{R}^{d_y}$, and $\mathcal{Z}:=\mathbb{R}^{d_z}$. Let $(X,Z,Y)$ be a random vector on $\mathcal{X}\times\mathcal{Z}\times\mathcal{Y}$ with law $P_{XZY}$. We denote by $P_{XY}$, $P_X$, and $P_Y$ the law of $(X,Y)$, $X$, and $Y$, respectively. We also denote by $\mathcal{\ddot{X}}:=\mathcal{X}\times\mathcal{Z}$, $\ddot{X}:=(X,Z)$, and $P_{\ddot{X}}$ its law. Let $P_X\otimes P_Y$ be the product of the two measures $P_X$ and $P_Y$. Given $(H_{\mathcal{\ddot{X}}},k_\mathcal{\ddot{X}})$ and $(H_{\mathcal{Y}},k_{\mathcal{Y}})$, two measurable reproducing kernel Hilbert spaces (RKHS) on $\mathcal{\ddot{X}}$ and $\mathcal{Y}$, respectively, we define the tensor-product RKHS $H=H_{\mathcal{\ddot{X}}}\otimes H_\mathcal{Y}$ associated with its \emph{tensor-product kernel} $k=k_{\mathcal{\ddot{X}}}\otimes k_{\mathcal{Y}}$, defined for all $\ddot{x},\ddot{x}'\in\mathcal{\ddot{X}}$ and $y,y'\in\mathcal{Y}$, as $k((\ddot{x},y),(\ddot{x}',y'))= k_{\mathcal{\ddot{X}}}(\ddot{x},\ddot{x}')\times k_{\mathcal{Y}}(y,y').$

\section{A new $\ell^p$ kernel-based testing procedure}
In this section, we present our statistical procedure to test for conditional independence. We begin by introducing a general measure based on the $\ell^p$ distance $d_p$ between mean embeddings which characterizes the conditional independence. We derive an oracle test statistic for which we obtain its asymptotic distribution under both the null and alternative hypothesis. Then, we provide an efficient procedure to effectively compute an approximation of our oracle statistic and show that it has the exact same asymptotic distribution. To avoid any bootstrap or permutation procedures, we offer a normalized version of our statistic and derive a simple and consistent test from it.


\subsection{Conditional Independence Criterion}
Let us first introduce the criterion we use to define our statistical test. We define a probability measure $P_{\ddot{X}\otimes Y|Z}$ on $\mathcal{\ddot{X}}\times\mathcal{Y}$ as
\begin{align*}
P_{\ddot{X}\otimes Y|Z}(A\times B):=\mathbb{E}_{Z}\left[\mathbb{E}_{\ddot{X}}[\mathbf{1}_A|Z]\mathbb{E}_{Y}[\mathbf{1}_B|Z]\right],
\end{align*}
for any $(A,B)\in\mathcal{B}(\mathcal{\ddot{X}})\times\mathcal{B}(\mathcal{Y})$, where $\mathbf{1}_A$ is the characteristic function of a measurable set $A$ and similarly for $B$. One can now characterize the independence of $X$ and $Y$ given $Z$ as follows:  $X\perp Y | Z$ if and only if $P_{XZY}=P_{\ddot{X}\otimes Y|Z}$~\citep[Theorem 8]{fukumizu2004dimensionality}. Therefore, we have a first simple characterization of the conditional independence: $X\perp Y | Z$ if and only if $d_{p}(P_{XZY},P_{\ddot{X}\otimes Y|Z})=0$. With this in place, we now state some assumptions on the kernel $k$ considered in the rest of this paper.

\begin{assumption}
\label{assump-kernel}
 The kernel $k : (\mathcal{\ddot{X}}\times \mathcal{Y})\times (\mathcal{\ddot{X}}\times \mathcal{Y}) \rightarrow \mathbb{R}$ is positive definite, characteristic, bounded, continuous and analytic. Moreover, the kernel $k$ is a tensor product of kernels $k_{\ddot{\mathcal{X}}}$ and $k_\mathcal{Y}$ on $\mathcal{\ddot{X}}$ and $\mathcal{Y}$, respectively.
\end{assumption}
It is worth noting that a sufficient condition for the kernel $k$ to be characteristic, bounded, continuous and analytic, is that both kernels $k_{\mathcal{\ddot{X}}}$ and $k_{\mathcal{Y}}$ are characteristic, bounded, continuous and analytic~\citep{szabo2018characteristic}. For example, if the kernels $k_\mathcal{\ddot{X}}$ and $k_\mathcal{Y}$ are Gaussian kernels\footnote{A gaussian kernel $K$ on $\mathcal{W}\subset\mathbb{R}^d$ satisfies for all $w,w'\in\mathcal{W}$, $K(w,w'):=\exp\left(-\frac{\Vert w - w'\Vert_2^2}{2\sigma^2}\right)$ for some $\sigma>0$.} on $\mathcal{\ddot{X}}$ and $\mathcal{Y}$ respectively, then $k=k_\mathcal{\ddot{X}}\otimes k_\mathcal{Y}$ satisfies Assumption~\ref{assump-kernel}~\citep{jitkrittum2017adaptive}. Using the analyticity of the kernel $k$, one can work with $d_{p,J}$ defined in~\eqref{eq-dlp_J} instead of $d_{p}$ to characterize the conditional independence.
\begin{prop}
Let $p\geq 1$, $J\geq 1$, $k$ be a kernel satisfying Assumption~\ref{assump-kernel}, $\Gamma$ an absolutely continuous Borel probability measure on $\mathcal{\ddot{X}}\times\mathcal{Y}$, and $\{(\mathbf{t}^{(1)}_j,t^{(2)}_j)\}_{j=1}^J$ sampled independently from $\Gamma$. Then $\Gamma$-almost surely, $d_{p,J}(P_{XZY},P_{\ddot{X}\otimes Y|Z})=0$ if and only if $X\perp Y | Z$.
\end{prop}
\begin{proof}
Recall that $X\perp Y | Z$ if and only if $P_{XZY}=P_{\ddot{X}\otimes Y|Z}$~\citep{fukumizu2008kernel}. If $k$ is bounded, characteristic, and analytic, then, by invoking~\citep[Theorem 2.1]{NEURIPS2019_0e2db0cb} we get that $d_{p,J}^p$ is a random metric on the space of Borel probability measures. This concludes the proof.
\end{proof}
The key advantage of using $d_{p,J}(P_{XZY},P_{\ddot{X}\otimes Y|Z})$ to measure the conditional dependence is that it only requires to compute the differences between the mean embeddings of $P_{XZY}$ and $P_{\ddot{X}\otimes Y|Z}$ at $J$ locations. In what follows, we derive from it a first oracle test statistic for conditional independence.

\subsection{A First Oracle Test Statistic}
When the kernel $k$ considered satisfies Assumption~\ref{assump-kernel}, we can obtain a simple expression of our measure $d_{p,J}(P_{XZY},P_{\ddot{X}\otimes Y|Z})$. Indeed, the tensor formulation of the kernel $k$ allows us to write the mean embedding of $P_{\ddot{X}\otimes Y|Z}$ for any $(\mathbf{t}^{(1)},t^{(2)})\in\mathcal{\ddot{X}}\times\mathcal{Y}$ as:
\begin{equation}
\begin{aligned}
\label{eq-ME-tensor-cond}
    \mu&_{P_{\ddot{X}\otimes Y|Z},k_{\mathcal{\ddot{X}}}\cdot k_{\mathcal{Y}}}(\mathbf{t}^{(1)},t^{(2)})=\\
    &\mathbb{E}_{Z}\left[\mathbb{E}_{\ddot{X}}\left[k_{\mathcal{\ddot{X}}}(\mathbf{t}^{(1)},\ddot{X})|Z\right]
    \mathbb{E}_{Y}\left[k_{\mathcal{Y}}(t^{(2)},Y)|Z\right] \right]\; .
\end{aligned}
\end{equation}

Then, by defining the witness function as
\begin{align*}
    \Delta(\mathbf{t}^{(1)},t^{(2)}) :=&\mathbb{E}\left[\left(k_{\mathcal{\ddot{X}}}(\mathbf{t}^{(1)},\ddot{X})- \mathbb{E}_{\ddot{X}}\left[k_{\mathcal{\ddot{X}}}(\mathbf{t}^{(1)},\ddot{X})|Z\right]\right)\right. \\
    &\times\left.\left(k_{\mathcal{Y}}(t^{(2)},Y)- \mathbb{E}_{Y}\left[k_{\mathcal{Y}}(t^{(2)},Y)|Z\right]\right)\right],
\end{align*}

and by considering $\{(\mathbf{t}^{(1)}_j,t^{(2)}_j)\}_{j=1}^J$ sampled independently according to $\Gamma$, we get that (see Appendix~\ref{form-witness} for more details)
$$ d_{p,J} (P_{XZY},P_{\ddot{X}\otimes Y|Z})=\left[\frac{1}{J}\sum_{j=1}^J \left|\Delta(\mathbf{t}^{(1)}_j,t^{(2)}_j)\right|^p\right]^{1/p}.$$

\textbf{Estimation.} Given $n$ observations $\{(x_i,z_i,y_i)\}_{i=1}^n$ that are drawn independently from $P_{XZY}$, we aim at obtaining an estimator of $ d_{p,J}^p (P_{XZY},P_{\ddot{X}\otimes Y|Z})$.
To do so, we introduce the following estimate of $\Delta(\mathbf{t}^{(1)},t^{(2)})$, defined as
\begin{align*}
 \Delta_{n}(\mathbf{t}^{(1)},t^{(2)})=&\frac{1}{n}\sum_{i=1}^n  \left(k_{\mathcal{\ddot{X}}}(\mathbf{t}^{(1)},\ddot{x}_i)- \mathbb{E}_{\ddot{X}}\left[k_{\mathcal{\ddot{X}}}(\mathbf{t}^{(1)},\ddot{X})|z_i\right]\right)\\
\times&\left(k_{\mathcal{Y}}(t^{(2)},y_i)- \mathbb{E}_{Y}\left[k_{\mathcal{Y}}(t^{(2)},Y)|z_i\right]\right).
\end{align*} 
With this in place, a natural candidate to estimate $d_{p,J}^p (P_{XZY},P_{\ddot{X}\otimes Y|Z})$ (up to the constant $J$) can be expressed as
\begin{align*}
\text{CI}_{n,p}&:=\sum_{j=1}^J \left|  \Delta_{n}(\mathbf{t}^{(1)}_j,t^{(2)}_j)\right|^p,
\end{align*}
where $(\mathbf{t}^{(1)}_1,t^{(2)}_1),\dots,(\mathbf{t}^{(1)}_J,t^{(2)}_J)\in\mathcal{\ddot{X}}\times\mathcal{Y}$ are sampled independently from $\Gamma$.

We now turn to derive the asymptotic distribution of this statistic. For that purpose, define, for all $j\in\{1,\dots,J\}$ and $i\in\{1,\dots,n\}$,
\begin{align*}
    u_i(j):=&\left(k_{\mathcal{\ddot{X}}}(\mathbf{t}^{(1)}_j,\ddot{x}_i)- \mathbb{E}_{\ddot{X}}\left[k_{\mathcal{\ddot{X}}}(\mathbf{t}^{(1)}_j,\ddot{X})|Z=z_i\right]\right) \\&\times\left(k_{\mathcal{Y}}(t^{(2)}_j,y_i)- \mathbb{E}_{Y}\left[k_{\mathcal{Y}}(t^{(2)}_j,Y)|Z=z_i\right]\right),
\end{align*}
$\mathbf{u}_i:=(u_i(1),\dots,u_i(J))^T$  and $\bm{\Sigma}:=\mathbb{E}(\mathbf{u}_1\mathbf{u}_1^T)$. We also denote by $\mathbf{S}_{n}:=\frac{1}{n}\sum_{i=1}^n \mathbf{u}_{i}$. Observe that $\text{CI}_{n,p}=\lVert\mathbf{S}_{n}\rVert_p^p$. In the following proposition we obtain the asymptotic distribution of our statistic $\text{CI}_{n,p}$.
\begin{prop}
\label{prop:oracle-law}
Suppose that Assumption~\ref{assump-kernel} is verified. Let $p\geq 1$, $J\geq 1$ and $((\mathbf{t}^{(1)}_1,t^{(2)}_1),\dots,(\mathbf{t}^{(1)}_J,t^{(2)}_J))\in(\mathcal{\ddot{X}}\times\mathcal{Y})$. Then, under $H_0$, we have: $\sqrt{n}\mathbf{S}_{n}\rightarrow \mathcal{N}(0,\bm{\Sigma})$. Moreover, under $H_1$, if $((\mathbf{t}^{(1)}_j,t^{(2)}_j))_{j=1}^J$ are sampled independently according to $\Gamma$, then $\Gamma$-almost surely, for any $q\in\mathbb{R}$, $\lim_{n\rightarrow\infty}P( n^{p/2}\emph{CI}_{n,p} \geq q)=1$.
\end{prop}

\begin{proof}
\vspace{-0.2cm}
Recall that $\mathbf{S}_n = \frac1n\sum_{i=1}^n\mathbf{u}_i$ where $\mathbf{u}_i$ are i.i.d. samples.  Under $H_0$, $\mathbb{E}\left[\mathbf{u}_i\right]=0$. Using the Central Limit Theorem, we get: $\sqrt{n}\mathbf{S}_n\to\mathcal{N}(0,\bm{\Sigma})$. Using the analyticity of the kernel $k$, under $H_1$, $\Gamma$-almost surely, there exists a $j\in\{1,\dots,J\}$ such that $\mathbb{E}\left[u_1(j)\right]\neq 0$. Therefore, we can deduce that $\Gamma$-almost surely, $\mathbf{S}:=\mathbb{E}\left[\mathbf{u_1}\right]\neq 0$. Now, for all $q>0$, we get: $P(n^{p/2}\text{CI}_{n,p}>q)\to 1$ because $\text{CI}_{n,p}\to\lVert\mathbf{S}\rVert_p^p $  when $n\to\infty$.
\vspace{-0.2cm}
\end{proof}

From the above proposition, we can define a consistent statistical test at level $0<\alpha<1$, by rejecting the null hypothesis if $n^{p/2}\text{CI}_{n,p}$ is larger than the $(1-\alpha)$ quantile of the asymptotic null distribution, which is the law associated with $\Vert X\Vert_p^p$, where $X$ follows the multivariate normal distribution $\mathcal{N}(0,\bm{\Sigma})$. However, in practice, $\text{CI}_{n,p}$ cannot be computed as it requires the access to samples from the conditional means involved in the statistic, namely $\mathbb{E}_{\ddot{X}}\left[k_{\mathcal{\ddot{X}}}(\mathbf{t}^{(1)}_j,\ddot{X})|Z\right]$ and $\mathbb{E}_{Y}\left[k_{\mathcal{Y}}(t^{(2)}_j,Y)|Z\right]$ for all $j\in\{1,\dots,J\}$, which are unknown. Below, we show how to estimate these conditional means by using Regularized Least-Squares~(RLS) estimators. 


\subsection{Approximation of the Test Statistic}
The oracle statistic defined above involves conditional means that are unknown and cannot be used directly in practice. To alleviate this issue, we provide here a practical test statistic which approximates the oracle one while conserving its asymptotic behavior.
 
Our goal here is to estimate $\mathbb{E}_{\ddot{X}}\left[k_{\mathcal{\ddot{X}}}(\mathbf{t}^{(1)}_j,\ddot{X})|Z=\cdot\right]$ and $\mathbb{E}_{Y}\left[k_{\mathcal{Y}}(t^{(2)}_j,Y)|Z=\cdot\right]$ for all $j\in\{1,\dots,J\}$ in order to effectively approximate of our statistic. To do so, we consider kernel-based regularized least squares (RLS) estimators. Let $1 \leq r\leq n$ and $\{(x_i,z_i,y_i)\}_{i=1}^r$ be a subset of $r$ samples. Let also $j\in\{1,\dots,J\}$, and denote by $H_{\mathcal{Z}}^{1,j}$ and $H_{\mathcal{Z}}^{2,j}$ two separable RKHSs on $\mathcal{Z}$. Denote also by $k_{\mathcal{Z}}^{1,j}$ and $k_{\mathcal{Z}}^{2,j}$ their associated kernels and $\lambda^{(1)}_{j,r},~ \lambda^{(2)}_{j,r}>0$ the regularization parameters involved in the RLS regressions. Then, the RLS estimators are the unique solutions of the following problems:
\begin{align*}
 &\min_{h\in H_{\mathcal{Z}}^{2,j}}\frac{1}{r}\sum_{i=1}^r \left(h(z_i) -  k_{\mathcal{Y}}(t^{(2)}_j,y_i)\right)^2 +\lambda^{(2)}_{j,r}\Vert h\Vert_{H_{\mathcal{Z}}^{2,j}}^2\; \text{and}\\
    &\min_{h\in H_{\mathcal{Z}}^{1,j}}\frac{1}{r} \sum_{i=1}^r\left(h(z_i) -  k_{\mathcal{\ddot{X}}}(\mathbf{t}^{(1)}_j,(x_i,z_i))\right)^2 +\lambda^{(1)}_{j,r}\Vert h\Vert_{H_{\mathcal{Z}}^{1,j}}^2,
\end{align*}
which we denote by $h^{(2)}_{j,r}$ and $h^{(1)}_{j,r}$, respectively. These estimators have simple expressions in terms of the kernels involved. For example, let $k_{\mathcal{\ddot{X}}}(\mathbf{t}^{(1)}_j,\ddot{X}_r):=[k_{\mathcal{\ddot{X}}}(\mathbf{t}^{(1)}_j,(x_1,z_1)),\dots,k_{\mathcal{\ddot{X}}}(\mathbf{t}^{(1)}_j,(x_r,z_r))]^T$, then for any $z\in\mathcal{Z}$, the estimator  $h^{(1)}_{j,r}$ can be expressed as
\begin{align*}
h^{(1)}_{j,r}(z)&=\sum_{i=1}^r [\alpha^{(1)}_{j,r}]_i k^{1,j}_{\mathcal{Z}}(z_i,z)\; , \text{~~with}\\
\alpha^{(1)}_{j,r}&:= (\mathbf{K}^{1,j}_{r,\mathcal{Z}}+r\lambda^{(1)}_{j,r}\text{Id}_r)^{-1} k_{\mathcal{\ddot{X}}}(\mathbf{t}^{(1)}_j,\ddot{X}_r)\in\mathbb{R}^{r}, 
\end{align*}
where $\mathbf{K}^{1,j}_{r,\mathcal{Z}}:=(k^{1,j}_{\mathcal{Z}}(z_i,z_j))_{1\leq i,j\leq r}$. Similarly, we obtain simple expressions of $h^{(2)}_{j,r}$.
We can now introduce our new estimator of the witness function at each location $(\mathbf{t}^{(1)}_j,t^{(2)}_j)$ as follows:
\begin{align*}
  \widetilde{\Delta}_{n,r}(\mathbf{t}_j^{(1)},t_j^{(2)}):= \frac{1}{n}\sum_{i=1}^n & \left(k_{\mathcal{\ddot{X}}}(\mathbf{t}_j^{(1)},\ddot{x}_i)- h^{(1)}_{j,r}(z_i)\right)\\
  &\times\left(k_{\mathcal{Y}}(t^{(2)}_j,y_i)- h^{(2)}_{j,r}(z_i)\right),
\end{align*} 
and the proposed test statistic becomes
\begin{align*}
\widetilde{\text{CI}}_{n,r,p}&:=\sum_{j=1}^J \left|  \widetilde{\Delta}_{n,r}(\mathbf{t}^{(1)}_j,t^{(2)}_j)\right|^p\;.
\end{align*} 
\paragraph{Asymptotic Distribution.} To get the asymptotic distribution, we need to make two extra assumptions. Let us define, for $m\in\{1,2\}$ and $j\in\{1,\dots,J\}$, $L^{m,j}_{Z}$---the operator on $L^2(\mathcal{Z},P_{Z})$ as
   $L^{m,j}_{Z}(g)(\cdot) = \int_{\mathcal{Z}} k^{m,j}_{\mathcal{Z}}(\cdot,z) g(z)dP_{Z}(z)$.

\begin{assumption}
\label{ass:spectrum}
There exists $Q>0$, and $\gamma\in[0,1]$ such that for all $\lambda>0$, $m\in\{1,2\}$ and $j\in\{1,\dots,J\}$:
\begin{align*}
    \text{Tr}((L^{m,j}_{Z}+\lambda I)^{-1}L^{m,j}_{Z})\leq Q \lambda^{-\gamma}.
\end{align*}
\end{assumption}

\begin{assumption}
\label{ass:source}
There exists $2\geq \beta>1$ such that for any  $j\in\{1,\dots,J\}$,
$(\mathbf{t}^{(1)},t^{(2)})\in\mathcal{\ddot{X}}\times\mathcal{Y}$, 
\begin{small}
\begin{align*}
    \mathbb{E}_{\ddot{X}}\left[k_{\mathcal{\ddot{X}}}(\mathbf{t}^{(1)},\ddot{X})|Z=\cdot\right]&\in \mathcal{R}\left(\left[L^{1,j}_{Z}\right]^{\beta/2}\right),\\
    \mathbb{E}_{Y}\left[k_{\mathcal{Y}}(t^{(2)},Y)|Z=\cdot\right]&\in \mathcal{R}\left(\left[L^{2,j}_{Z}\right]^{\beta/2}\right),
\end{align*}
\end{small}
where $\mathcal{R}\left(\left[L^{m,j}_{Z}\right]^{\beta/2}\right)$ is the image space of $\left[L^{m,j}_{Z}\right]^{\beta/2}$. Moreover, there exists $L,\sigma>0$ such that for all $l\geq 2$ and  $P_Z$-almost all $z\in\mathcal{Z}$
\begin{small}
\begin{align*}
    &\mathbb{E}_{\ddot{X}\mid Z=z}\left[\Big|k_{\mathcal{\ddot{X}}}(\mathbf{t}^{(1)},\ddot{X})-  \mathbb{E}_{\ddot{X}}\left[k_{\mathcal{\ddot{X}}}(\mathbf{t}^{(1)},\ddot{X})\mid Z\right]\Big|^{l}\right]\leq \frac{l!\sigma^2L^{l-2}}{2},\\
      &\mathbb{E}_{Y|Z=z}\left[\Big|k_{\mathcal{Y}}(t^{(2)},Y)-  \mathbb{E}_{Y}\left[k_{\mathcal{Y}}(t^{(2)},Y)|Z\right]\Big|^{l}\right]\leq \frac{l!\sigma^2L^{l-2}}{2}.
\end{align*}
\end{small}
\end{assumption}

These assumptions are central in our proofs and are common in kernel statistic studies~\citep{caponnetto2007optimal,fischer2020sobolev,rudi2017generalization}. Under these assumptions,~\cite{fischer2020sobolev} proved optimal learning rates for RLS in RKHS norm, which is essential to guarantee that our new statistic $\widetilde{\text{CI}}_{n,r,p}$, estimated with RLS, has the same asymptotic law as our oracle estimator $\text{CI}_{n,p}$.

To derive the asymptotic distribution of our new test statistic, we also need to define for all $j\in\{1,\dots,J\}$ and $i\in\{1,\dots,n\}$,
    $\widetilde{u}_{i,r}(j):= (k_{\mathcal{\ddot{X}}}(\mathbf{t}^{(1)}_j,\ddot{x}_i)- h^{(1)}_{j,r}(z_i))
    (k_{\mathcal{Y}}(t^{(2)}_j,y_i)- h^{(2)}_{j,r}(z_i))$,
$\widetilde{\mathbf{u}}_{i,r}:=(\widetilde{u}_{i,r}(1),\dots,\widetilde{u}_{i,r}(J))^T$, and $\widetilde{\mathbf{S}}_{n,r}:=\frac{1}{n}\sum_{i=1}^n  \widetilde{\mathbf{u}}_{i,r}$. Note that $\widetilde{\text{CI}}_{n,r,p}=\lVert\widetilde{\mathbf{S}}_{n,r}\rVert_p^p$. 
In the following proposition, we show the asymptotic behavior of the statistic of interest. The proof of this proposition is given in Appendix~\ref{prv:rls-law}.
\begin{prop} 
\label{prop:rls-law}
Suppose that Assumptions~\ref{assump-kernel}-\ref{ass:spectrum}-\ref{ass:source} are verified. Let $p\geq 1$, $J\geq 1$, $((\mathbf{t}^{(1)}_1,t^{(2)}_1),\dots,(\mathbf{t}^{(1)}_J,t^{(2)}_J))\in(\mathcal{\ddot{X}}\times\mathcal{Y})^J$, $r_n$ such that $n^{\frac{\beta+\gamma}{2\beta}}\in o(r_n)$ and $\lambda_{r_n}=r_n^{-\frac{1}{1+\gamma}}$. Then, under $H_0$, we have $\sqrt{n}\widetilde{\mathbf{S}}_{n,r_n}\rightarrow \mathcal{N}(0,\bm{\Sigma})$. Moreover, under $H_1$, if  the $((\mathbf{t}^{(1)}_j,t^{(2)}_j))_{j=1}^J$ are sampled independently according to $\Gamma$, then $\Gamma$-almost surely, for any $q\in\mathbb{R}$, $\lim_{n\rightarrow\infty}P( n^{p/2}\widetilde{\emph{CI}}_{n,r_n,p} \geq q)=1$.
\end{prop}

From the above proposition, we can derive a consistent test at level $\alpha$ for $0<\alpha<1$. Indeed, we obtain the asymptotic null distribution of $n^{p/2}\widetilde{\text{CI}}_{n,r_n,p}$ and we show that under the alternative hypothesis $H_1$, $\Gamma$-almost surely, $n^{p/2}\widetilde{\text{CI}}_{n,r_n,p}$ is arbitrarily large as $n$ goes to infinity. For a fixed level $\alpha$, the test rejects $H_0$ if  $n^{p/2}\widetilde{\text{CI}}_{n,r_n,p}$ exceeds the ($1 - \alpha$)-quantile of its asymptotic null distribution and this test is therefore consistent. For example, when $p\in\{1,2\}$, the asymptotic null distribution of $n^{p/2}\widetilde{\text{CI}}_{n,r_n,p}$ is either a sum of correlated Nakagami variables\footnote{the probability density function of a Nakagami distribution of parameters $m\geq \frac{1}{2}$ and $\omega>0$ is  for all $ x\geq 0$,\\ $f(x,m,\omega)=\frac{2m^m }{G(m)\omega^m}x^{2m-1}\exp(\frac{-m}{\omega}x^2)$ where $G$ is the Euler Gamma function.} ($p=1$) or a sum of correlated chi square variables ($p=2$). 
However, computing the quantiles of these asymptotic null distributions can be computationally expensive as it requires a bootstrap or permutation procedure. In the following, we consider a different approach in which we normalize the statistic to obtain a simple asymptotic null distribution.

\subsection{Normalization of the Test Statistic} 
Herein, we consider a normalized variant of our statistic $\widetilde{\text{CI}}_{n,r,p}$ in order to obtain a tractable asymptotic null distribution. Denote $\bm{\Sigma}_{n,r}:=\frac{1}{n}\sum_{i=1}^n \widetilde{\bm{u}}_{i,r}\widetilde{\bm{u}}_{i,r}^T$ and let $\delta_n>0$, then the normalized statistic considered is given by
\begin{align*}
    \widetilde{\text{NCI}}_{n,r,p}:=\Vert  (\bm{\Sigma}_{n,r}+\delta_n\text{Id}_J)^{-1/2}\widetilde{\mathbf{S}}_{n,r}\Vert_{p}^p.
\end{align*}
In the next proposition, we show that our normalized approximate statistic converges in law to the standard multivariate normal distribution. The proof is given in Appendix~\ref{prv:norm-law}.
\begin{prop}
\label{prop:norm-law}
Suppose that Assumptions~\ref{assump-kernel}-\ref{ass:spectrum}-\ref{ass:source} are verified. Let $p\geq 1$, $J\geq 1$, $((\mathbf{t}^{(1)}_1,t^{(2)}_1),\dots,(\mathbf{t}^{(1)}_J,t^{(2)}_J))\in(\mathcal{\ddot{X}}\times\mathcal{Y})^J$, $r_n$ such that $n^{\frac{\beta+\gamma}{2\beta}}\in o(r_n)$, $\lambda_n=r_n^{-\frac{1}{1+\gamma}}$ and $(\delta_n)_{n\geq 0}$ a sequence of positive real numbers such that $\lim_{n\rightarrow\infty}\delta_n=0$. Then, under $H_0$, we have $\sqrt{n}(\bm{\Sigma}_{n,r}+\delta_n\text{Id}_J)^{-1/2}\mathbf{S}_{n,r_n}\rightarrow \mathcal{N}(0,\text{Id}_J)$. Moreover, under $H_1$, if the $((\mathbf{t}^{(1)}_j,t^{(2)}_j))_{j=1}^J$ are sampled independently according to $\Gamma$, then $\Gamma$-almost surely, for any $q\in\mathbb{R}$, $\lim_{n\rightarrow\infty}P(n^{p/2}\widetilde{\emph{NCI}}_{n,r_n,p} \geq q)=1$.
\end{prop}

\begin{remark}
We emphasize that $J$ need not increase with $n$ for test consistency. Note also that the regularization parameter $\delta_n$ allows to ensure that $(\bm{\Sigma}_{n,r}+\delta_n\text{Id}_J)^{-1/2}$ can be stably computed. In practice, $\delta_n$ requires no tuning, and can be set to be a very small constant.
\end{remark}

Our normalization procedure allows us to derive a simple statistical test, which is distribution-free under the null hypothesis. 

\paragraph{Statistical test at level $\alpha$:} Compute $n^{p/2}\widetilde{\text{NCI}}_{n,r,p}$, choose the threshold $\tau$ corresponding to the $(1-\alpha)$ quantile of the asymptotic null distribution, and reject the null hypothesis whenever  $n^{p/2}\widetilde{\text{NCI}}_{n,r,p}$ is larger than $\tau$. For example, if $p=2$,  the threshold $\tau$ is the $(1 -\alpha)$-quantile of $\chi^2(J)$, i.e., a sum of $J$ \emph{independent} standard $\chi^2 $ variables.

\paragraph{Total Complexity:} Our normalized statistic $\widetilde{\text{NCI}}_{n,r,p}$ requires first to compute $\alpha^{(1)}_{j,r}$ and $\alpha^{(2)}_{j,r}$. These quantities can be evaluated in at most $\mathcal{O}(r^2d+r^3)$ algebraic operations where $d$ corresponds to the computational cost of evaluating the kernels involved in the RLS regressions. We will use the above for the complexity analysis of our method, although one can consider the theoretical estimation given by the Coppersmith–Winograd algorithm~\citep{coppersmith1987matrix} that reduces the computational cost to $\mathcal{O}(r^2d+r^{2.376})$. Once $\alpha^{(1)}_{j,r}$ and  $\alpha^{(2)}_{j,r}$ are available, evaluating the RLS estimators $h_{j,r}^{(1)}$ and $h_{j,r}^{(2)}$ requires only $\mathcal{O}(rd)$ operations. Then $\widetilde{\Delta}_{n,r}$ can be evaluated in $\mathcal{O}(nrd + r^2d + r^3)$ operations and $\widetilde{\text{CI}}_{n,r,p}$ has therefore a computational complexity of $\mathcal{O}(J(nrd + r^2d+r^3))$. The computation of $\widetilde{\text{NCI}}_{n,r,p}$ requires inverting a $J \times J$ matrix $\bm{\Sigma}_{n,r}+\delta_n\text{Id}_J$, but this is fast and numerically stable: we empirically observe that only a small value of J is required (see Section~\ref{sec-experiments-main}), e.g. less than 10.  Finally the total computational cost to evaluate $\widetilde{\text{NCI}}_{n,r,p}$ is $\mathcal{O}(J(nrd + r^2d + r^3) + nJ^2 + J^3)$.

\subsection{Hyperparameters}
The hyperparameters of our statistics $\widetilde{\text{NCI}}_{n,r,p}$ fall into two categories: those directly involved with the test and those of the regression. We assume from now on that all the kernels involved in the computation of our statistics are \emph{Gaussian kernels}, and consider $n$ i.i.d. observations $\{(x_i,z_i,y_i)\}_{i=1}^n$.

The first category includes both the choice of the locations $((t_x,t_z)_j,(t_y)_j))_{j=1}^J$ on which differences between the mean embeddings are computed and the choice of the kernels $k_{\mathcal{\ddot{X}}}$ and $k_{\mathcal{Y}}$. Each location $t_x,t_y,t_z$  is randomly chosen according to a Gaussian variable with mean and covariance of $\{x_i\}_{i=1}^n$, $\{y_i\}_{i=1}^n$, and $\{z_i\}_{i=1}^n$, respectively. As we consider Gaussian kernels, we should also choose the bandwidths. Here, we restrict ourselves to one-dimensional kernel bandwidths $\sigma_{\mathcal{{X}}}$, $\sigma_{\mathcal{{Y}}}$, and $\sigma_{\mathcal{{Z}}}$ for the kernels $k_{\mathcal{X}}$, $k_{\mathcal{Y}}$, and $k_{\mathcal{Z}}$, respectively. More precisely, we select the median of $\{\lVert x_i-x_j\rVert_2\}_{1 \leq i<j\leq n}$, $\{\lVert y_i-y_j\rVert_2\}_{1\leq i<j\leq n}$, and $\{\lVert z_i-z_j\rVert_2\}_{1\leq i<j\leq n}$ for $\sigma_{\mathcal{X}}$, $\sigma_{\mathcal{Y}}$,  and $\sigma_{\mathcal{Z}}$, respectively.


The other category contains all the kernels $k^{m,j}$ and the regularization parameters $\lambda^{(m)}_{j,r}$ involved in the RLS problems. These parameters should be selected carefully to avoid either underfitting of the regressions, which may increase the type-I error, or overfitting, which may result in a large type-II error. To optimize these, similarly to~\cite{zhang2012kernel}, we consider a GP regression that maximizes the likelihood of the observations. While carrying out a precise GP regression can be prohibitive, in practice, we run this method only on a batch of size $200$ observations randomly selected and we perform only $10$ iterations for choosing the hyperparameters involved in the RLS problems. Hence, our optimization procedure does not affect the total computational cost as it is independent of the number of observations $n$. 

\begin{remark}
Note that here we select the locations $((t_x,t_z)_j,(t_y)_j))_{j=1}^J$ randomly. If one wants to choose the locations by maximizing the power the test, then a bi-level optimization problem appears as the RLS estimators depend on the locations chosen and we believe that it is out of the scope of this paper.
\end{remark}

\section{Experiments}
\label{sec-experiments-main}
\begin{figure}[ht]
\begin{tabular}{cc} 
\includegraphics[height=2.4cm]{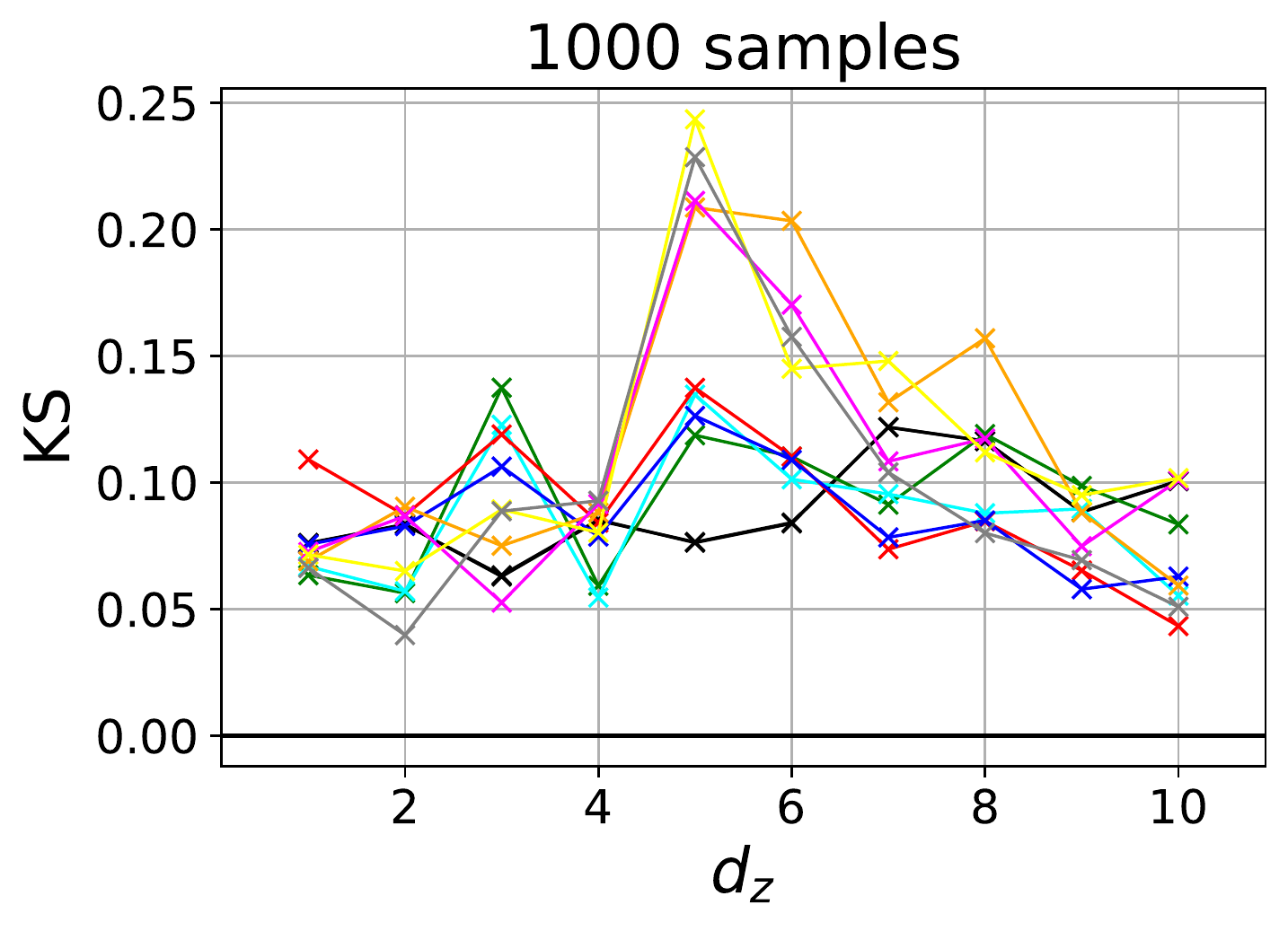}& 
\includegraphics[height=2.4cm]{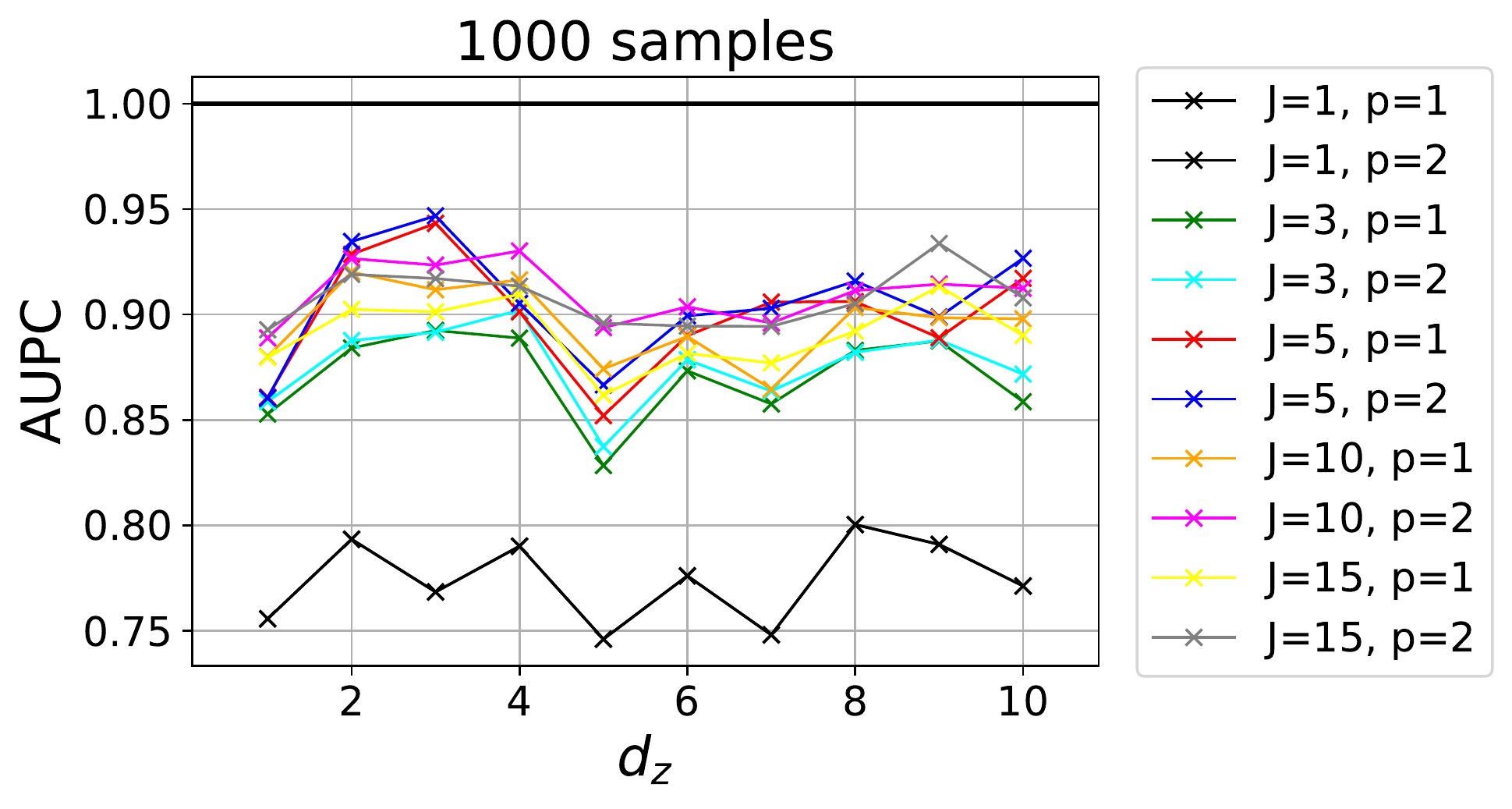}
\end{tabular}
\caption{Comparison of the KS statistic (\emph{left}) and the AUPC (\emph{right}) of our test statistic $\widetilde{\text{NCI}}_{n,r,p}$ when the data is generated respectively from the models defined in~\eqref{exp-strobl-h0} and~\eqref{exp-strobl-h1} with Gaussian noises for multiple $p$ and $J$. For each problem, we draw $n=1000$ samples and repeat the experiment 100 times. We set $r=1000$ and report the results obtained when varying the dimension $d_z$ of each problem from 1 to 10. Observe that when $J=1$, for all $p\geq 1$ $\widetilde{\text{NCI}}_{n,r,1}=\widetilde{\text{NCI}}_{n,r,p}$, therefore there is only one  common black curve.
\label{fig-exp-param}}
\vspace{-0.4cm}
\end{figure}

\begin{figure}[ht]
\begin{tabular}{cc} 
\includegraphics[width=0.22\textwidth]{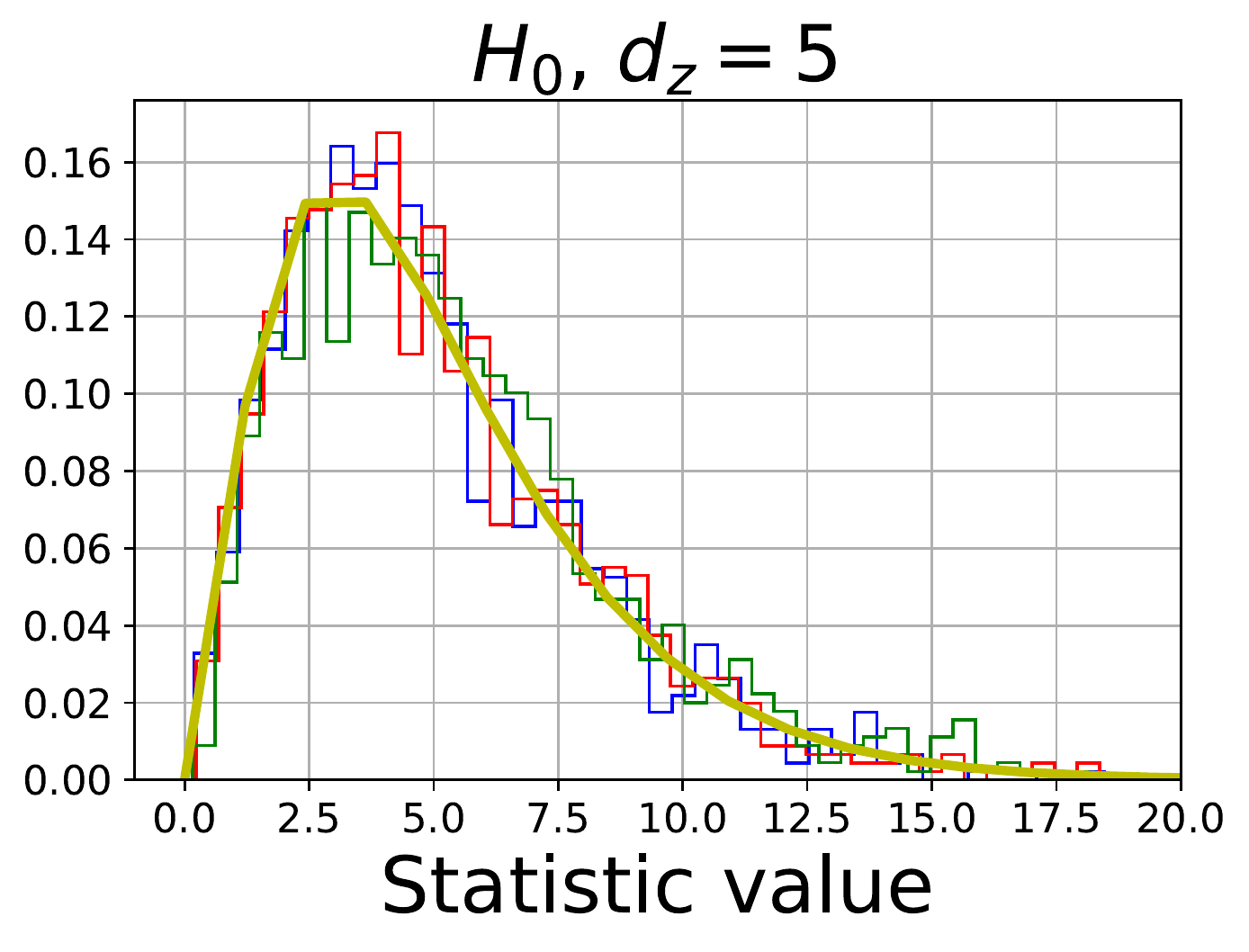}& \includegraphics[width=0.22\textwidth]{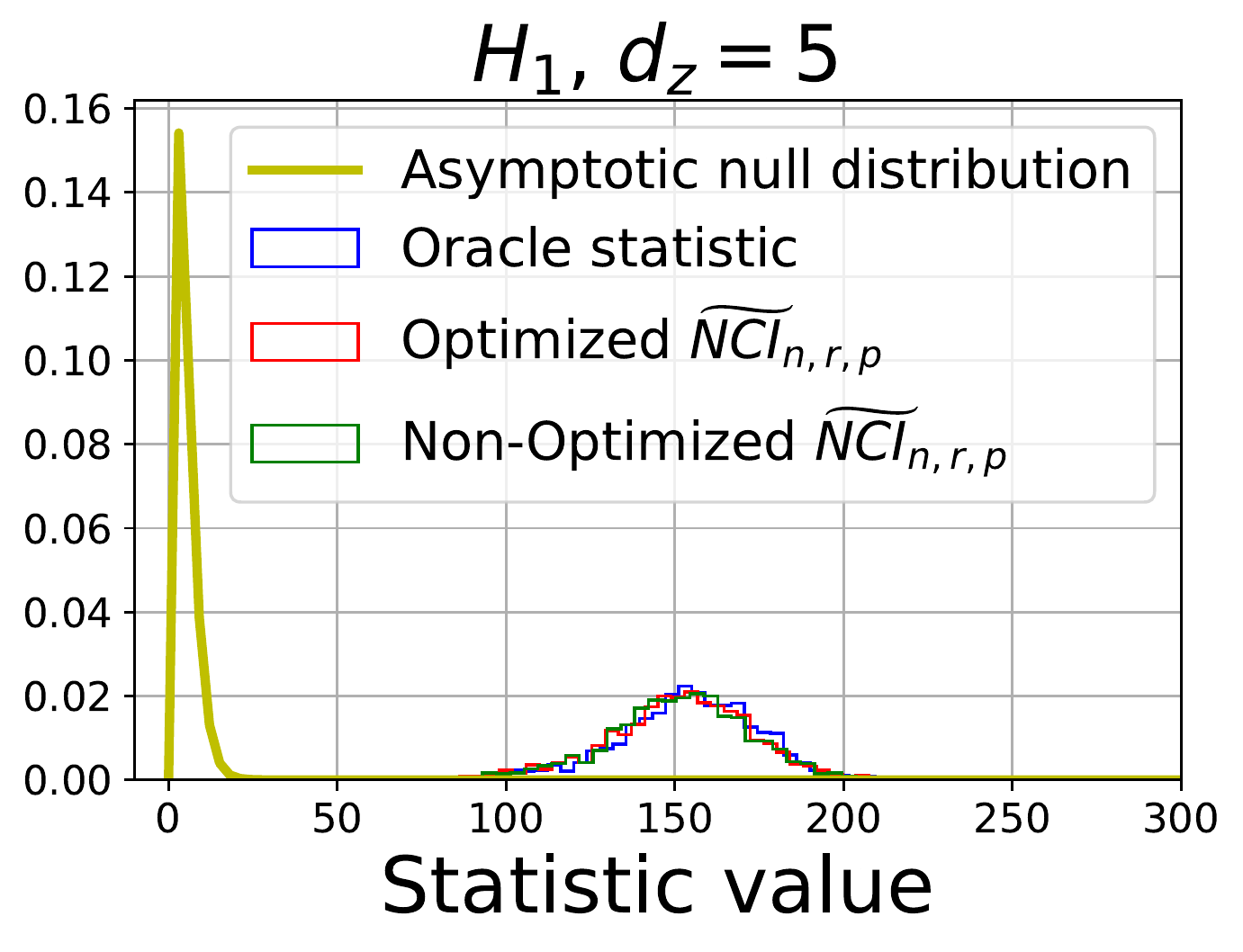} \\
\includegraphics[width=0.22\textwidth]{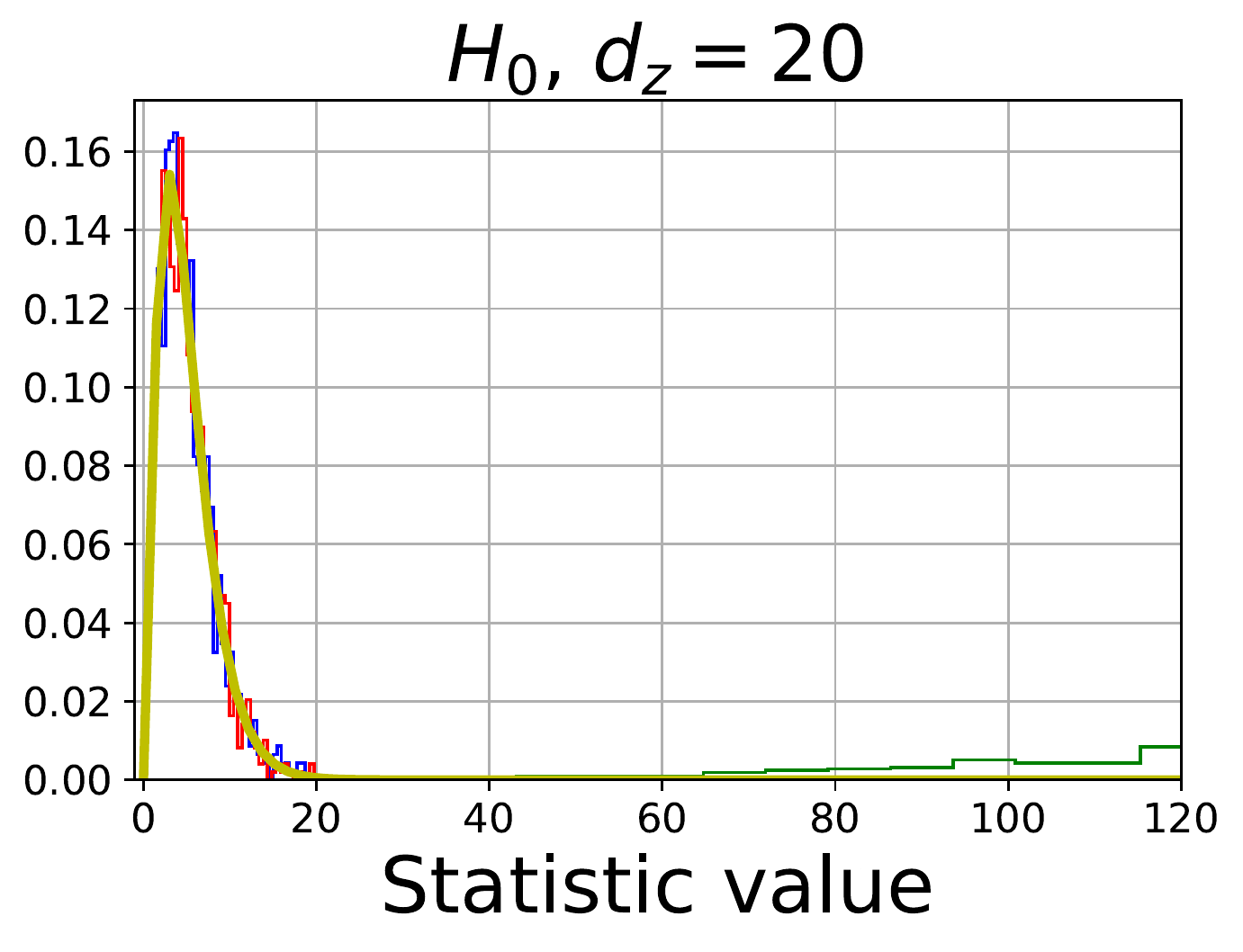}& \includegraphics[width=0.22\textwidth]{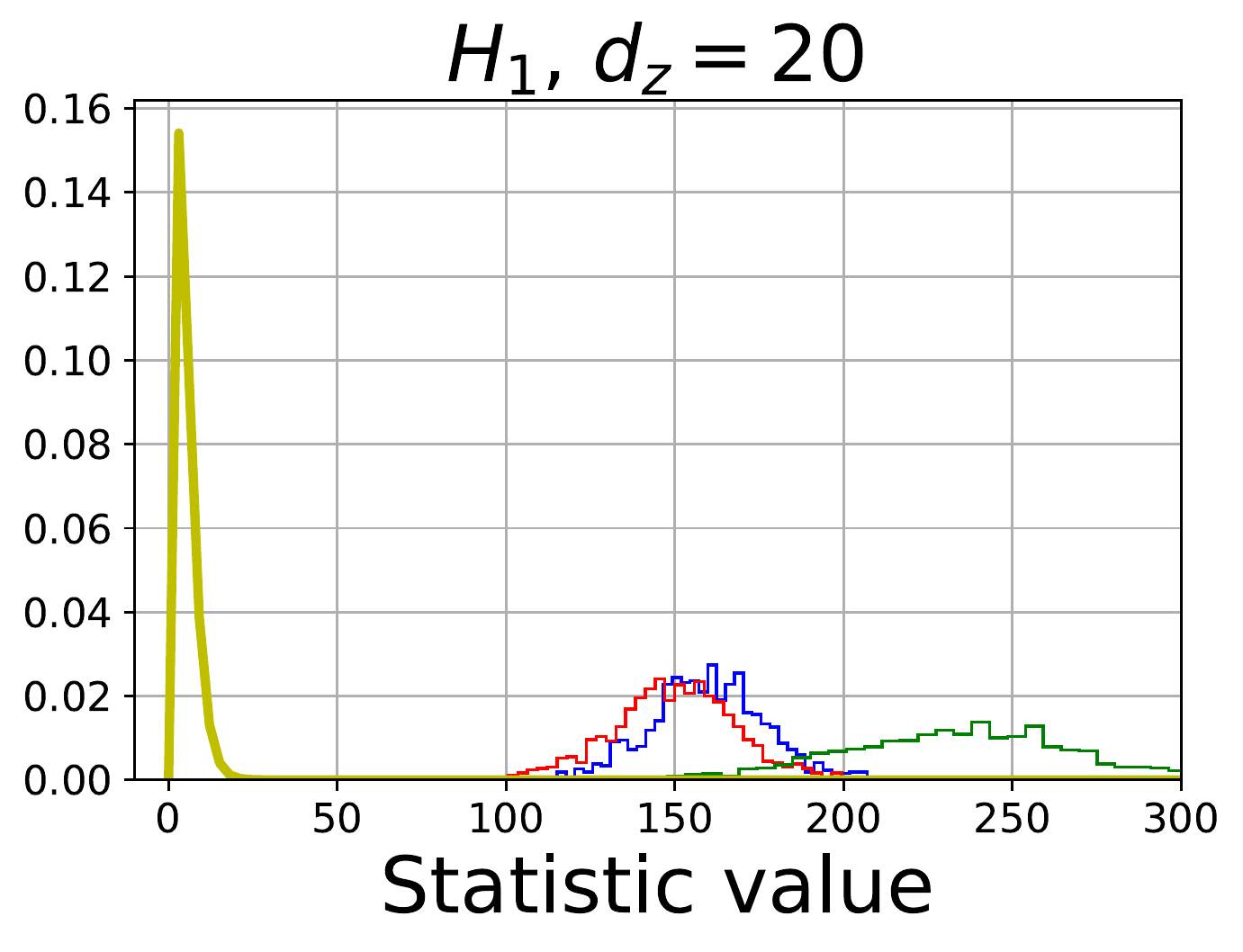} \\ 
\end{tabular}
\caption{Comparisons between the empirical distributions of the normalized version of the oracle statistic  $\widehat{\text{CI}}_{n,p}$ and the approximate normalized statistic $\widetilde{\text{NCI}}_{n,r,p}$, with the theoretical asymptotic null distribution when the data is generated either from the model defined in~\eqref{exp-illustration-h0} (\emph{left}) or the one defined in~\eqref{exp-illustration-h1} (\emph{right}). We set the dimension of $Z$ to be either $d_z=5$ (\emph{top row}) or $d_z=20$ (\emph{bottom row}). For each problem, we draw $n=1000$ samples and repeat the experiment 1000 times. In all the experiments, we set $J=5$ and $p=2$, thus the asymptotic null distribution follows a $\chi^2(5)$. Observe that both the oracle statistic and the approximated one recover the true asymptotic distribution under the null hypothesis. When $H_1$ holds, we can see that the two statistics manage to reject the null hypothesis. This figure also illustrates the empirical distribution of our approximate statistic when we do not optimize the hyperparameters involved in the RLS estimators: in this case we do not control the type-I error in the high dimensional setting.\label{fig-illustation-theory}}
\vspace{-0.4cm}
\end{figure}

\begin{figure*}[h]
\begin{tabular}{cccc} 
\includegraphics[height=2.9cm]{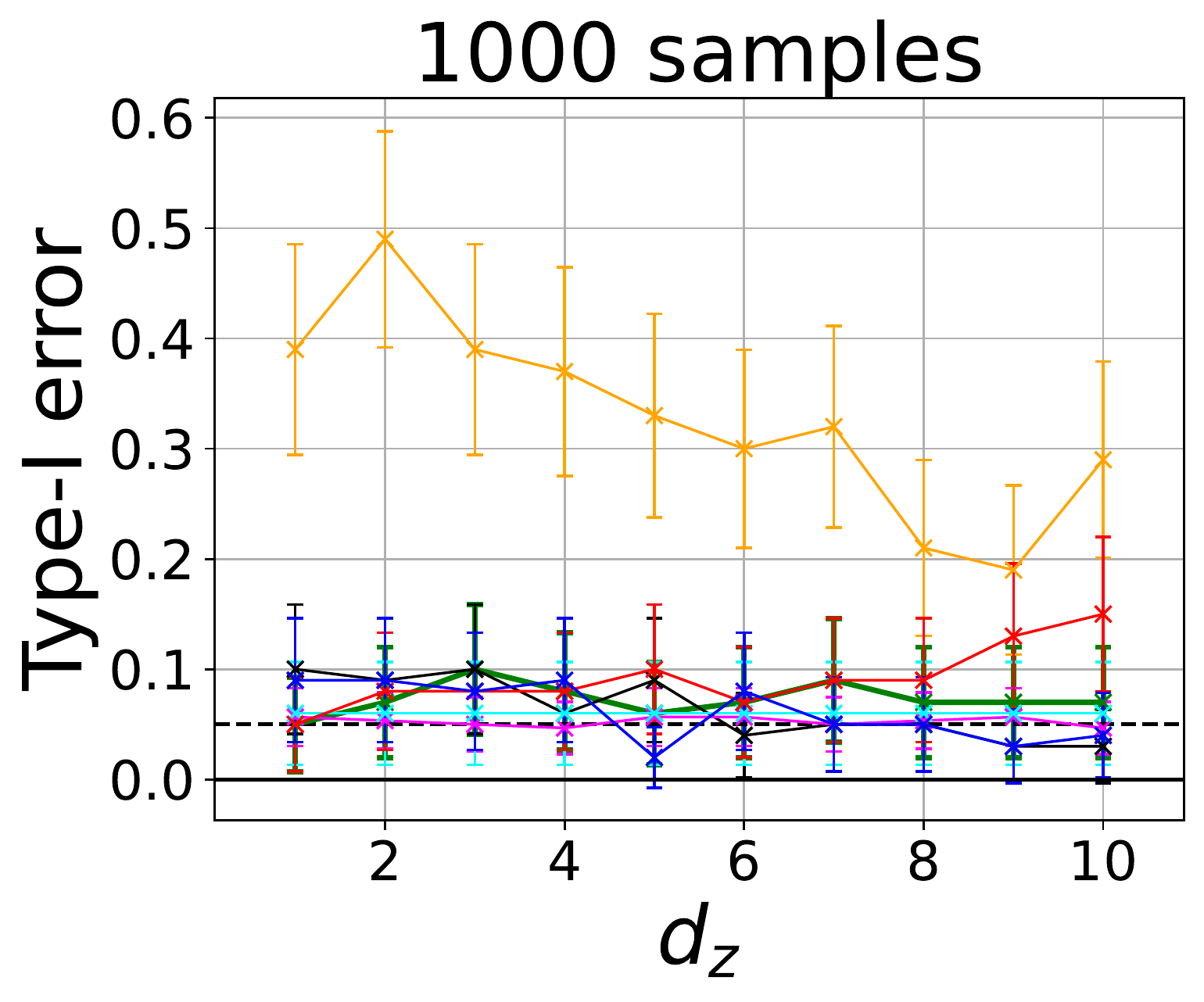}& \includegraphics[height=2.9cm]{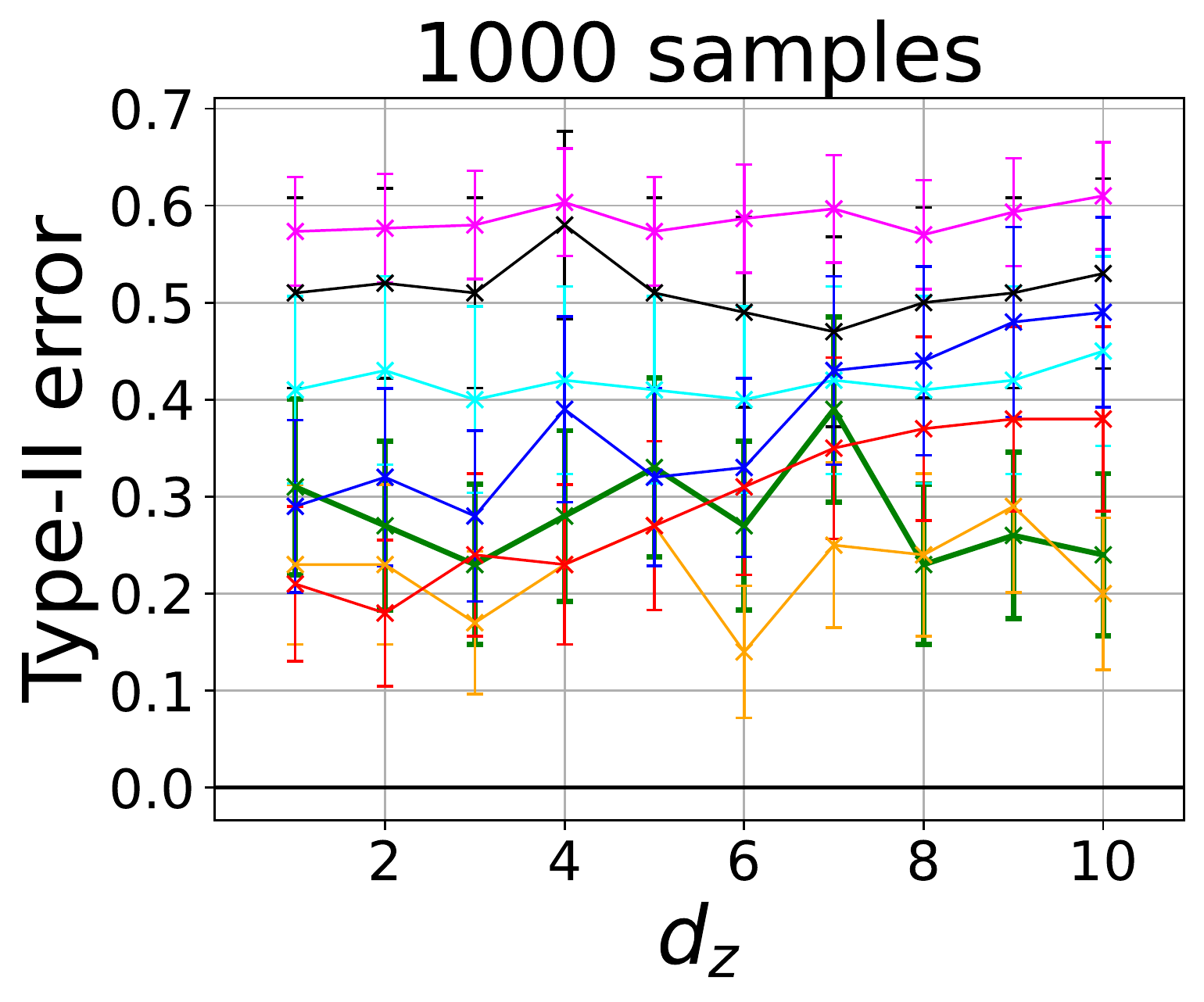} & 
\includegraphics[height=2.9cm]{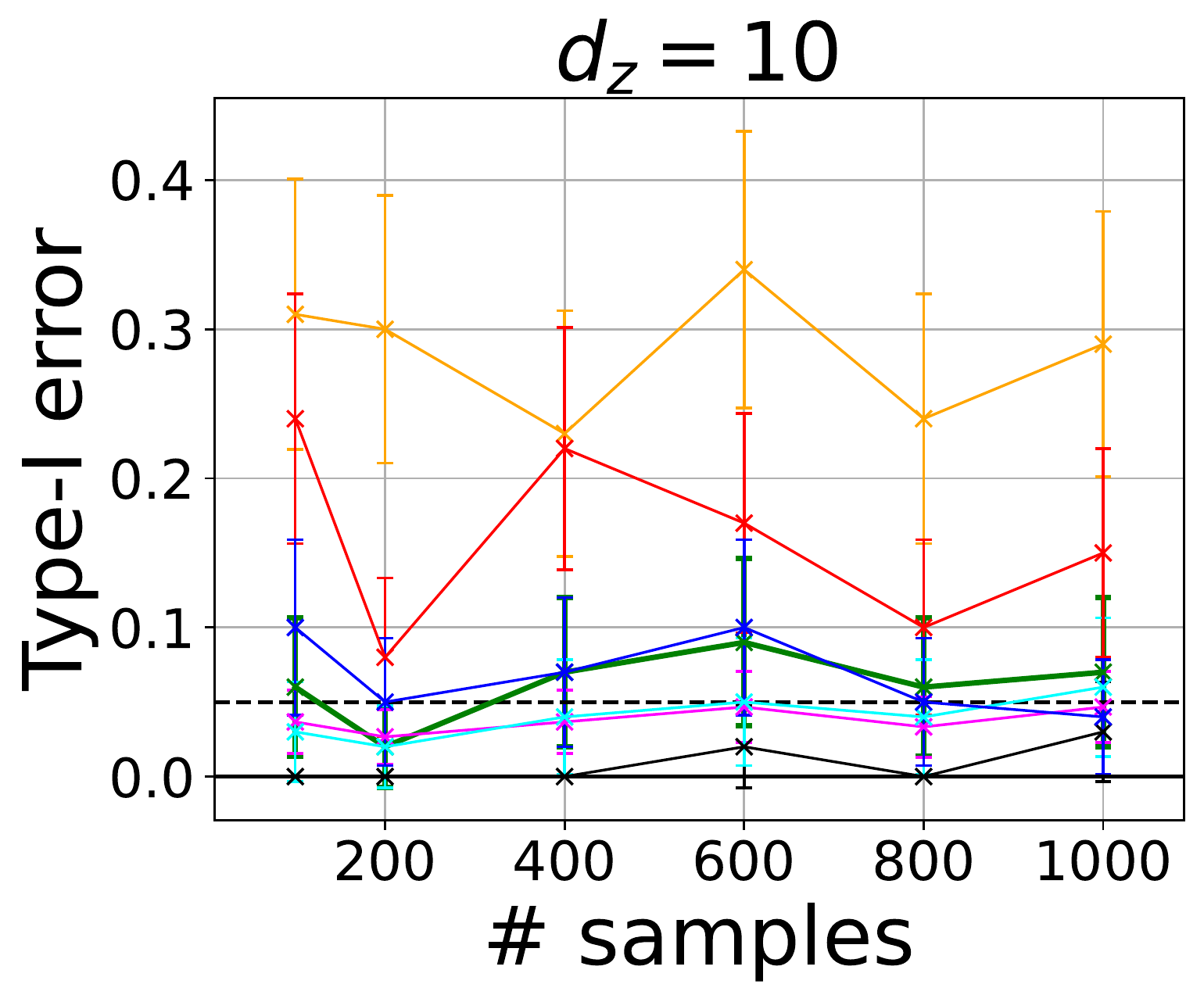}& \includegraphics[height=2.9cm]{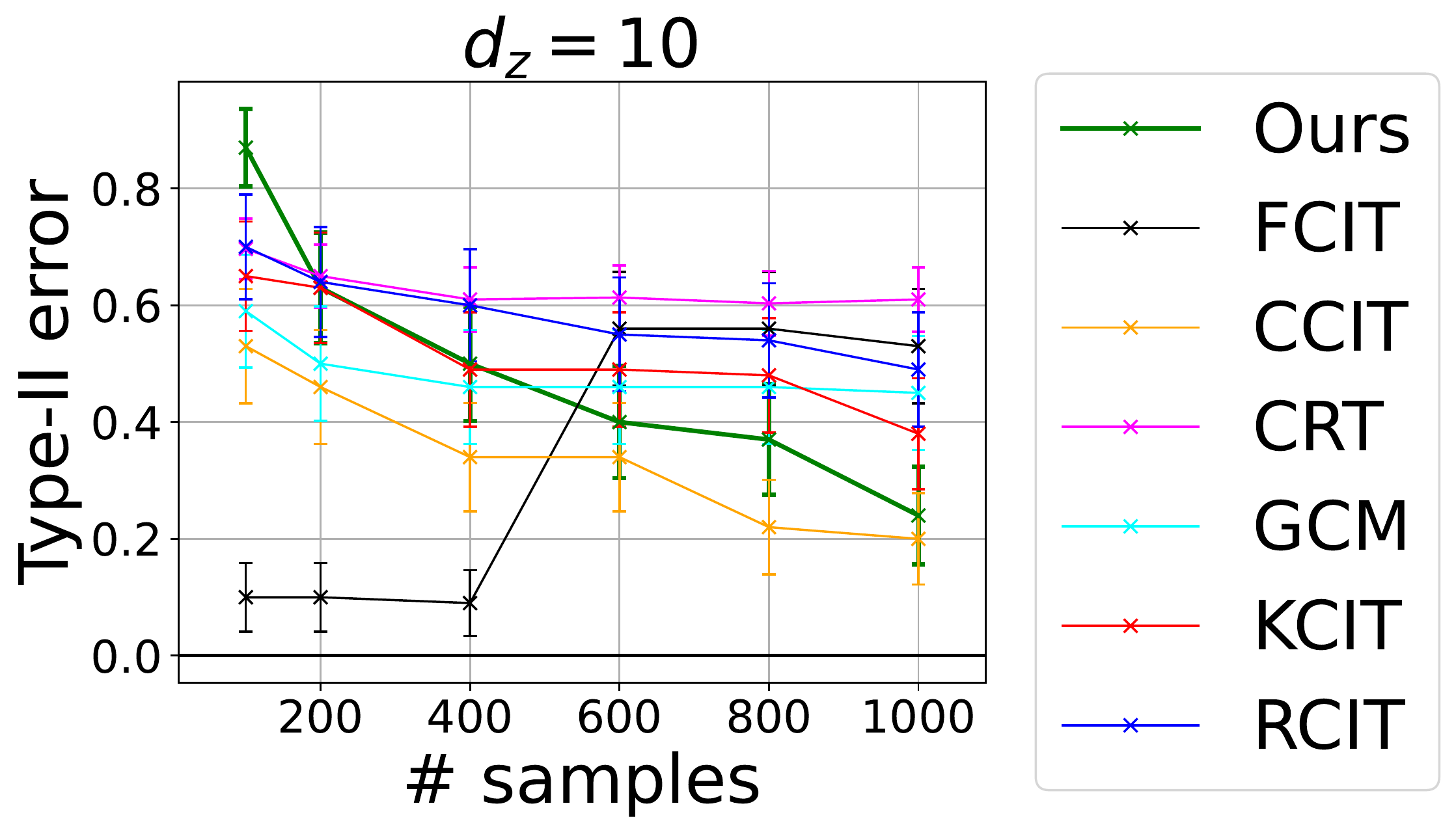}
\end{tabular}
\caption{Comparison of the type-I error at level $\alpha=0.05$ (dashed line) and the type-II error (lower is better) of our test procedure with other SoTA tests on the two problems presented in~\eqref{exp-strobl-h0} and~\eqref{exp-strobl-h1}  with Gaussian noises. Each point in the figures is obtained by repeating the experiment for 100 independent trials. (\emph{Left, middle-left}): type-I and type-II errors obtained by each test when varying the dimension $d_z$ from 1 to 10; here, the number of samples $n$ is fixed and equals $1000$. (\emph{Middle-right, right}): type-I and type-II errors obtained by each test when varying the number of samples $n$ from 100 to 1000; here, the dimension $d_z$ is fixed and equals $10$. 
\label{fig-exp-strobl-type}}
\vspace{-0.3cm}
\end{figure*}

The goal of this section is three fold: (i) to investigate the effects of the parameters $J$ and $p$ on the performances of our method, (ii) to validate our theoretical results depicted in Propositions~\ref{prop:oracle-law} and \ref{prop:norm-law}, and (iii) to compare our method with those proposed in the literature. In more detail, we first compare the performance of our method, both in terms of both power and type-I error, by varying the hyperparameters $J$ and $p$. We show that our method is robust to the choice of $p$, and also show that the power increases as $J$ increases. Then, we explore synthetic toy problems where one can derive an explicit formulation of the conditional means involved in our test statistic. In these cases, we can compute our proposed oracle statistic $\widehat{\text{CI}}_{n,p}$ and its normalized version, allowing us to show that under the null hypothesis we recover the theoretical asymptotic null distribution obtained in Proposition~\ref{prop:oracle-law}. We also reach similar conclusions regarding our approximate normalized test statistic, $\widetilde{\text{NCI}}_{n,r,p}$. In addition, in this experiment, we investigate the effect of the proposed optimization procedure for choosing the hyperparameters involved in the RLS estimators of $\widetilde{\text{NCI}}_{n,r,p}$, and show its benefits. Finally, we demonstrate on several synthetic experiments that our proposed testing procedure outperforms state-of-the-art (SoTA) methods both in terms of statistical power and type-I error, even in the high dimensional setting. The code is available at~
\href{https://github.com/meyerscetbon/lptest}{https://github.com/meyerscetbon/lp-ci-test}\footnote{Our code requires a slight modification of the Gaussian Process Regression implemented in scikit-learn~\cite{scikit-learn} to limit the number of iterations involved in the optimization procedure.}.

\textbf{Benchmarks.} We consider 6 synthetic data sets and compare the power and type-I error of our test $\widetilde{\text{NCI}}_{n,r,p}$ to the following 6 existing CI methods: \textbf{KCIT}~\citep{zhang2012kernel}, \textbf{RCIT}~\citep{strobl2019approximate}, \textbf{CCIT}~\citep{sen2017modelpowered}, \textbf{CRT}~\citep{candes2018panning} using correlation statistic from \citep{BellotS19}, \textbf{FCIT}~\citep{chalupka2018fast} and \textbf{GCM}~\citep{gcm2020}. Software packages of all the above tests are freely available online and each experiment was run on a single CPU.

\textbf{Evaluation.} To evaluate the performance of the tests, we consider four metrics. Under $H_0$, we report either the Kolmogorov-Smirnov (KS) test statistic between the distribution of p-values returned by the tests and the uniform distribution on $[0,1]$, or the type-I errors at level $\alpha=0.05$. Note that a valid conditional independence test should control the type-I error rate at any level $\alpha$. Here, a test that generates a p-value that follows the uniform distribution over $[0,1]$ will achieve this requirement. The latter property of the p-values translates to a small KS statistic value. Under $H_1$, we compute either the area under the power curve (AUPC) of the empirical cumulative density function of the p-values returned by the tests, or the resulting type-II error. A conditional test has higher power when its AUPC is closer to one. Alternatively, the smaller the type-II error is, the more powerful the test is.


\begin{figure*}[h]
\begin{tabular}{cccc} 
\includegraphics[height=2.9cm]{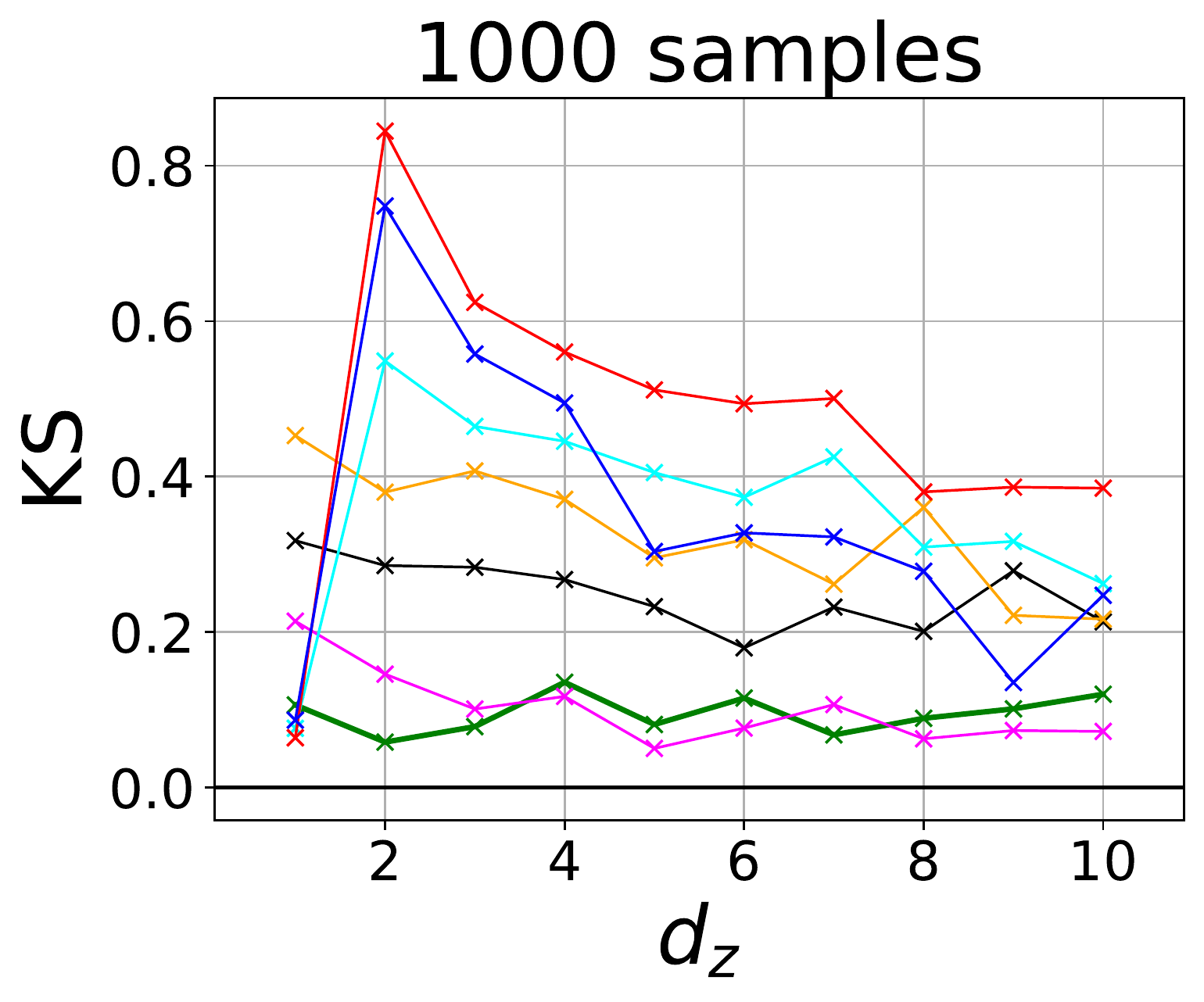}& \includegraphics[height=2.9cm]{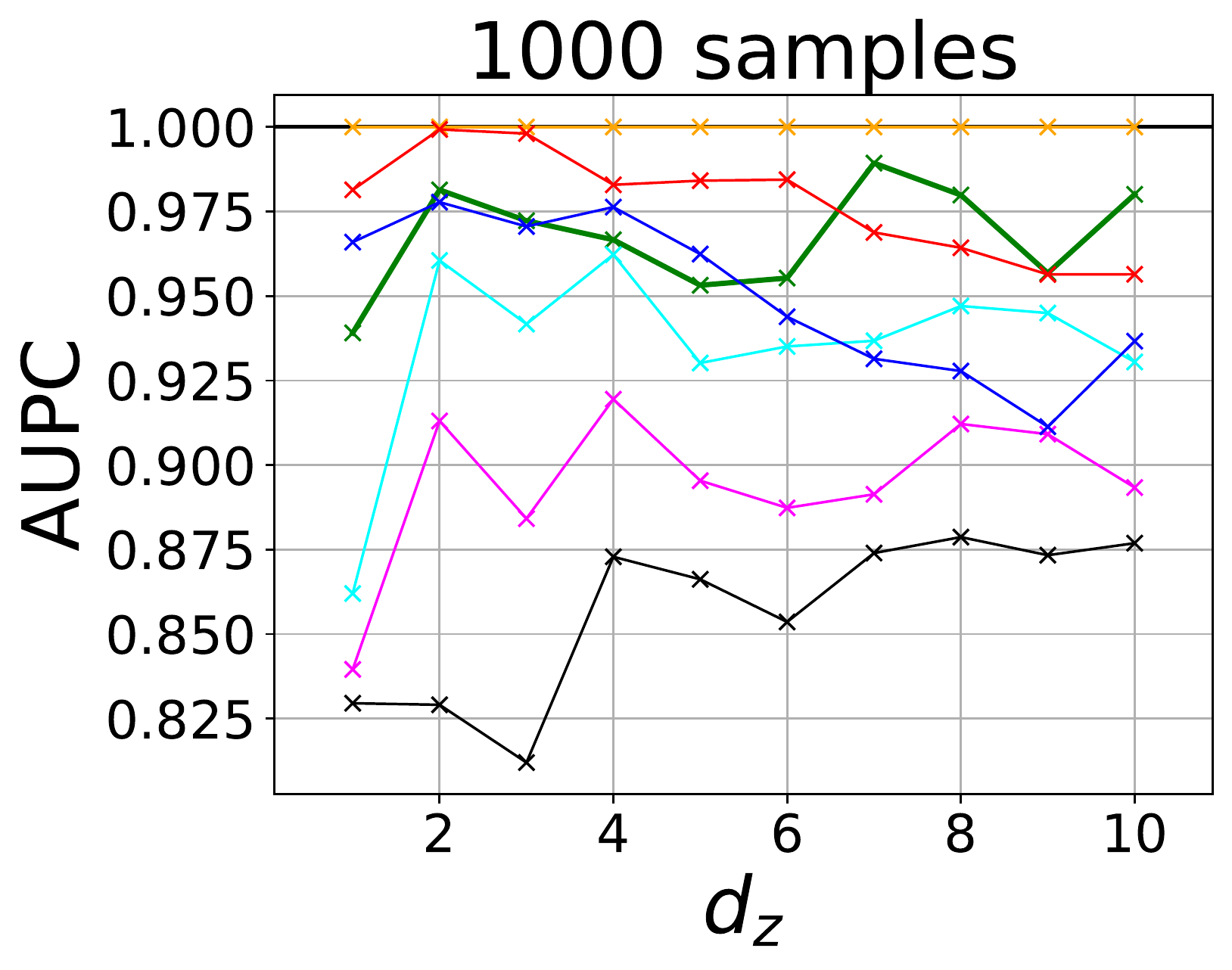} & 
\includegraphics[height=2.9cm]{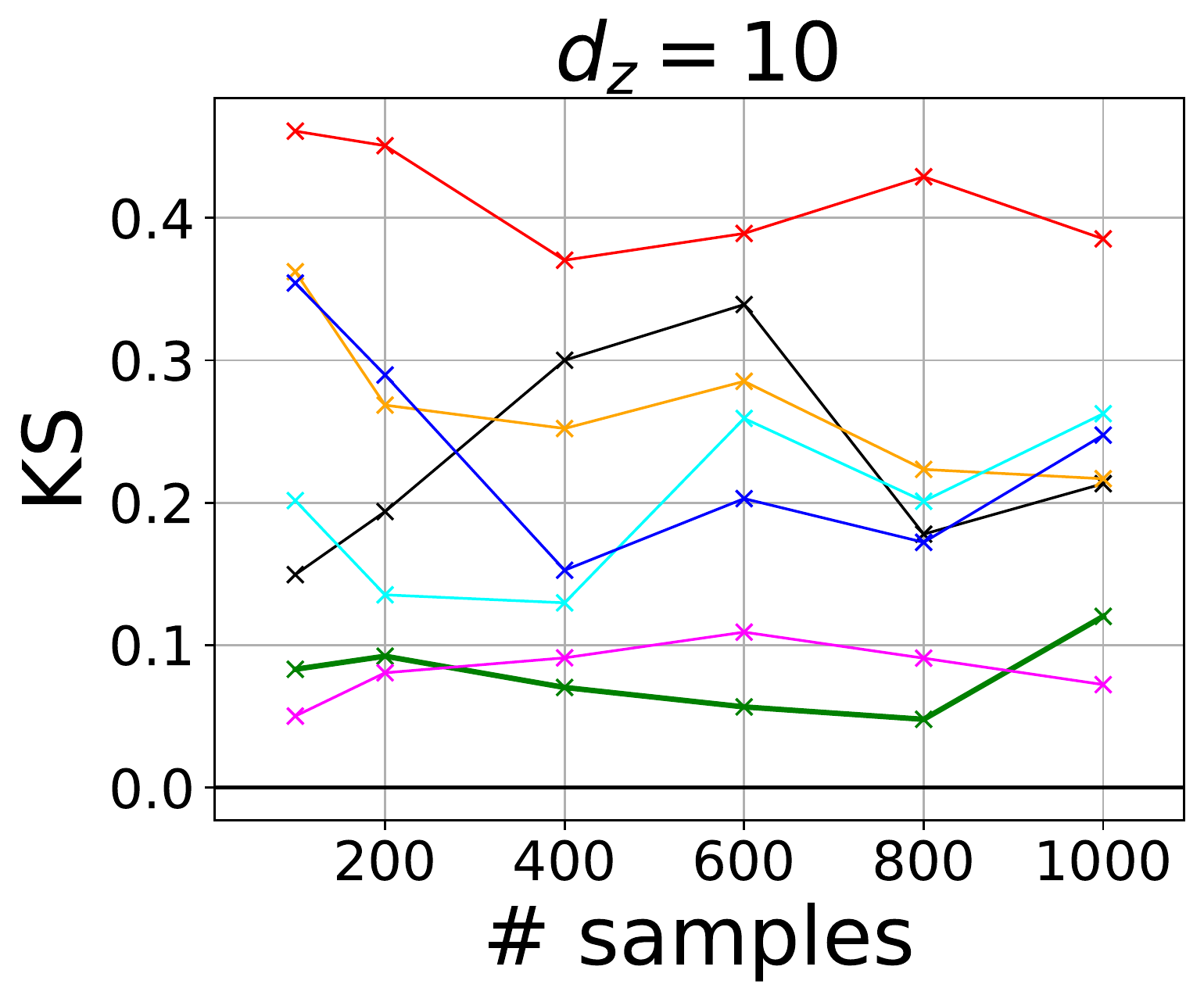}& \includegraphics[height=2.9cm]{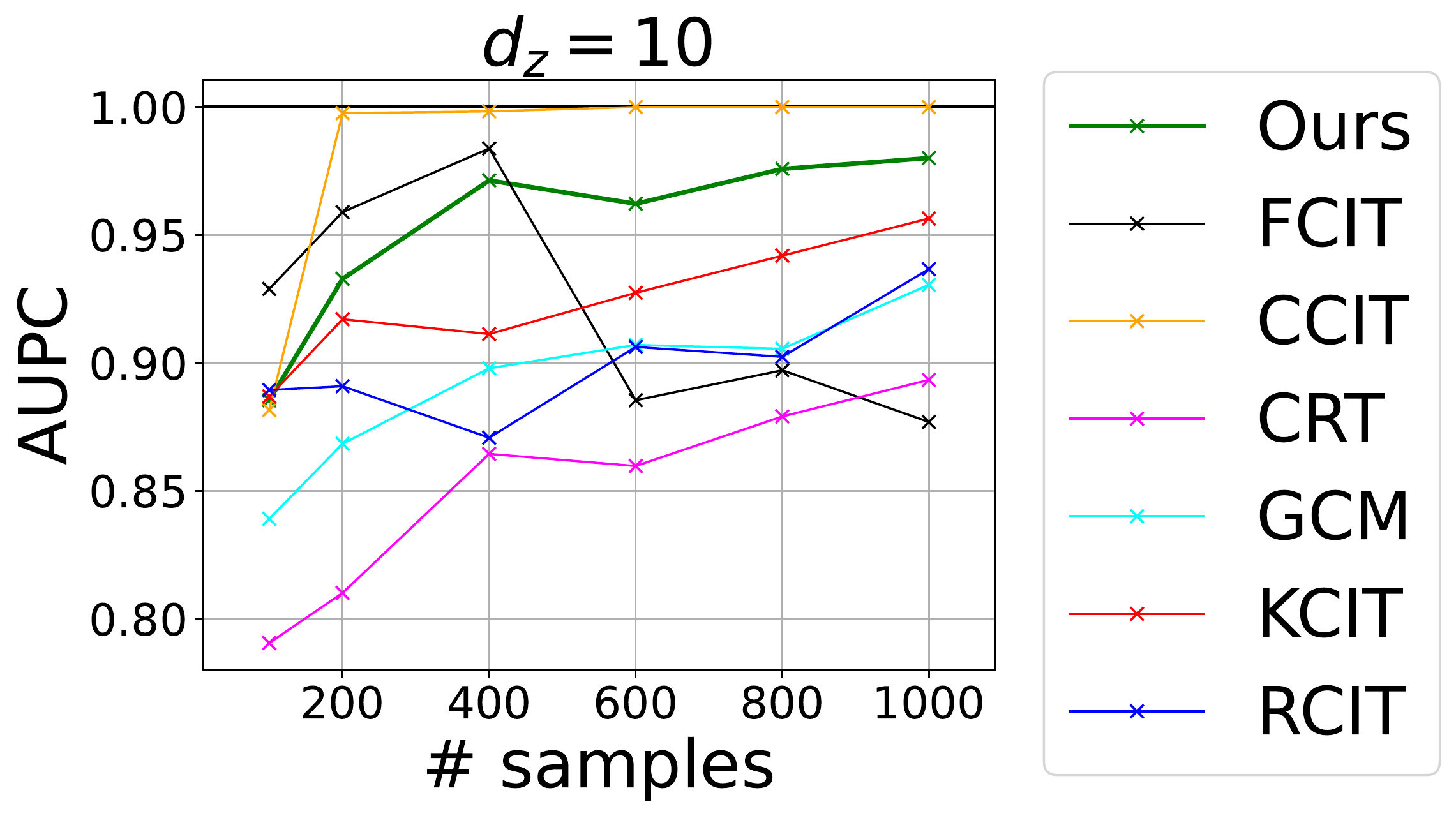} 
\end{tabular}
\caption{Comparison of the KS statistic and the AUPC of our testing procedure with other SoTA tests on the two problems presented in Eq.~\eqref{li-exp-h0} and Eq.~\eqref{li-exp-h1}  with Laplace noises. Each point in the figures is obtained by repeating the experiment for 100 independent trials. (\emph{Left, middle-left}): the KS statistic and AUPC (respectively) obtained by each test when varying the dimension $d_z$ from 1 to 10; here, the number of samples $n$ is fixed and equals $1000$. (\emph{Middle-right, right}): the KS and AUPC (respectively), obtained by each test when varying the number of samples $n$ from 100 to 1000; here, the dimension $d_z$ is fixed and equals $10$.
\label{fig-exp-li-ks-laplace}}
\vspace{-0.4cm}
\end{figure*}

\textbf{Effects of $p$ and $J$.} Our first experiment studies the effects of $p$ and $J$ on our proposed method. To do so, we follow the synthetic experiment proposed in~\cite{strobl2019approximate}. To evaluate the type-I error, we generate data that follows the model:
\begin{align}
\label{exp-strobl-h0}
    X=f_1(\varepsilon_x), \ Y=f_2(\varepsilon_y),~\text{and Z}\sim\mathcal{N}(0_d,I_{d_z}),
\end{align}
where $Z$, $\varepsilon_x$, and $\varepsilon_y$ are samples from jointly independent standard Gaussian or Laplace distributions, and $f_1$ and $f_2$ are smooth
functions chosen uniformly from the set $\{(\cdot), (\cdot)^2, (\cdot)^3, \tanh(\cdot), \exp(-|\cdot|)\}$. To compare the power of the tests, we also consider the model:
\begin{align}
\label{exp-strobl-h1}
    X=f_1(\varepsilon_x +0.8\varepsilon_b), Y=f_2(\varepsilon_y+0.8\varepsilon_b),
\end{align}
where $\varepsilon_b$ is sampled from a standard Gaussian or Laplace distribution. In Figure~\ref{fig-exp-param}, we compare the KS statistic and the AUPC of our method when varying $p$ and $J$. That figure shows that (i) our method is robust to the choice of $p$, and (ii) the performances of the test do not necessarily increase as $J$ increases. Armed with theses observations, in the following experiments, we always set $p=2$ and $J=5$ for our method.

\textbf{Effect of the rank $r$.} In this experiment, we investigate the effect of the rank regression $r$ on our proposed method both in terms of performance and time. For that purpose, in Figure~\ref{fig-rn-dependence}, we consider the two problems presented in~\eqref{exp-strobl-h0} and~\eqref{exp-strobl-h1}  with Gaussian noises and show the type-I and type-II when varying the ratio $r/n$ for multiple sample size $n$. We observe that the rank $r$ does not affect the power of the method, however we observe that the type-I error decreases as the ratio increases. Therefore the rank $r$ allows in practice to deal with the tradeoff between the computational time and the control of the type-I error. In the following experiment we always set $r=n$ for simplicity.

\textbf{Illustrations of our theoretical findings.} 
The following experiment confirms that validity of our theoretical results from Propositions~\ref{prop:oracle-law} and \ref{prop:norm-law}. For that purpose, we generate two synthetic data sets for which either $H_0$ or $H_1$ holds. Concretely, we define a first triplet $(X,Y,Z)$ as follows:
\begin{align}
\label{exp-illustration-h0}
X = \text{P}_1(Z) + \varepsilon_x,~~Y = \text{P}_1(Z) + \varepsilon_y.
\end{align}
Above, $\varepsilon_x$ and $\varepsilon_y$ follow two independent standard normal distributions, $Z\sim \mathcal{N}(0_{d_z},\Sigma)$ with $\Sigma\in\mathbb{R}^{d_z\times d_z}$. The covariance matrix $\Sigma$  is obtained by multiplying a random matrix whose entries are independent and follow standard normal distribution, by its transpose, and $\text{P}_1$ is a projection onto the first coordinate. As a result, in this case, we have that $X\perp Y \mid Z$. We also consider a modification of the above data generating function for which $H_1$ holds. This is done by adding a noise component $\varepsilon_b$ that is shared across $X$ and $Y$ as follows:
\begin{align}
\label{exp-illustration-h1}
X = \text{P}_1(Z) + \varepsilon_x + \varepsilon_b,~~ Y = \text{P}_1(Z) + \varepsilon_y + \varepsilon_b,
\end{align}
where $\varepsilon_b$ follows the standard normal distribution. Since we consider \emph{Gaussian kernels}, we can obtain an explicit formulation of $\mathbb{E}_{\ddot{X}}\left[k_{\mathcal{\ddot{X}}}(\mathbf{t}^{(1)}_j,\ddot{X})|Z=\cdot\right]$ and $\mathbb{E}_{Y}\left[k_{\mathcal{Y}}(t^{(2)}_j,Y)|Z=\cdot\right]$ for both data generation functions. See Appendix~\ref{sec-theoritical-findings} for more details. 
Consequently, we are able to compute both the normalized version of our oracle statistic $\widehat{\text{CI}}_{n,p}$ and our approximate normalized statistic $\widetilde{\text{NCI}}_{n,r,p}$. In Figure~\ref{fig-illustation-theory}, we show that both statistics manage to recover the asymptotic distribution under $H_0$, and reject the null hypothesis under $H_1$. In addition, we show that in the high dimensional setting, only our optimized version of $\widetilde{\text{NCI}}_{n,r,p}$---obtained by optimizing the hyperparameters involved in the RLS estimators of our statistic---manages to recover the asymptotic distribution under $H_0$.

\begin{figure*}[t!]
\begin{tabular}{cccc} 
\includegraphics[height=3.6cm]{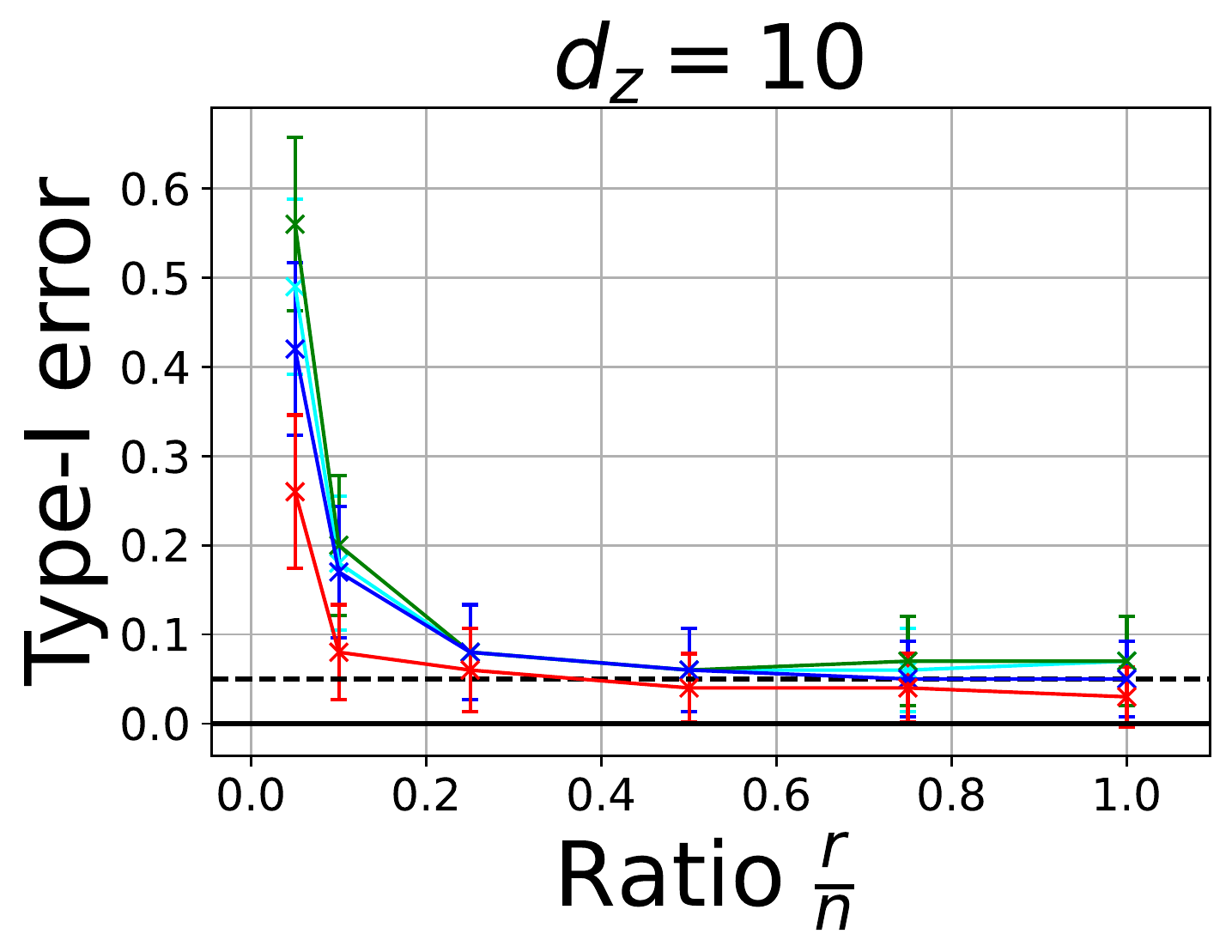}& \includegraphics[height=3.6cm]{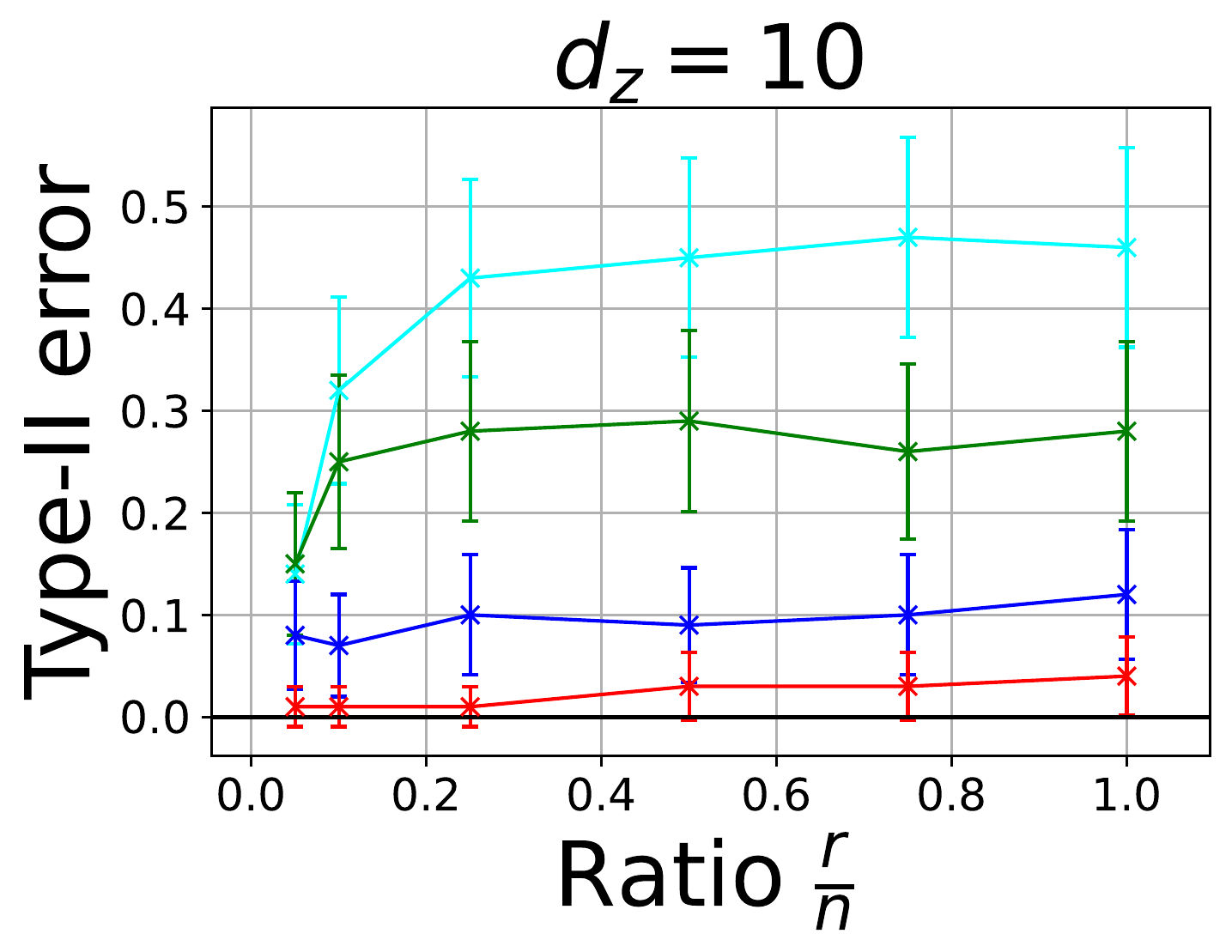}&
\includegraphics[height=3.6cm]{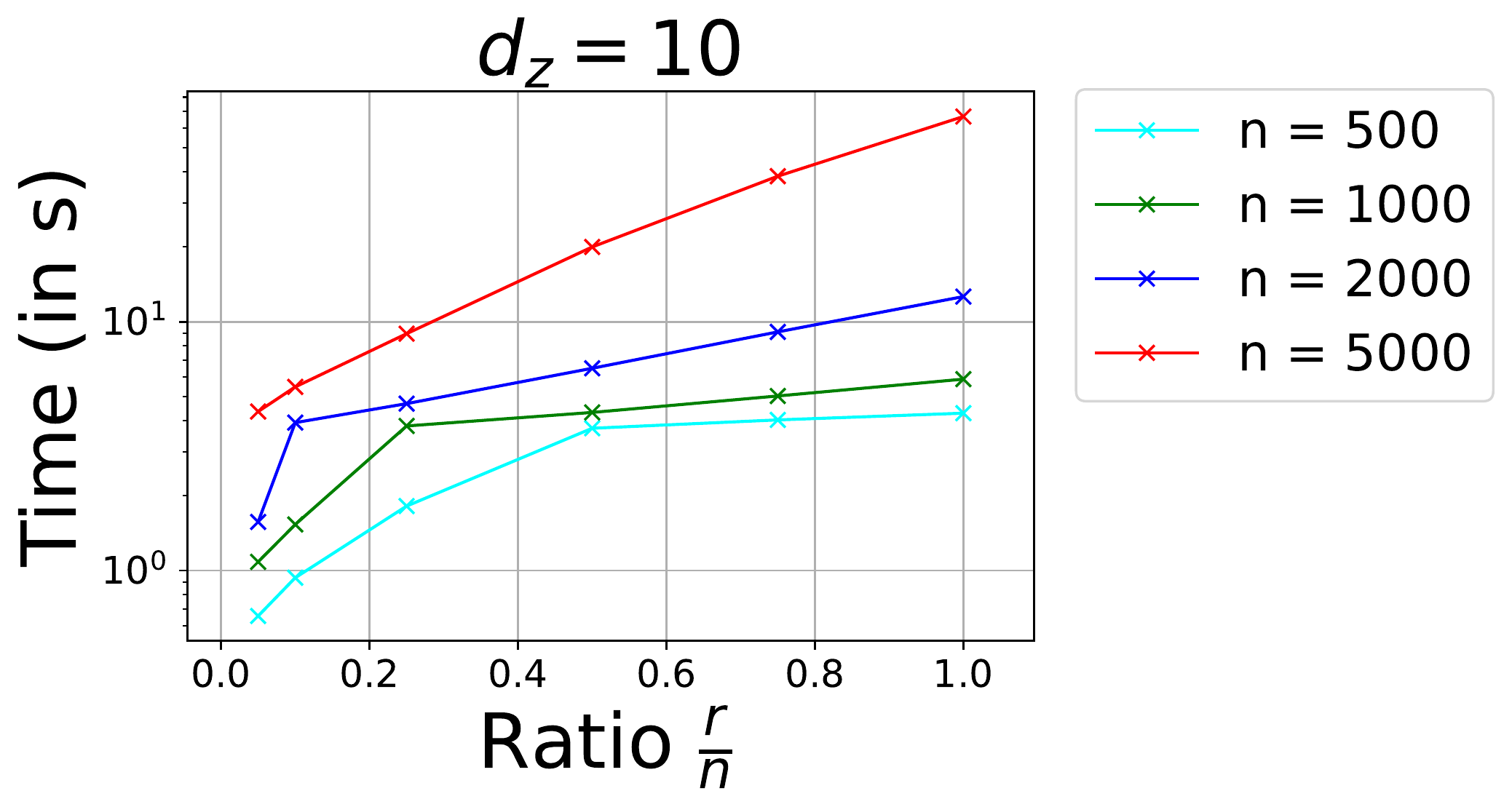}\\
\end{tabular}
\caption{Comparison of the type-I error at level $\alpha=0.05$ (dashed line) and the type-II error (lower is better) of our test procedure on the two problems presented in~\eqref{exp-strobl-h0} and~\eqref{exp-strobl-h1}  with Gaussian noises. Each point in the figures is obtained by repeating the experiment for 100 independent trials. (\emph{Left, Middle}): type-I and type-II errors obtained by each test when varying the ratio regression rank/total numbers of samples for different numbers of samples. (\emph{Right}): time in seconds (log-scale) to compute the statistic  when varying the ratio regression rank/total number of samples for different number of samples.
\label{fig-rn-dependence}}
\vspace{-0.3cm}
\end{figure*}

\textbf{Comparisons with existing tests.} In our next experiments, we compare the performance of our method (implemented with the optimized version of our statistic) with state-of-the-art techniques for conditional independence testing. We first study the two data generating functions from \eqref{exp-strobl-h0} and \eqref{exp-strobl-h1}. For each of these problems, we consider two settings. In the first, we fix the dimension $d_z$ while varying the number of samples $n$. In the second, we fix the number of samples while varying the dimension of the problem. To evaluate the performance of the tests, we compare the type-I errors at level $\alpha=0.05$ under the first model~\eqref{exp-strobl-h0}, and, for the second model~\eqref{exp-strobl-h1}, we evaluate the power of the test by presenting the type-II error. Figures~\ref{fig-exp-strobl-type} (Gaussian case) and \ref{fig-exp-strobl-laplace-supp} (Laplace case) demonstrate that our method consistently controls the type-I error and obtains a power similar to the best SoTA tests. In Figures~\ref{fig-exp-strobl-ks-supp} and \ref{fig-exp-strobl-ks-laplace-supp}, we compare the KS statistic and the AUPC of the different tests, and obtain similar conclusions. In addition, in Figure~\ref{fig-exp-strobl-type-mixture}, we consider the same setting as in Figure~\ref{fig-exp-strobl-ks-supp} where we samples noises randomly according to a non-symmetric mixture of Gaussians and obtain the same results. In Figure~\ref{fig-exp-strobl-highdim-gaussian-supp} and \ref{fig-exp-strobl-highdim-laplace-supp}, we investigate the high dimensional regime and show that our test is the only one which manages to control the type-I error while being competitive in term of power with other methods. See Appendix~\ref{sec-exp-storbl} for more details. 

\begin{figure}[h!]
\centering
\includegraphics[width=0.18\textwidth]{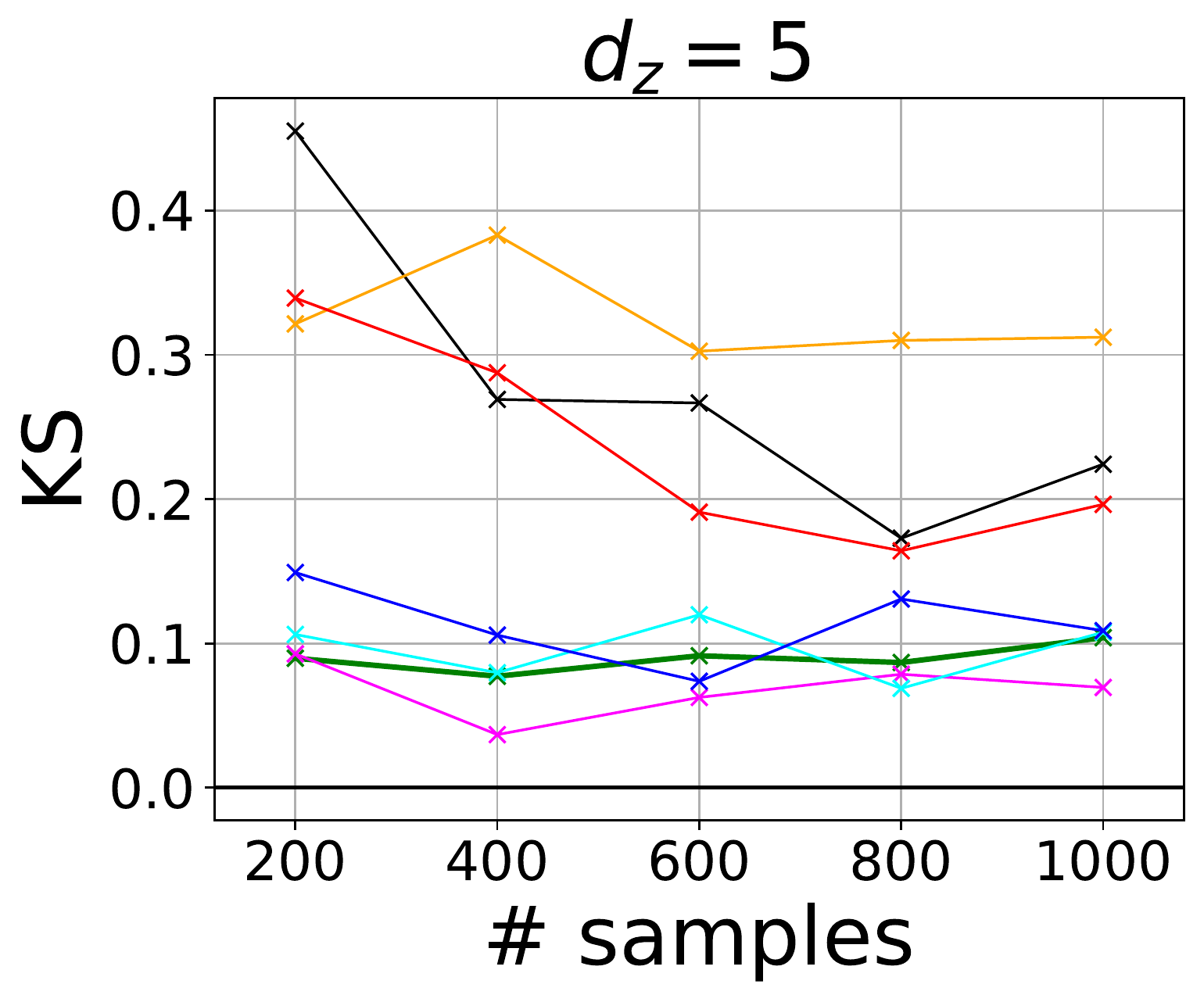}
\includegraphics[width=0.26\textwidth]{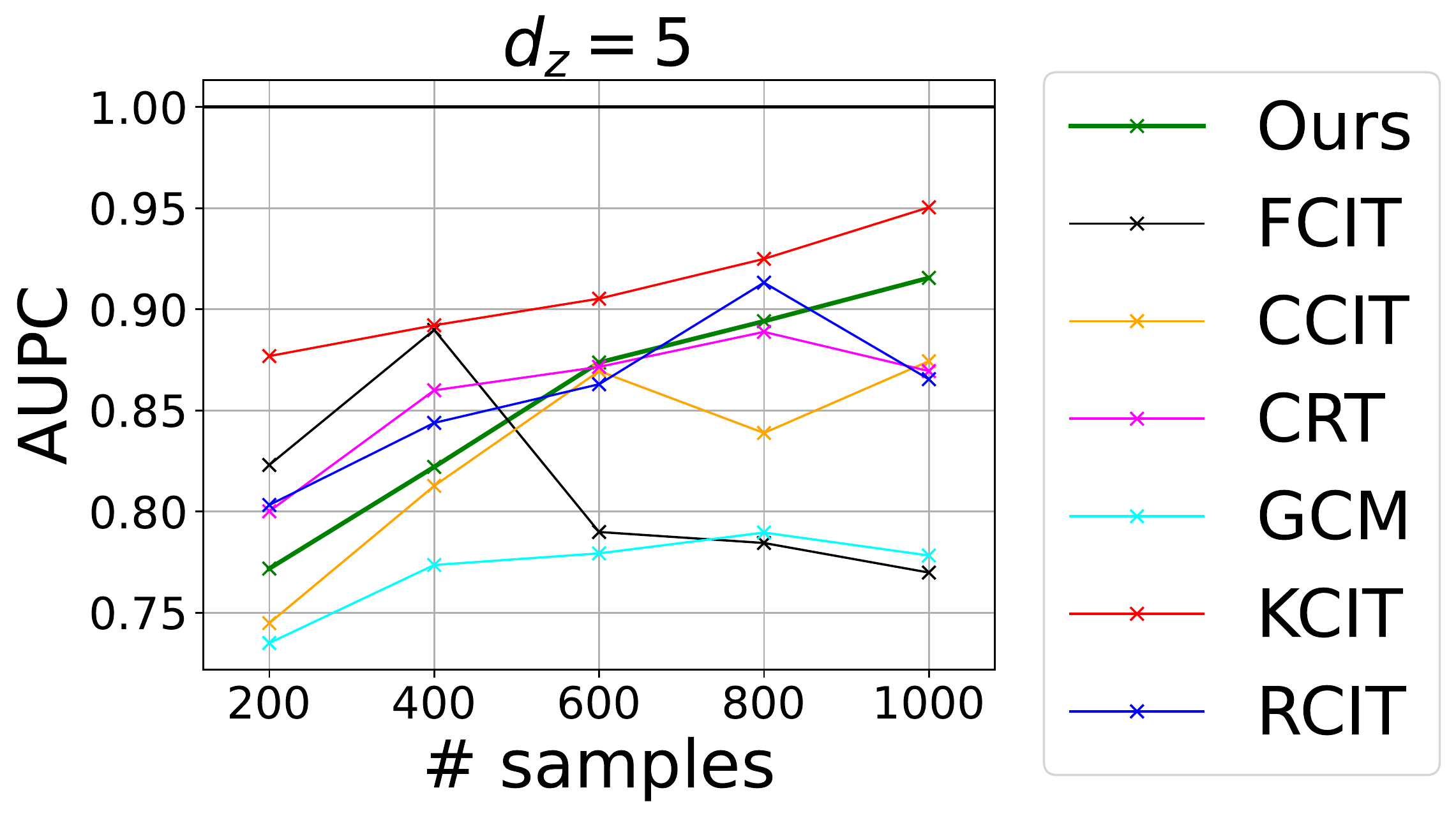}
\caption{In this experiment we compare the KS statistic and the AUPC of our test procedure with other SoTA tests on the two problems presented in~\eqref{exp-strobl-h0} and~\eqref{exp-strobl-h1} where noises are randomly sampled according to a non-symmetric mixture of Gaussians. Each point in the figures is obtained by repeating the experiment for 100 independent trials. Here we fix the dimension to be $d=5$ and we vary the number of samples $n$.
\label{fig-exp-strobl-type-mixture}}
\end{figure}

We now conduct another series of experiments that build upon the synthetic data sets presented in~\citep{zhang2012kernel,li2020nonparametric,doran2014permutation,BellotS19}.
To compare type-I error rates, we generate simulated data for which $H_0$ is true:
\begin{align}
\label{li-exp-h0}
X=f_1\left(\bar{Z}+\varepsilon_x\right),Y=f_2\left(\bar{Z}+\varepsilon_y\right).
\end{align}
Above, $\bar{Z}$ is the average of $Z=(Z_1,\cdots,Z_{d_z})$, $\varepsilon_x$ and $\varepsilon_y$ are sampled independently from a standard Gaussian or Laplace distribution, and $f_1$ and $f_2$ are smooth
functions chosen uniformly from the set $\{(\cdot), (\cdot)^2, (\cdot)^3, \tanh(\cdot), \exp(-|\cdot|)\}$. To evaluate the power, we consider the following data generating function:
\begin{align}
\begin{aligned}
\label{li-exp-h1}
 X=f_1\left( \bar{Z}+\varepsilon_x\right)+\varepsilon_b,Y=f_2\left(\bar{Z}+\varepsilon_y\right)+\varepsilon_b,
\end{aligned}
\end{align}
where $\varepsilon_b$ is a standard Gaussian or Laplace distribution.

As in the previous experiment, for each model, we study two settings by either fixing the dimension $d_z$, or the sample size $n$. In Figure~\ref{fig-exp-li-ks-laplace} (Laplace case) and~\ref{fig-exp-li-ks-gauss-supp} (Gaussian case), we compare the KS and the AUPC of our method with the SoTA tests and demonstrate that our procedure manages to be powerful while controlling the type-I error. In Figures~\ref{fig-exp-li-type-supp} and \ref{fig-exp-li-laplace-supp}, we also compare the type-I and type-II errors of the different tests, and obtain similar conclusions. In addition, we investigate the high dimensional regime and show in Figure~\ref{fig-exp-li-highdim-gauss-supp} and \ref{fig-exp-li-highdim-laplace-supp} that our test outperforms all the other proposed methods in most of the settings. See Appendix~\ref{sec-exp-li} for more details.

\paragraph{Conclusion.} We introduced a new kernel-based statistic for testing CI. We derived its asymptotic null distribution and designed a simple testing procedure that emerges from it. To our knowledge, we are the first article to propose an asymptotic test for CI with a tractable null distribution. Using various synthetic experiments, we demonstrated that our approach is competitive with other SoTA methods both in terms of type-I and type-II errors, even in the high dimensional setting.

\clearpage
\newpage
\section*{Acknowledgements}
M.S.~was supported by a "Chaire d’excellence de l’IDEX Paris Saclay". Y.R.~was supported by the Israel Science Foundation (grant 729/21). Y.R.~also thanks the Career Advancement Fellowship, Technion, for providing research support.
\bibliography{biblio}
\bibliographystyle{plainnat}
\clearpage
\appendix
    \onecolumn
\section*{Supplementary Material}
\section{Proofs}
\subsection{On the Formulation of the Witness Function}
\label{form-witness}
Let $(\mathbf{t}_j)_{j=1}^J$ sampled independently from the $\Gamma$ distribution, then by definition of $d_{p,J}(\cdot,\cdot)$, we have that
\begin{align*}
    d_{p,J}(P_{XZY},P_{\ddot{X}\otimes Y|Z}):=\left[\frac{1}{J}\sum_{j=1}^J \left|\mu_{P_{XZY},k_{\mathcal{\ddot{X}}}\cdot k_{\mathcal{Y}}}(\mathbf{t}_j)-\mu_{P_{\ddot{X}\otimes Y|Z},k_{\mathcal{\ddot{X}}}\cdot k_{\mathcal{Y}}}(\mathbf{t}_j)\right|^p\right]^{\frac{1}{p}},
\end{align*}
Moreover thanks to Assumption~\ref{assump-kernel}, we have that for any $(\mathbf{t}^{(1)},t^{(2)})\in\mathcal{\ddot{X}}\times\mathcal{Y}$ 
\begin{align*}
    \mu_{P_{\ddot{X}\otimes Y|Z},k_{\mathcal{\ddot{X}}}\cdot k_{\mathcal{Y}}}(\mathbf{t}^{(1)},t^{(2)})&=
    \mathbb{E}_{Z}\left[\mathbb{E}_{\ddot{X}}\left[k_{\mathcal{\ddot{X}}}(\mathbf{t}^{(1)},\ddot{X})|Z\right]
    \mathbb{E}_{Y}\left[k_{\mathcal{Y}}(t^{(2)},Y)|Z\right] \right]\; ,~~\text{and}\\
    \mu_{P_{XZY},k_{\mathcal{\ddot{X}}}\cdot k_{\mathcal{Y}}}(\mathbf{t}^{(1)},t^{(2)})&=
    \mathbb{E}\left[k_{\mathcal{\ddot{X}}}(\mathbf{t}^{(1)},\ddot{X})
 k_{\mathcal{Y}}(t^{(2)},Y) \right]\; .
\end{align*}

Let us now introduce the following witness function
\begin{align*}
    \Delta(\mathbf{t}^{(1)},t^{(2)}) :=\mathbb{E}\left[\left(k_{\mathcal{\ddot{X}}}(\mathbf{t}^{(1)},\ddot{X})- \mathbb{E}_{\ddot{X}}\left[k_{\mathcal{\ddot{X}}}(\mathbf{t}^{(1)},\ddot{X})|Z\right]\right)\times\left(k_{\mathcal{Y}}(t^{(2)},Y)- \mathbb{E}_{Y}\left[k_{\mathcal{Y}}(t^{(2)},Y)|Z\right]\right)\right]\;.
\end{align*}
Therefore we obtain that

\begin{align*}
    \Delta(\mathbf{t}^{(1)},t^{(2)}) &=\mathbb{E}\left[k_{\mathcal{\ddot{X}}}(\mathbf{t}^{(1)},\ddot{X})(k_{\mathcal{Y}}(t^{(2)},Y)\right]\\ 
    &- \mathbb{E}\left[k_{\mathcal{\ddot{X}}}(\mathbf{t}^{(1)},\ddot{X})\mathbb{E}_{Y}\left[k_{\mathcal{Y}}(t^{(2)},Y)|Z\right]\right]\\
    &+ \mathbb{E}\left[\mathbb{E}_{\ddot{X}}\left[k_{\mathcal{\ddot{X}}}(\mathbf{t}^{(1)},\ddot{X})|Z\right] \mathbb{E}_{Y}\left[k_{\mathcal{Y}}(t^{(2)},Y)|Z\right] \right]\\
   &- \mathbb{E}\left[\mathbb{E}_{\ddot{X}}\left[k_{\mathcal{\ddot{X}}}(\mathbf{t}^{(1)},\ddot{X})|Z\right]k_{\mathcal{Y}}(t^{(2)},Y)\right]\;.
\end{align*}
Now remark that
\begin{align*}
    \mathbb{E}\left[k_{\mathcal{\ddot{X}}}(\mathbf{t}^{(1)},\ddot{X})\mathbb{E}_{Y}\left[k_{\mathcal{Y}}(t^{(2)},Y)|Z\right]\right] &= \mathbb{E}\left[\mathbb{E}\left[k_{\mathcal{\ddot{X}}}(\mathbf{t}^{(1)},\ddot{X})\mathbb{E}_{Y}\left[k_{\mathcal{Y}}(t^{(2)},Y)|Z\right]\big|Z\right]\right]\\
    &=\mathbb{E}\left[\mathbb{E}_{Y}\left[k_{\mathcal{Y}}(t^{(2)},Y)|Z\right]  \mathbb{E}_{\ddot{X}}\left[k_{\mathcal{\ddot{X}}}(\mathbf{t}^{(1)},\ddot{X})|Z\right]\right]\;.
\end{align*}
Similarly, we have that 
\begin{align*}
\mathbb{E}\left[\mathbb{E}_{\ddot{X}}\left[k_{\mathcal{\ddot{X}}}(\mathbf{t}^{(1)},\ddot{X})|Z\right]k_{\mathcal{Y}}(t^{(2)},Y)\right] = \mathbb{E}\left[\mathbb{E}_{Y}\left[k_{\mathcal{Y}}(t^{(2)},Y)|Z\right]  \mathbb{E}_{\ddot{X}}\left[k_{\mathcal{\ddot{X}}}(\mathbf{t}^{(1)},\ddot{X})|Z\right]\right]
\end{align*}
from which follows that 
\begin{align*}
    \Delta(\mathbf{t}^{(1)},t^{(2)}) & = \mathbb{E}\left[k_{\mathcal{\ddot{X}}}(\mathbf{t}^{(1)},\ddot{X})(k_{\mathcal{Y}}(t^{(2)},Y)\right] - \mathbb{E}\left[\mathbb{E}_{Y}\left[k_{\mathcal{Y}}(t^{(2)},Y)|Z\right]  \mathbb{E}_{\ddot{X}}\left[k_{\mathcal{\ddot{X}}}(\mathbf{t}^{(1)},\ddot{X})|Z\right]\right]\\
    &=  \mu_{P_{XZY},k_{\mathcal{\ddot{X}}}\cdot k_{\mathcal{Y}}}(\mathbf{t}^{(1)},t^{(2)}) - \mu_{P_{\ddot{X}\otimes Y|Z},k_{\mathcal{\ddot{X}}}\cdot k_{\mathcal{Y}}}(\mathbf{t}^{(1)},t^{(2)})\;.
\end{align*}

\subsection{Proof of Proposition~\lowercase{\ref{prop:rls-law}}}
\label{prv:rls-law}

\begin{proof} For all $j\in[J]$:
    
\begin{align}
  \sqrt{n}\widetilde{\Delta}_{n,r}(\mathbf{t}^{(1)}_j,t^{(2)}_j)&= \sqrt{n}\frac{1}{n}\sum_{i=1}^n  \left(k_{\mathcal{\ddot{X}}}(\mathbf{t}^{(1)}_j,\ddot{x}_i)- h^{(1)}_{j,r}(z_i)\right)\left(k_{\mathcal{Y}}(t^{(2)}_j,y_i)- h^{(2)}_{j,r}(z_i)\right)\nonumber\\
  &= \sqrt{n}\Delta_{n}(\mathbf{t}^{(1)}_j,t^{(2)}_j)\label{eq:term-tcl}\\
  &+\sqrt{n} \frac1n\sum_{i=1}^n\left(k_{\mathcal{\ddot{X}}}(\mathbf{t}^{(1)}_j,\ddot{x}_i)-\mathbb{E}_{\ddot{X}}\left[k_{\mathcal{\ddot{X}}}(\mathbf{t}^{(1)}_j,\ddot{X})|Z=z_i\right]\right)\left(\mathbb{E}_{Y}\left[k_{\mathcal{Y}}(t^{(2)}_j,Y)|Z=z_i\right]- h^{(2)}_{j,r}(z_i)\right)\label{eq:term-cross1}\\
  &+\sqrt{n}\frac1n\sum_{i=1}^n\left(\mathbb{E}_{\ddot{X}}\left[k_{\mathcal{\ddot{X}}}(\mathbf{t}^{(1)}_j,\ddot{X})|Z=z_i\right]-h^{(1)}_{j,r}(z_i)\right)\left(k_{\mathcal{Y}}(t^{(2)}_j,y_i)-\mathbb{E}_{Y}\left[k_{\mathcal{Y}}(t^{(2)}_j,Y)|Z=z_i\right]\right)\label{eq:term-cross2}\\
  &+\sqrt{n}\frac1n\sum_{i=1}^n\left(\mathbb{E}_{\ddot{X}}\left[k_{\mathcal{\ddot{X}}}(\mathbf{t}^{(1)}_j,\ddot{X})|Z=z_i\right]-h^{(1)}_{j,r}(z_i)\right)\left(\mathbb{E}_{Y}\left[k_{\mathcal{Y}}(t^{(2)}_j,Y)|Z=z_i\right]-h^{(2)}_{j,r}(z_i)\right)\label{eq:term-cross3}
\end{align} 

Let us treat the four terms of this decomposition. The term~\eqref{eq:term-tcl} has been treated by Propostion~\ref{prop:oracle-law}, and satisfies, under the null hypothesis $H_0$
\begin{align*}
\sqrt{n}\Delta_{n}(\mathbf{t}_j^{(1)},t_j^{(2)})\to_{n\to\infty} \mathcal{N}\left(0,\mathbb{E}\left[\left(k_{\mathcal{\ddot{X}}}(\mathbf{t}^{(1)}_j,\ddot{X})-\mathbb{E}_{\ddot{X}}\left[k_{\mathcal{\ddot{X}}}(\mathbf{t}^{(1)}_j,\ddot{X})|Z\right]\right)\left(k_{\mathcal{Y}}(t^{(2)}_j,Y)-\mathbb{E}_{Y}\left[k_{\mathcal{Y}}(t^{(2)}_j,Y)|Z\right]\right)\right]\right)\;.
\end{align*}

Let us now show that the last term~\eqref{eq:term-cross3} converges towards $0$ in probability. Let us denote for all $j$, $e^{(1)}_{j}:z\to\mathbb{E}_{\ddot{X}}\left[k_{\mathcal{\ddot{X}}}(\mathbf{t}^{(1)}_j,\ddot{X})|Z=z\right] $ and $e^{(2)}_{j}:z\to\mathbb{E}_{Y}\left[k_{\mathcal{\ddot{X}}}(t^{(2)}_j,Y)|Z=z\right]$, both elements of $ H_{\mathcal{Z}}$ by Assumption~\ref{ass:source}. Then we have, for all $i\in[n]$:

\begin{align*}
   \left(e^{(1)}_{j}(z_i)-h^{(1)}_{j,r}(z_i)\right)\left(e^{(2)}_{j}(z_i)-h^{(2)}_{j,r}(z_i)\right)=\langle\left(e^{(1)}_{j}-h^{(1)}_{j,r}\right)\otimes\left(e^{(2)}_{j}-h^{(2)}_{j,r}\right), k_{\mathcal{Z}}(z_i,\cdot)\otimes k_{\mathcal{Z}}(z_i,\cdot)\rangle.
\end{align*}
Then we deduce, by denoting: $\mu_{ZZ} := \mathbb{E}\left[ k_{\mathcal{Z}}(Z,\cdot)k_{\mathcal{Z}}(Z,\cdot)\right]$ and $\hat{\mu}_{ZZ}:=\frac1n\sum_{i=1}^n k_{\mathcal{Z}}(z_i,\cdot)k_{\mathcal{Z}}(z_i,\cdot)$, that
\begin{align*}
    \frac1n\sum_{i=1}^n\left(\mathbb{E}_{\ddot{X}}\left[k_{\mathcal{\ddot{X}}}(\mathbf{t}^{(1)}_j,\ddot{X})|Z=z_i\right]-h^{(1)}_{j,r}(z_i)\right)\left(\mathbb{E}_{Y}\left[k_{\mathcal{Y}}(t^{(2)}_j,Y)|Z=z_i\right]-h^{(2)}_{j,r}(z_i)\right)\\
    = \langle\left(e^{(1)}_{j}-h^{(1)}_{j,r}\right)\otimes\left(e^{(2)}_{j}-h^{(2)}_{j,r}\right),\frac{1}{n}\sum_{i=1}^n k_{\mathcal{Z}}(z_i,\cdot)\otimes k_{\mathcal{Z}}(z_i,\cdot)\rangle\\
    =\langle\left(e^{(1)}_{j}-h^{(1)}_{j,r}\right)\otimes\left(e^{(2)}_{j}-h^{(2)}_{j,r}\right),\mu_{ZZ}\rangle+\langle\left(e^{(1)}_{j}-h^{(1)}_{j,r}\right)\otimes\left(e^{(2)}_{j}-h^{(2)}_{j,r}\right),\hat{\mu}_{ZZ}-\mu_{ZZ}\rangle\;.
\end{align*}
Then remark that:
\begin{align*}
    \lvert\langle\left(e^{(1)}_{j}-h^{(1)}_{j,r}\right)\otimes\left(e^{(2)}_{j}-h^{(2)}_{j,r}\right),\mu_{ZZ}\rangle\rvert&=\lvert\mathbb{E}_Z\left[\left(e^{(1)}_{j}(Z)-h^{(1)}_{j,r}(Z)\right)\left(e^{(2)}_{j}(Z)-h^{(2)}_{j,r}(Z)\right)\right]\rvert\\
    &\leq \lVert e^{(1)}_{j}-h^{(1)}_{j,r}\rVert_{L^2(P_Z)} \lVert  e^{(2)}_{j}-h^{(2)}_{j,r}\rVert_{L^2(P_Z)} \;.
\end{align*}
Under the Assumptions~\ref{ass:spectrum}-\ref{ass:source}, for $\lambda_{r} = \frac{1}{r^{\beta+\gamma}}$, we have, using the results from~\citep[Theorem 1]{fischer2020sobolev}: $\lVert  e^{(1)}_{j}-h^{(1)}_{j,r}\rVert_{L^2(P_Z)}^2\leq \frac{C\tau^2}{r^{\frac{\beta}{\beta+\gamma}}}$ with probability $1-4e^{-\tau}$ and  $ \lVert e^{(2)}_{j}-h^{(2)}_{j,r}\rVert_{L^2(P_Z)}^2 \leq \frac{C\tau^2}{r^{\frac{\beta}{\beta+\gamma}}}$ with probability $1-4e^{-\tau}$, for some constant $C$ independent from $n$ and $\tau$. then by union bound, we deduce with probability $1-8e^{-\tau}$ we have:
\begin{align*}
    \sqrt{n}\lvert\langle\left(e^{(1)}_{j}-h^{(1)}_{j,r}\right)\otimes\left(e^{(2)}_{j}-h^{(2)}_{j,r}\right),\mu_{ZZ}\rangle\rvert\leq \sqrt{n}\frac{C^2\tau^4}{r^{\frac{\beta}{\beta+\gamma}}}\;.
\end{align*}

Then, if $\sqrt{n}\in o(r^{\frac{\beta}{\beta+\gamma}})$, we have:
$ \sqrt{n}\lvert\langle\left(e^{(1)}_{j}-h^{(1)}_{j,r}\right)\otimes\left(e^{(2)}_{j}-h^{(2)}_{j,r}\right),\mu_{ZZ}\rangle\rvert\to 0$ in probability when $n\to\infty$. Moreover:
\begin{align*}
    \lvert \left(e^{(1)}_{j}-h^{(1)}_{j,r}\right)\otimes\left(e^{(2)}_{j}-h^{(2)}_{j,r}\right),\hat{\mu}_{ZZ}-\mu_{ZZ}\rangle\rvert\leq \lVert e^{(1)}_{j}-h^{(1)}_{j,r}\rVert_{ H_\mathcal{Z}} \lVert e^{(2)}_{j}-h^{(2)}_{j,r}\rVert_{ H_\mathcal{Z}}\lVert\hat{\mu}_{ZZ}-\mu_{ZZ}\rVert_{ H_\mathcal{Z}\otimes H_\mathcal{Z}}\; ,
\end{align*}
and by Markov inequality, $\lVert\hat{\mu}_{ZZ}-\mu_{ZZ}\rVert_{H_\mathcal{Z}\otimes H_\mathcal{Z}} \leq \sqrt{\frac{C'}{n\delta}}$ with probability $1-\delta$ for some constant $C'$. Moreover, under Assumption~\ref{ass:spectrum}-\ref{ass:source}, we have $\lVert e^{(1)}_{j}-h^{(1)}_{j,r}\rVert_{ H_\mathcal{Z}}\to 0$ and $\lVert e^{(2)}_{j}-h^{(2)}_{j,r}\rVert_{ H_\mathcal{Z}}\to 0$ in probability. Then, we deduce that $\sqrt{n}\lvert \langle\left(e^{(1)}_{j}-h^{(1)}_{j,r}\right)\otimes\left(e^{(2)}_{j}-h^{(2)}_{j,r}\right),\hat{\mu}_{ZZ}-\mu_{ZZ}\rangle\rvert\to 0$ in probability. Finally, the term~\eqref{eq:term-cross3} goes to $0$ in probability.

The terms~\eqref{eq:term-cross1} and~\eqref{eq:term-cross2} are similar and can be treated the same way. We only focus on the term~\eqref{eq:term-cross1}. For all $i\in[n]$:
\begin{align*}
\lvert\frac1n\sum_{i=1}^n\left(k_{\mathcal{\ddot{X}}}(\mathbf{t}^{(1)}_j,\ddot{x}_i)-\mathbb{E}_{\ddot{X}}\left[k_{\mathcal{\ddot{X}}}(\mathbf{t}^{(1)}_j,\ddot{X})|Z=z_i\right]\right)\left(\mathbb{E}_{Y}\left[k_{\mathcal{Y}}(t^{(2)}_j,Y)|Z=z_i\right]- h^{(2)}_{j,r}(z_i)\right)\rvert\\
  =\lvert\frac1n\sum_{i=1}^n\langle k_{\mathcal{\ddot{X}}}(t^{(1)}_j,\cdot),k_{\mathcal{\ddot{X}}}(\ddot{x}_i,\cdot)-\mathbb{E}_{\ddot{X}}\left[k_{\mathcal{\ddot{X}}}(\ddot{X},\cdot)|Z=z_i\right]\rangle_{ H_\mathcal{\ddot{X}}}
  \langle e^{(2)}_{j}-h^{(2)}_{j,r},k_{\mathcal{Z}}(z_i,\cdot)\rangle_{ H_\mathcal{Z}}\rvert\\
  =\lvert\frac1n\sum_{i=1}^n\langle k_{\mathcal{\ddot{X}}}(t^{(1)},\cdot)\otimes \left(e^{(2)}_{j}-h^{(2)}_{j,r}\right),\left(k_{\mathcal{\ddot{X}}}(\ddot{x}_i,\cdot)-\mathbb{E}_{\ddot{X}}\left[k_{\mathcal{\ddot{X}}}(\ddot{X},\cdot)|Z=z_i\right]\right)\otimes k_{\mathcal{Z}}(z_i,\cdot)\rangle_{ H_\mathcal{\ddot{X}}\otimes H_\mathcal{Z}}\rvert\\
  =\lvert\langle k_{\mathcal{\ddot{X}}}(t^{(1)},\cdot)\otimes \left(e^{(2)}_{j}-h^{(2)}_{j,r}\right),\frac1n\sum_{i=1}^n\left(k_{\mathcal{\ddot{X}}}(\ddot{x}_i,\cdot)-\mathbb{E}_{\ddot{X}}\left[k_{\mathcal{\ddot{X}}}(\ddot{X},\cdot)|Z=z_i\right]\right)\otimes k_{\mathcal{Z}}(z_i,\cdot)\rangle_{ H_\mathcal{\ddot{X}}\otimes H_\mathcal{Z}}\rvert\\
  \leq \lVert k_{\mathcal{\ddot{X}}}(t^{(1)},\cdot)\rVert_{ H_\mathcal{\ddot{X}}}\lVert e^{(2)}_{j}-h^{(2)}_{j,r}\rVert_{H_\mathcal{Z}}\left(\lVert \hat{\mu}^1_{\ddot{X}Z}-\mu_{\ddot{X}Z}\rVert_{ H_\mathcal{\ddot{X}}\otimes H_\mathcal{Z}}+ \lVert\hat{\mu}^2_{\ddot{X}}-\mu_{\ddot{X}Z}\rVert_{ H_\mathcal{\ddot{X}}\otimes H_\mathcal{Z}}\right)
\end{align*}
where: $\hat{\mu}^1_{\ddot{X}Z} := \frac1n\sum_{i=1}^nk_{\mathcal{\ddot{X}}}(\ddot{x}_i,\cdot)\otimes k_{\mathcal{Z}}(z_i,\cdot)$, $\hat{\mu}^2_{\ddot{X}Z}:=\frac1n\sum_{i=1}^n\mathbb{E}_{\ddot{X}}\left[k_{\mathcal{\ddot{X}}}(\ddot{X},\cdot)|Z=z_i\right]\otimes k_{\mathcal{Z}}(z_i,\cdot)$, and $\mu_{\ddot{X}Z}:=\mathbb{E}\left[k_{\mathcal{Y}}(y,\cdot)k_{\mathcal{Z}}(z,\cdot)\right]$.

By the law of large numbers, we have: $\hat{\mu}^1_{\ddot{X}Z}$ and $\hat{\mu}^2_{\ddot{X}Z}$ converge almost surely towards $\mu_{\ddot{X}Z}$. Moreover by Markov inequality, $\lVert\hat{\mu}^1_{\ddot{X}Z}-\mu_{\ddot{X}Z}\rVert_{ H_\mathcal{\ddot{X}}\otimes H_\mathcal{Z}} \leq \sqrt{\frac{C}{n\delta}}$ with probability $1-\delta$, and  $\lVert\hat{\mu}^2_{\ddot{X}Z}-\mu_{\ddot{X}Z}\rVert_{ H_\mathcal{\ddot{X}}\otimes H_\mathcal{Z}} \leq \sqrt{\frac{C}{n\delta}}$ with probability $1-\delta$. Then with probability $1-2\delta$, $\sqrt{n}\left(\lVert\hat{\mu}^1_{\ddot{X}Z}-\mu_{\ddot{X}Z}\rVert_{ H_\mathcal{\ddot{X}}\otimes H_\mathcal{Z}}+\lVert\hat{\mu}^2_{\ddot{X}Z}-\mu_{\ddot{X}Z}\rVert_{ H_\mathcal{\ddot{X}}\otimes H_\mathcal{Z}} \right)\leq2\sqrt{\frac{C}{\delta}}$. Moreover, under Assumption~\ref{ass:spectrum}-\ref{ass:source}, using the results from~\cite{fischer2020sobolev}, we have that $\lVert e^{(2)}_{j}-h^{(2)}_{j,r}\rVert_{H_\mathcal{Z}}$ converges towards $0$ in probability. Then the term~\eqref{eq:term-cross1}
converges in probability towards $0$. The same reasoning holds for~\eqref{eq:term-cross2}. 

Finally, by Slutsky's Lemma:
\begin{align*}
    \sqrt{n}\widetilde{\Delta}_{n,r}(\mathbf{t}^{(1)}_j,t^{(2)}_j)\to_{n\to\infty} \mathcal{N}\left(0,\mathbb{E}\left[\left(k_{\mathcal{\ddot{X}}}(\mathbf{t}^{(1)}_j,\ddot{X})-\mathbb{E}_{\ddot{X}}\left[k_{\mathcal{\ddot{X}}}(\mathbf{t}^{(1)}_j,\ddot{X})|Z\right]\right)\left(k_{\mathcal{Y}}(t^{(2)}_j,Y)-\mathbb{E}_{Y}\left[k_{\mathcal{Y}}(t^{(2)}_j,Y)|Z\right]\right)\right]\right).
\end{align*}
Now we have $\widetilde{\mathbf{S}}_{n,r}=\left(\widetilde{\Delta}_{n,r}(\mathbf{t}^{(1)}_j,t^{(2)}_j)\right)_{j\in[J]}=\left(\Delta_{n}(\mathbf{t}^{(1)}_j,t^{(2)}_j)\right)_{j\in[J]}+ \left(\widetilde{\Delta}_{n,r}(\mathbf{t}^{(1)}_j,t^{(2)}_j)-\Delta_{n}(\mathbf{t}^{(1)}_j,t^{(2)}_j)\right)_{j\in[J]}$ and  we have shown that $\sqrt{n}\left(\widetilde{\Delta}_{n,r_n}(\mathbf{t}^{(1)}_j,t^{(2)}_j)-\Delta_{n}(\mathbf{t}^{(1)}_j,t^{(2)}_j)\right)_{j\in[J]}$ goes to $0$ in probability. Then by Slutsky's Lemma and Proposition~\ref{prop:oracle-law}, we get: $\widetilde{\mathbf{S}}_{n,r_n}\to\mathcal{N}\left(0,\bm{\Sigma}\right)$.

Let $r>0$. Under $H_1$, $\mathbf{\mathbf{S}}_{n,r_n}\to \mathbf{S}\neq 0$. Let us consider a realization of $(\mathbf{t}^{(1)}_j,t_j^{(2)})_{j\in[J]}$ such that $\lVert \mathbf{S}\rVert_p\neq 0$. So $P(n^{p/2}\lVert \mathbf{S}_{n,r_n}\rVert_p\geq r)\to 1$ as $n\to\infty$ because $\lVert \mathbf{S}\rVert_p\neq 0$.
\end{proof}

\subsection{Proof of Proposition~\lowercase{\ref{prop:norm-law}}}
\label{prv:norm-law}
\begin{proof}
First notice that:
\begin{align*}
  \bm{\widetilde{\Sigma}}_{n,r}&:=\frac{1}{n}\sum_{i=1}^n \widetilde{\mathbf{u}}_{i,r}\widetilde{\mathbf{u}}_{i,r}^T+\delta_n \text{Id}_J  \\
  &=  \bm{\widehat{\Sigma}}_{n}+\frac1n\sum_{i=1}^n \widehat{\mathbf{u}}_{i}\left(\widetilde{\mathbf{u}}_{i,r}-\widehat{\mathbf{u}}_{r}\right)^T+\frac1n\sum_{i=1}^n \left(\widetilde{\mathbf{u}}_{i,r}-\widehat{\mathbf{u}}_{r}\right)\widehat{\mathbf{u}}_{i}^T+\frac1n\sum_{i=1}^n\left(\widetilde{\mathbf{u}}_{i,r}-\widehat{\mathbf{u}}_{r}\right) \left(\widetilde{\mathbf{u}}_{i,r}-\widehat{\mathbf{u}}_{r}\right)^T+\delta_n \text{Id}_J\;.
\end{align*}

By the law of large numbers, we get that under $H_0$: $\bm{\widehat{\Sigma}}_{n}\to\bm{\Sigma}$. Moreover:
\begin{align*}
    \left[\frac1n\sum_{i=1}^n \widehat{\mathbf{u}}_{i}\left(\widetilde{\mathbf{u}}_{i,r}-\widehat{\mathbf{u}}_{r}\right)^T\right]_{kl} = \frac{1}{n}\sum_{i=1}^n\left(k_{\mathcal{Y}}(t_k^{(2)},y_i)-\mathbb{E}_{Y}\left[k_{\mathcal{Y}}(t_k^{(2)},Y)|Z=z_i\right]\right)\left(\mathbb{E}_{\ddot{X}}\left[k_{\mathcal{\ddot{X}}}(\mathbf{t}_l^{(1)},\ddot{X})|Z=z_i\right]-h^{(1)}_{l,r}(z_i)\right)
\end{align*}
which has been proven to converge in probability to $0$ in the proof of Proposition~\ref{prop:rls-law}. Then $\frac1n\sum_{i=1}^n \widehat{\mathbf{u}}_{i}\left(\widetilde{\mathbf{u}}_{i,r}-\widehat{\mathbf{u}}_{r}\right)^T$ converges in probability to $0$. Similarly $\frac1n\sum_{i=1}^n \left(\widetilde{\mathbf{u}}_{i,r}-\widehat{\mathbf{u}}_{r}\right)\widehat{\mathbf{u}}_{i}^T$ and  $\frac1n\sum_{i=1}^n\left(\widetilde{\mathbf{u}}_{i,r}-\widehat{\mathbf{u}}_{r}\right) \left(\widetilde{\mathbf{u}}_{i,r}-\widehat{\mathbf{u}}_{r}\right)^T$ also converge in probability to $0$. Then by Slutsky's Lemma, $ \bm{\widetilde{\Sigma}}_{n,r}$ converges in probability to $\bm{\Sigma}$. By Slutsky's Lemma (again) and by Propostion~\ref{prop:rls-law}, we have that:
$\bm{\widetilde{\Sigma}}_{n,r}^{-1}\bm{\widetilde{S}}_{n,r}$ converges to a standard gaussian distribution $\mathcal{N}(0,\text{Id})$. The second part of the proposition is the same as the proof of Proposition~\ref{prop:rls-law}.
\end{proof}
\section{On the computation of Oracle statistic in Figure~\ref{fig-illustation-theory}}
\label{sec-theoritical-findings}

To compute the oracle statistic we needed to compute exactly the conditional expectation implied in our statistic. In the case of gaussian kernels and gaussian distributed data for $Z$, the computation of this conditional expectation is reduced to the computation of moment-generating function of a non-centered $\chi^2$ distribution.

\label{sec-rank-rn}

\newpage
\subsection{Additional experiments on Problems~\eqref{exp-strobl-h0} and~\eqref{exp-strobl-h1}}
\label{sec-exp-storbl}

\subsubsection{Gaussian Case}

\begin{figure*}[h]
\begin{tabular}{cccc} 
\includegraphics[height=2.9cm]{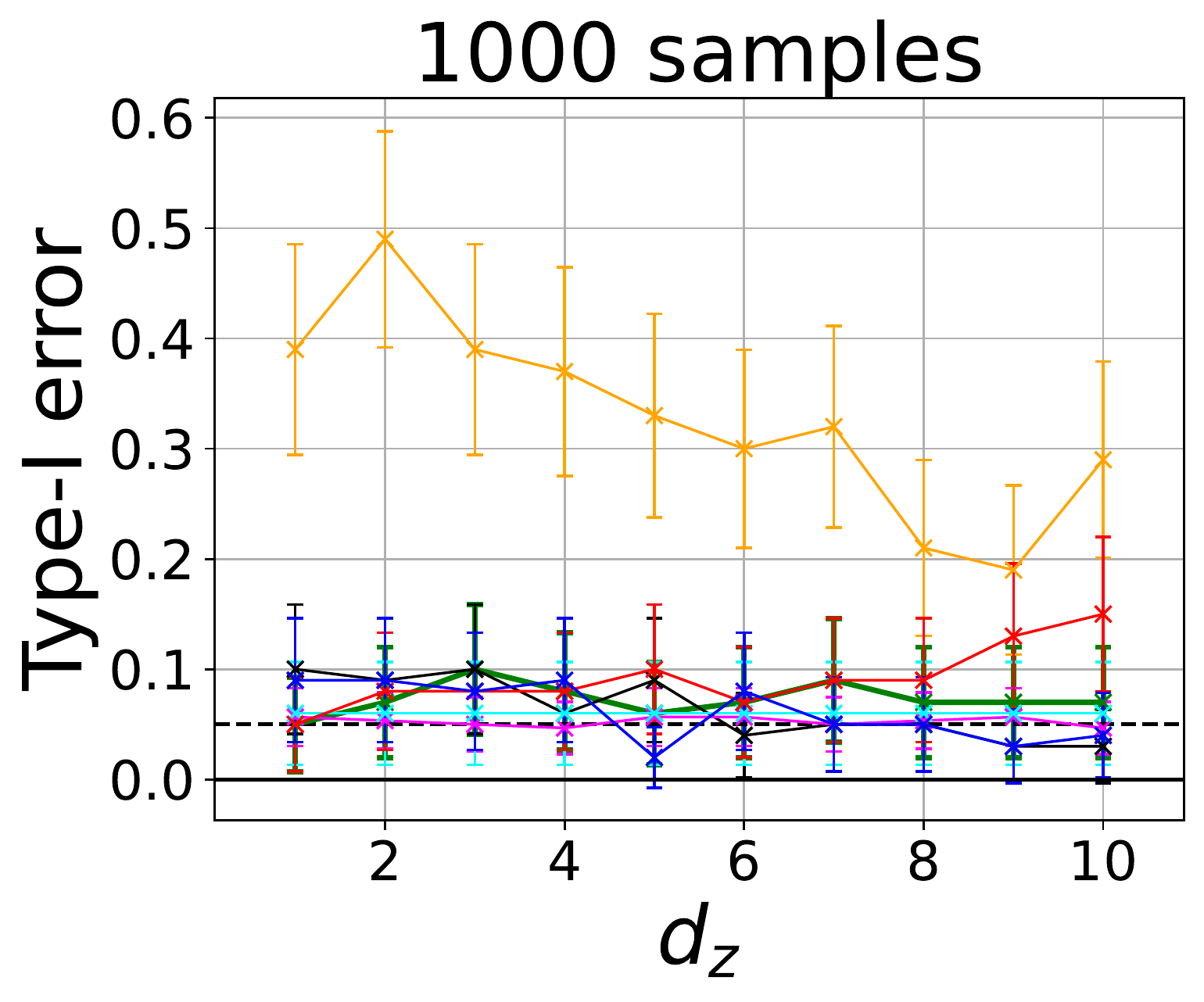}& \includegraphics[height=2.9cm]{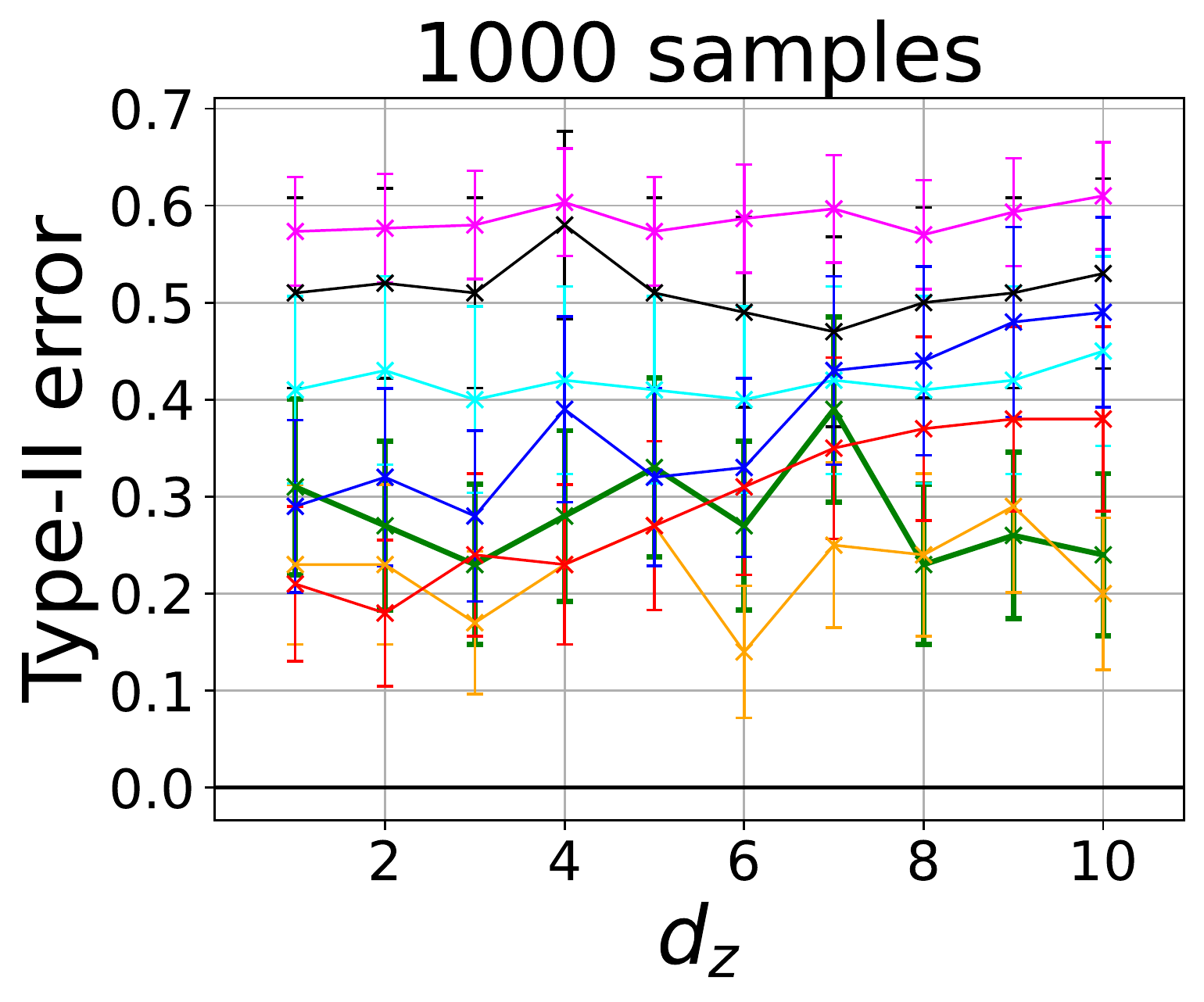} & 
\includegraphics[height=2.9cm]{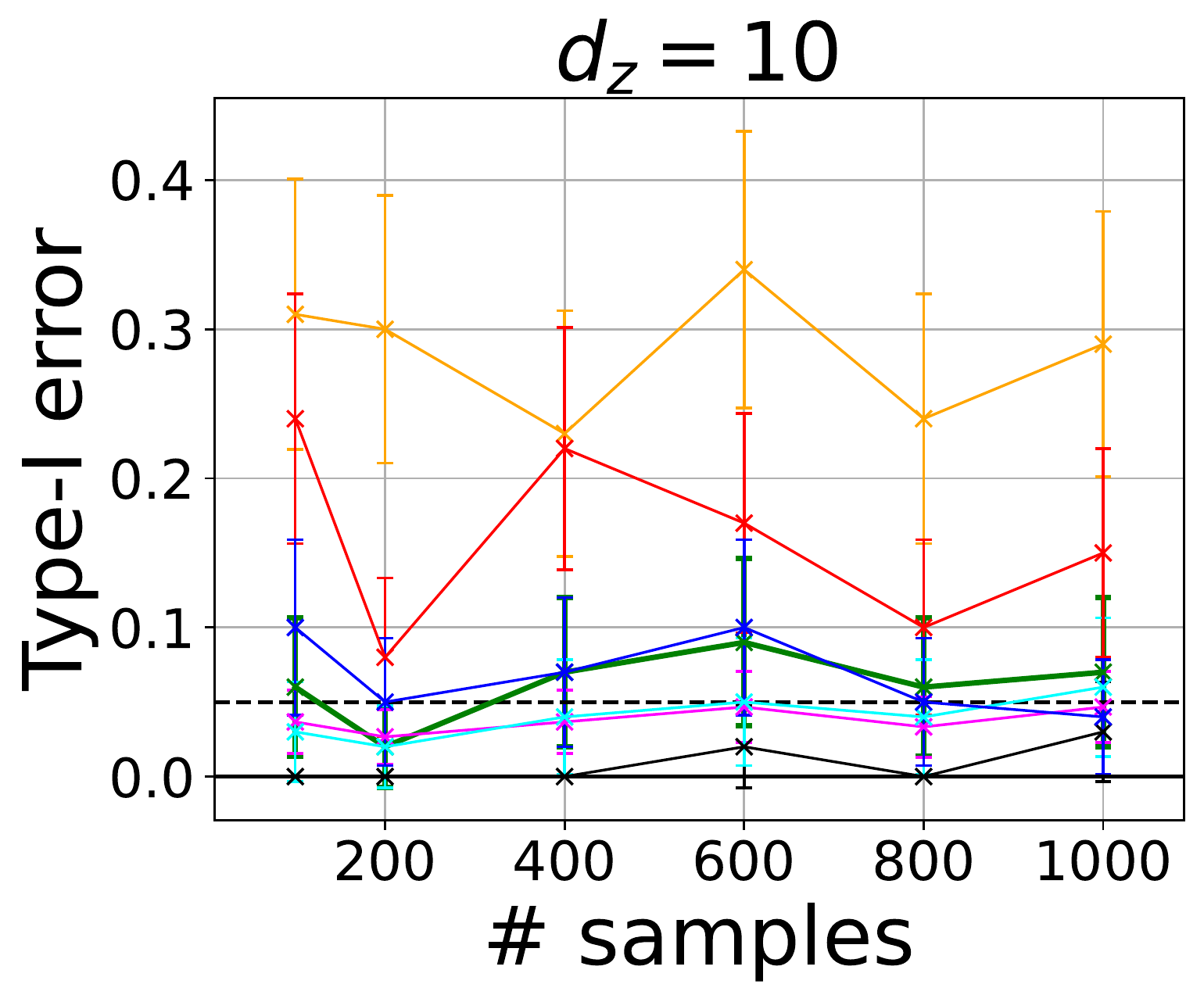}& \includegraphics[height=2.9cm]{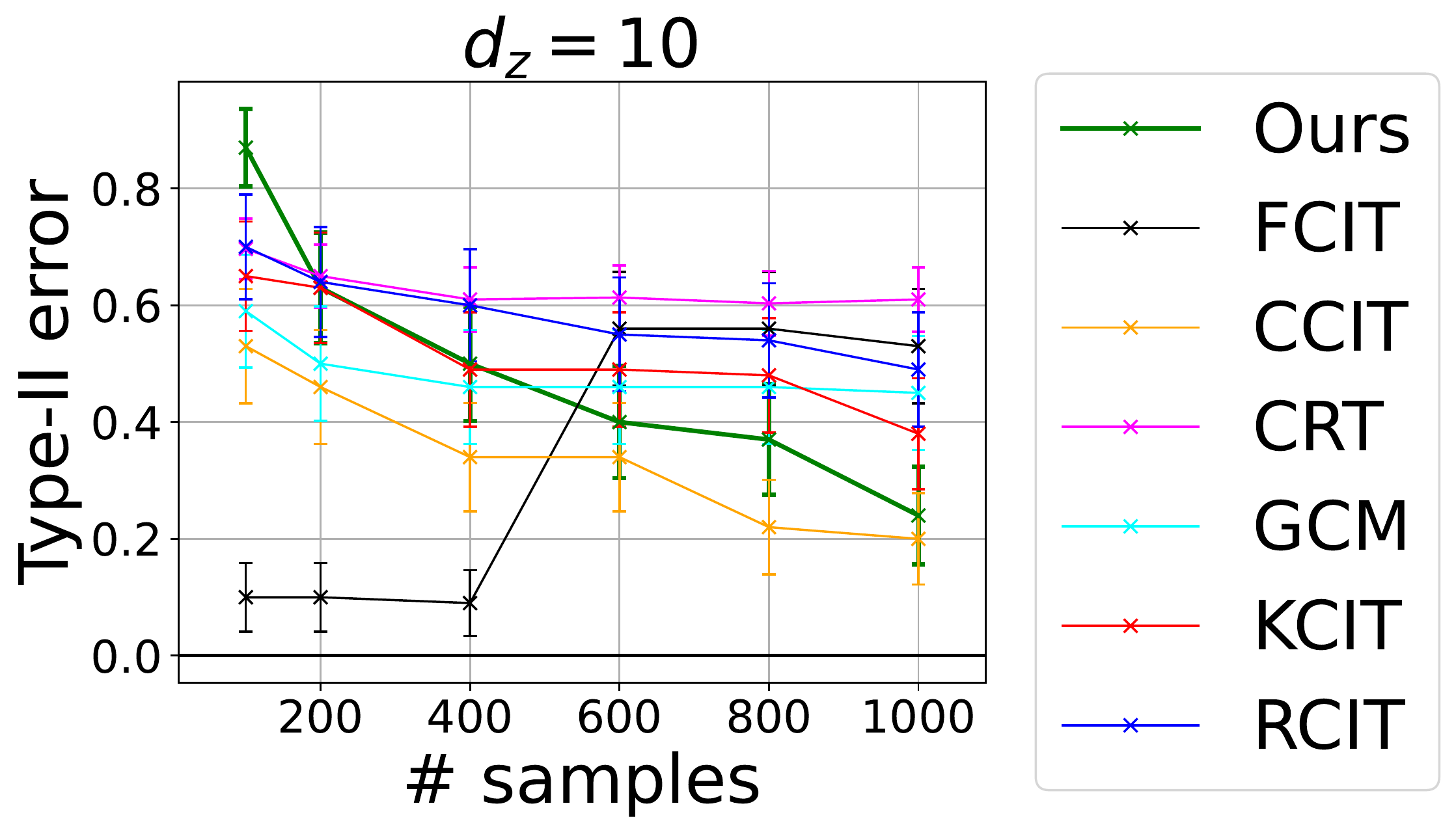}
\end{tabular}
\caption{Comparison of the type-I error at level $\alpha=0.05$ (dashed line) and the type-II error (lower is better) of our test procedure with other SoTA tests on the two problems presented in~\eqref{exp-strobl-h0} and~\eqref{exp-strobl-h1}  with Gaussian noises. Each point in the figures is obtained by repeating the experiment for 100 independent trials. (\emph{Left, middle-left}): type-I and type-II errors obtained by each test when varying the dimension $d_z$ from 1 to 10; here, the number of samples $n$ is fixed and equals to $1000$. (\emph{Middle-right, right}): type-I and type-II errors obtained by each test when varying the number of samples $n$ from 100 to 1000; here, the dimension $d_z$ is fixed and equals to $10$. 
\label{fig-exp-strobl-type-supp}}
\vspace{-0.3cm}
\end{figure*}

\begin{figure*}[ht]
\begin{tabular}{cccc} 
\includegraphics[height=2.9cm]{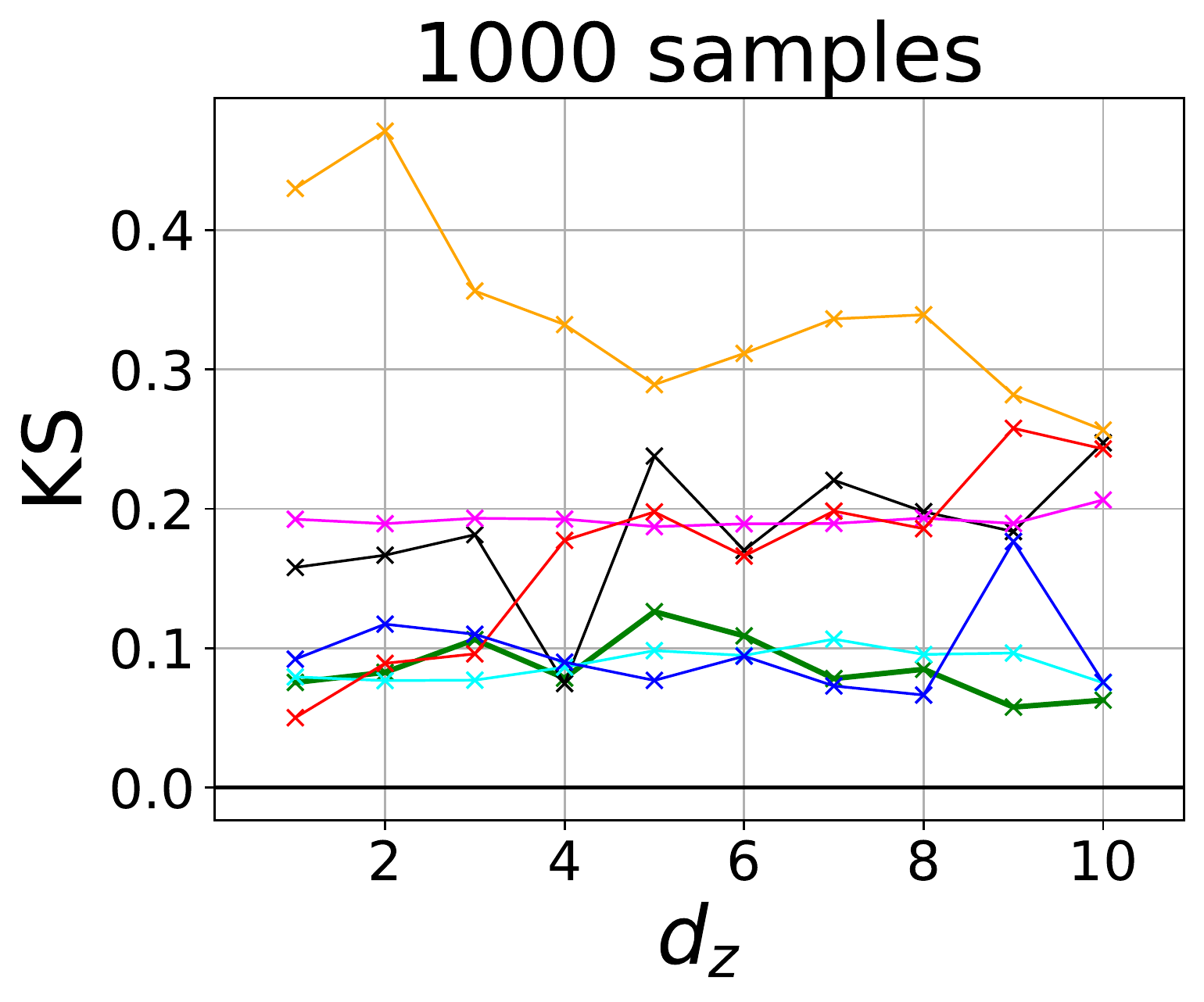}& \includegraphics[height=2.9cm]{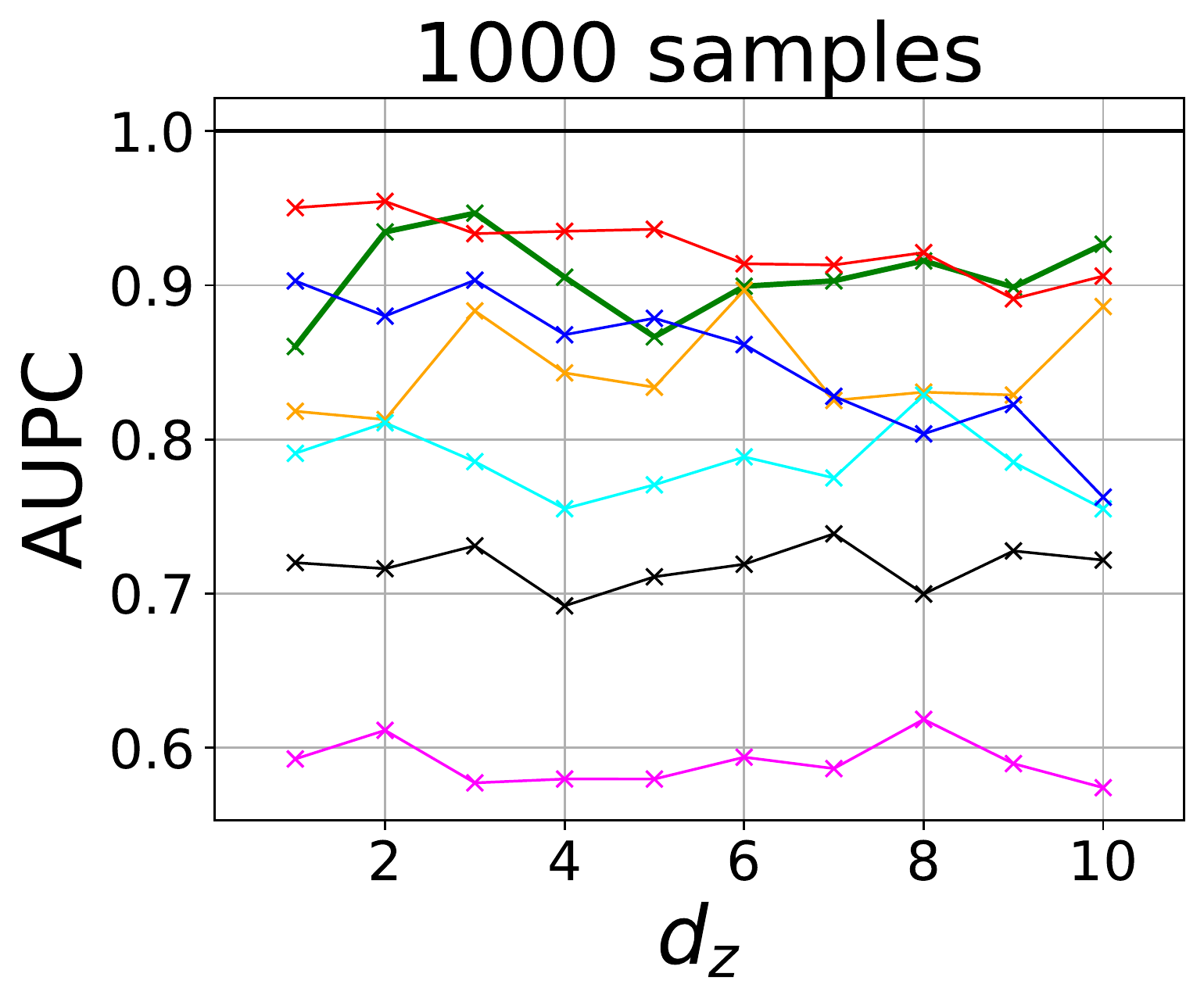} & 
\includegraphics[height=2.9cm]{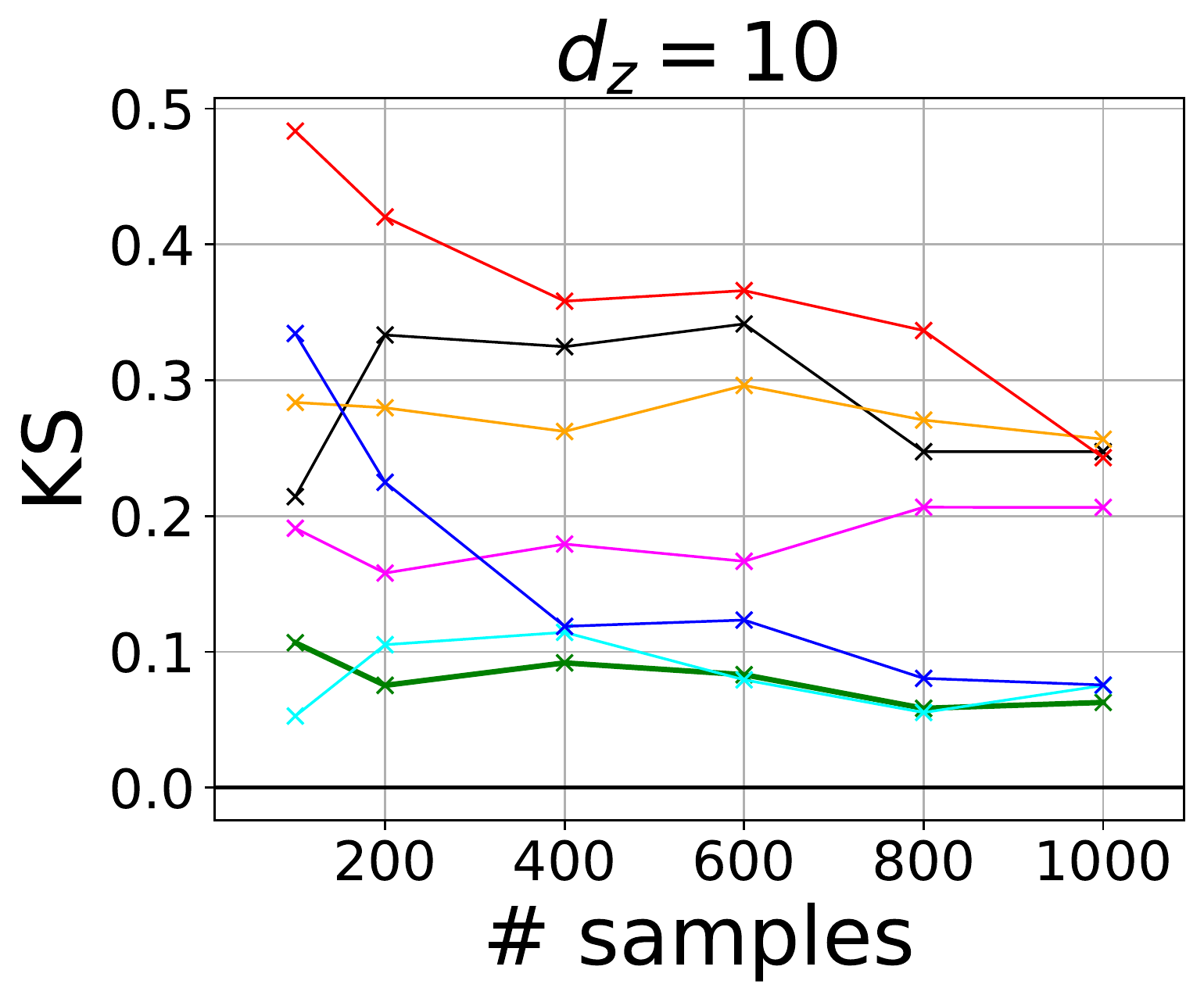}& \includegraphics[height=2.9cm]{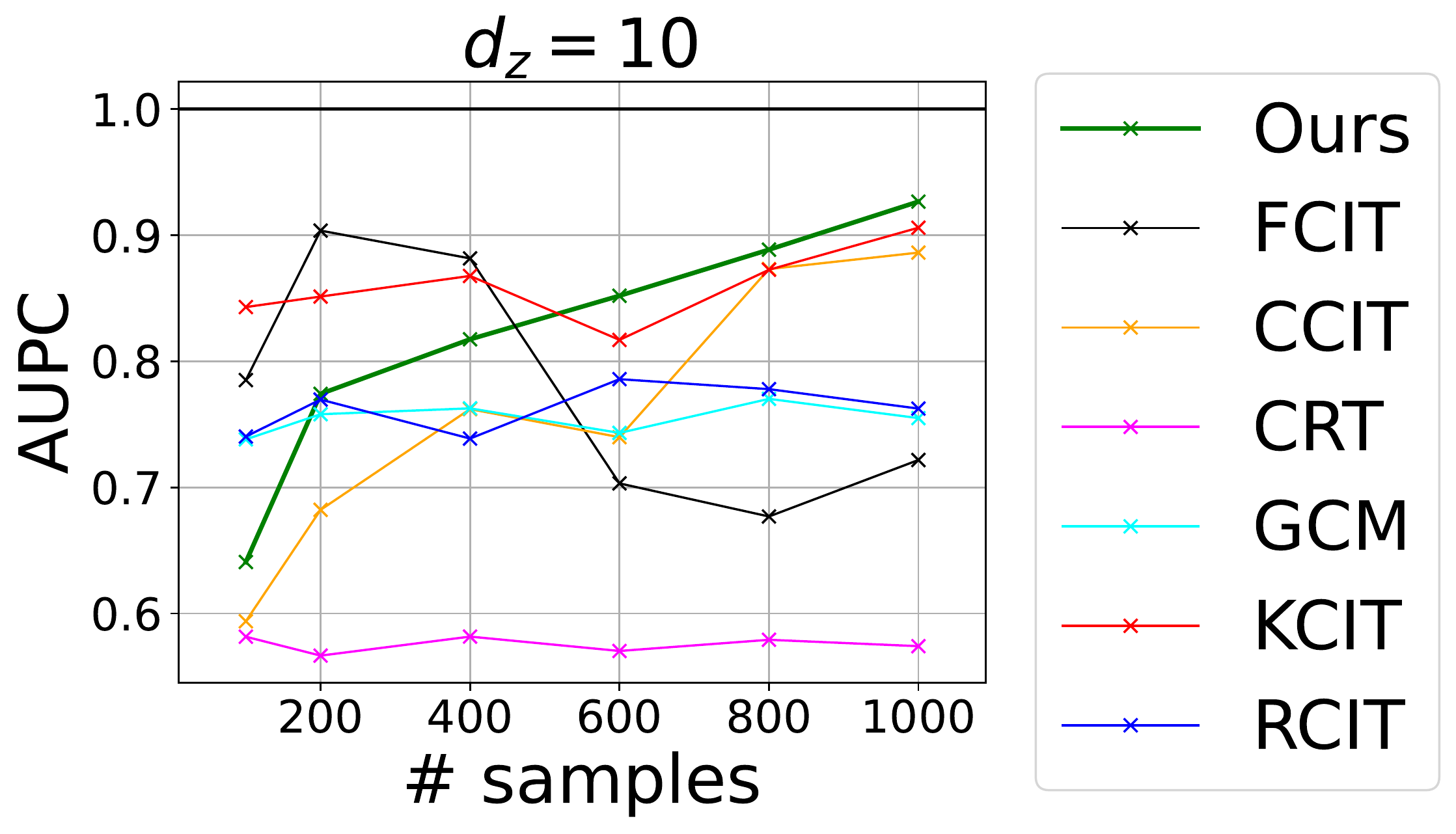} 
\end{tabular}
\caption{Comparison of the KS statistic (lower is better) and the AUPC (higher is better) of our testing procedure with other SoTA tests on the two problems presented in~\eqref{exp-strobl-h0} and~\eqref{exp-strobl-h1}  with Gaussian noises. Each point in the figures is obtained by repeating the experiment for 100 independent trials. (\emph{Left, middle-left}): the KS and AUPC obtained by each test when varying the dimension $d_z$ from 1 to 10, while fixing the number of samples $n$ to $1000$. (\emph{Middle-right, right}): the KS and AUPC obtained by each test when varying the number of samples $n$ from 100 to 1000, while fixing the dimension $d_z$ to $10$. 
\label{fig-exp-strobl-ks-supp}}
\vspace{-0.3cm}
\end{figure*}

\begin{figure*}[h]
\begin{tabular}{cccc} 
\includegraphics[height=2.9cm]{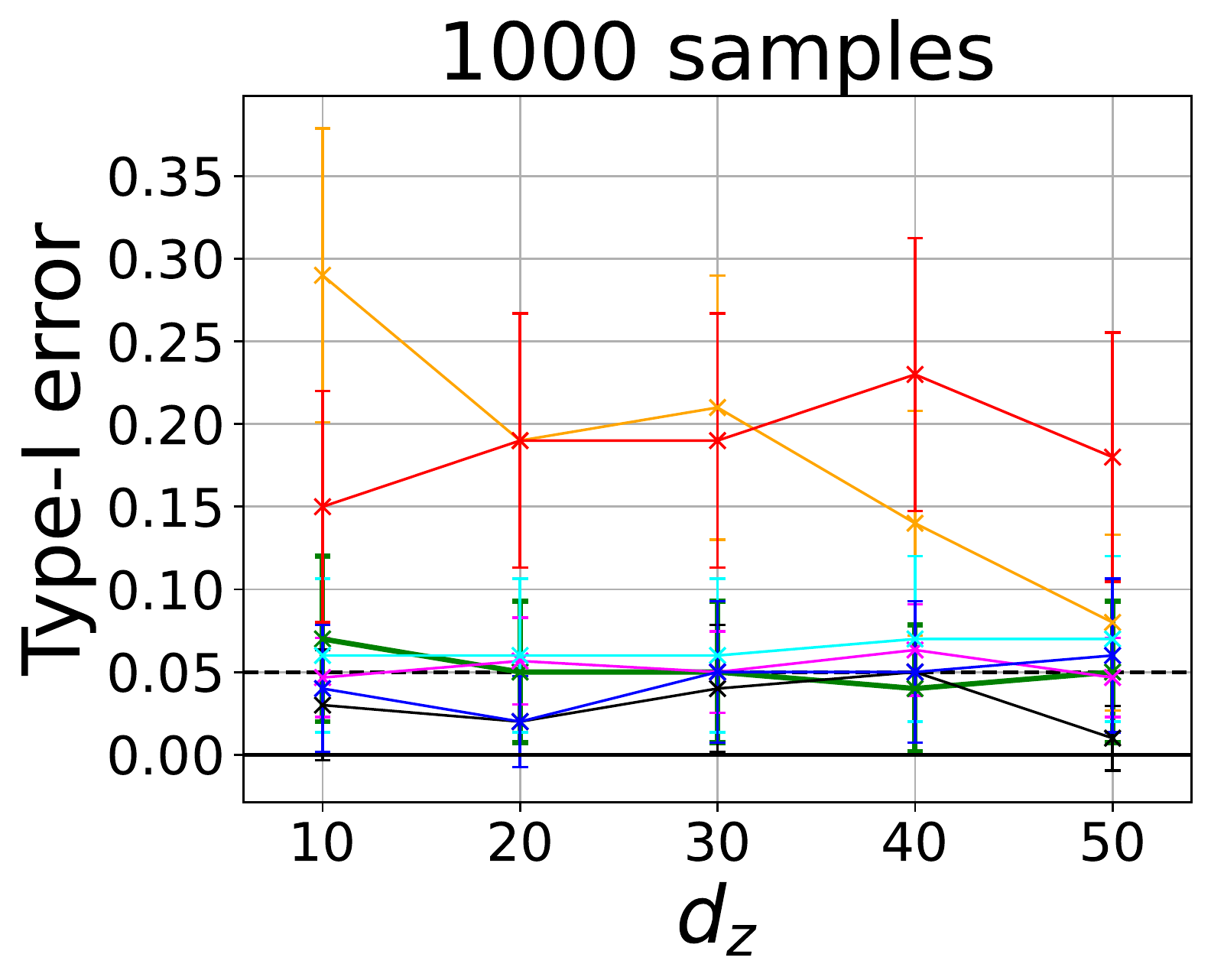}& \includegraphics[height=2.9cm]{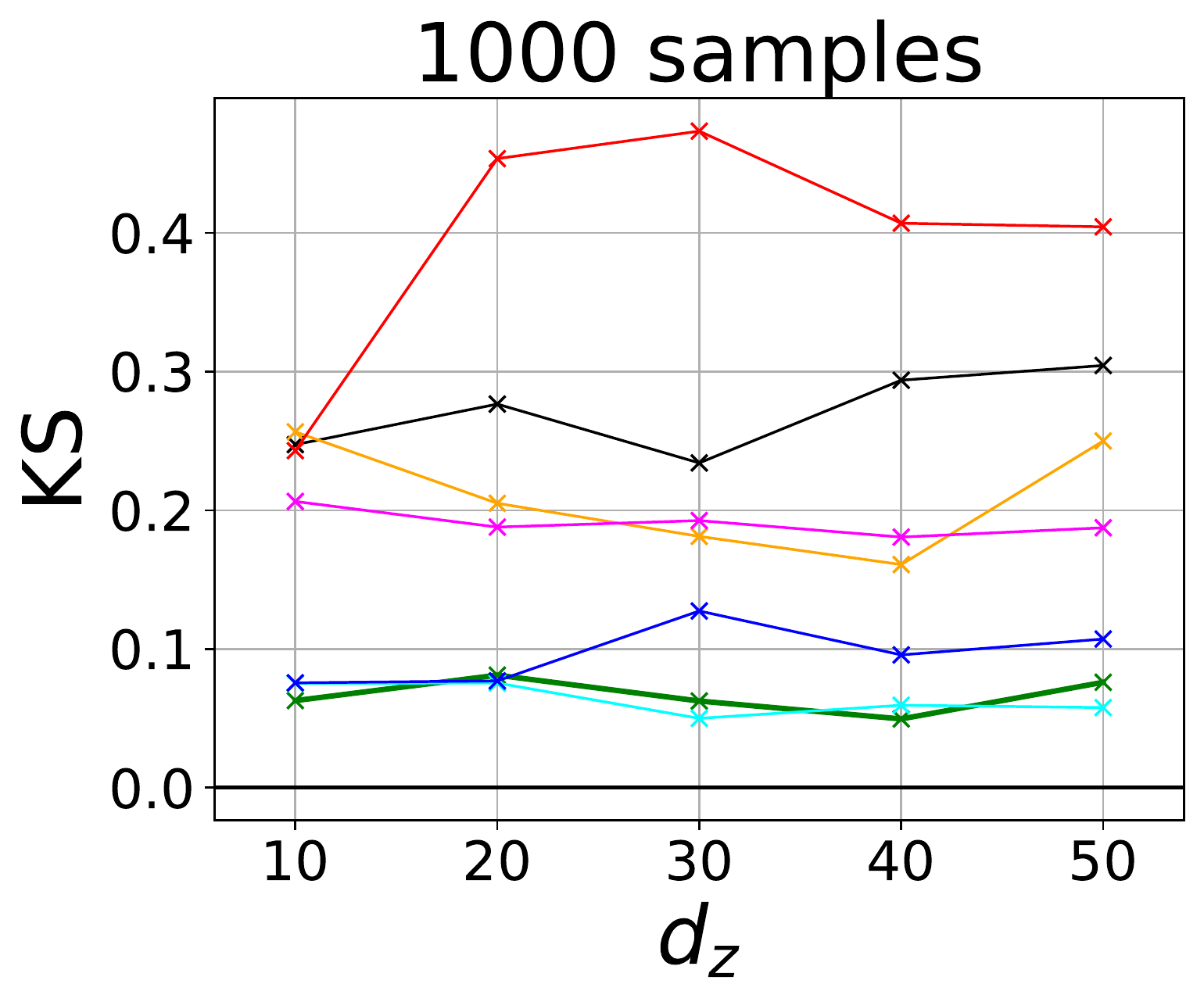} & 
\includegraphics[height=2.9cm]{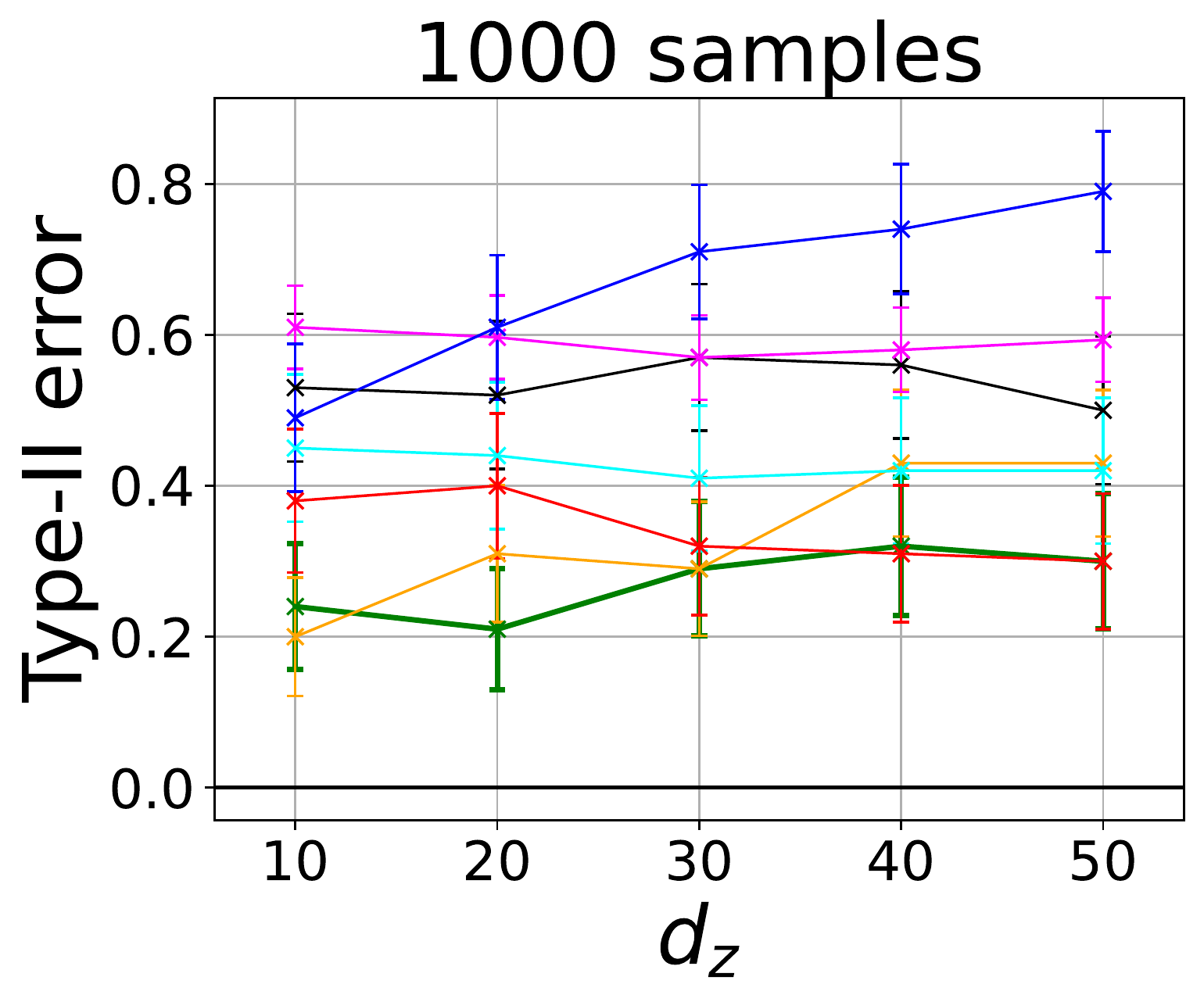}& \includegraphics[height=2.9cm]{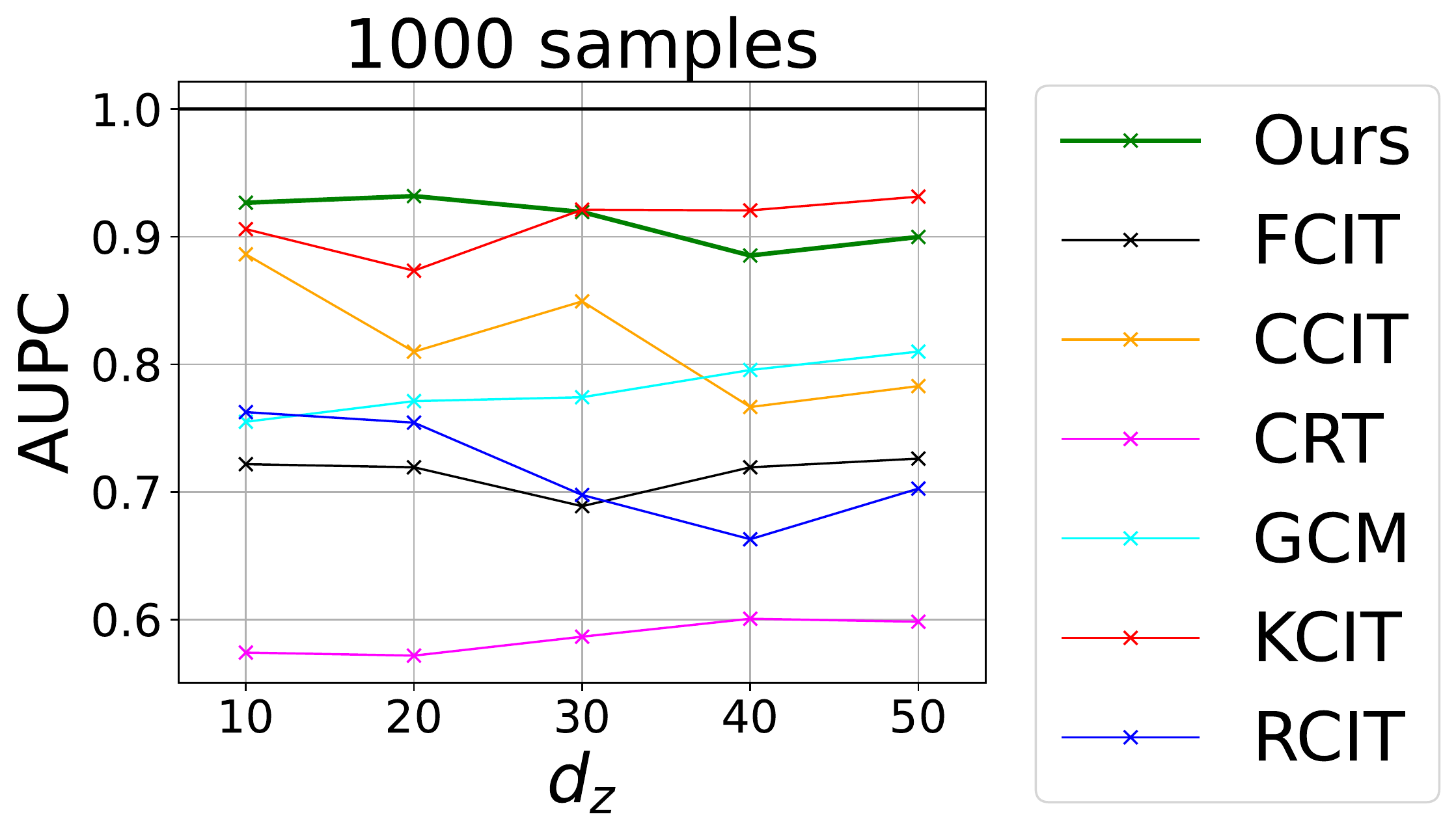}
\end{tabular}
\caption{Comparison of the type-I error at level $\alpha=0.05$ (dashed line), type-II error (lower is better), KS statistic and the AUPC of our testing procedure with other SoTA tests on the two problems presented in Eq.~\eqref{exp-strobl-h0} and Eq.~\eqref{exp-strobl-h1} with Gaussian noises. Each point in the figures is obtained by repeating the experiment for 100 independent trials. In each plot the dimension $d_z$ is varying from 10 to 50; here, the number of samples $n$ is fixed and equals to $1000$. 
\label{fig-exp-strobl-highdim-gaussian-supp}}
\vspace{-0.3cm}
\end{figure*}

\newpage

\subsubsection{Laplace Case}

\begin{figure*}[htb]
\begin{tabular}{cccc} 
\includegraphics[height=2.9cm]{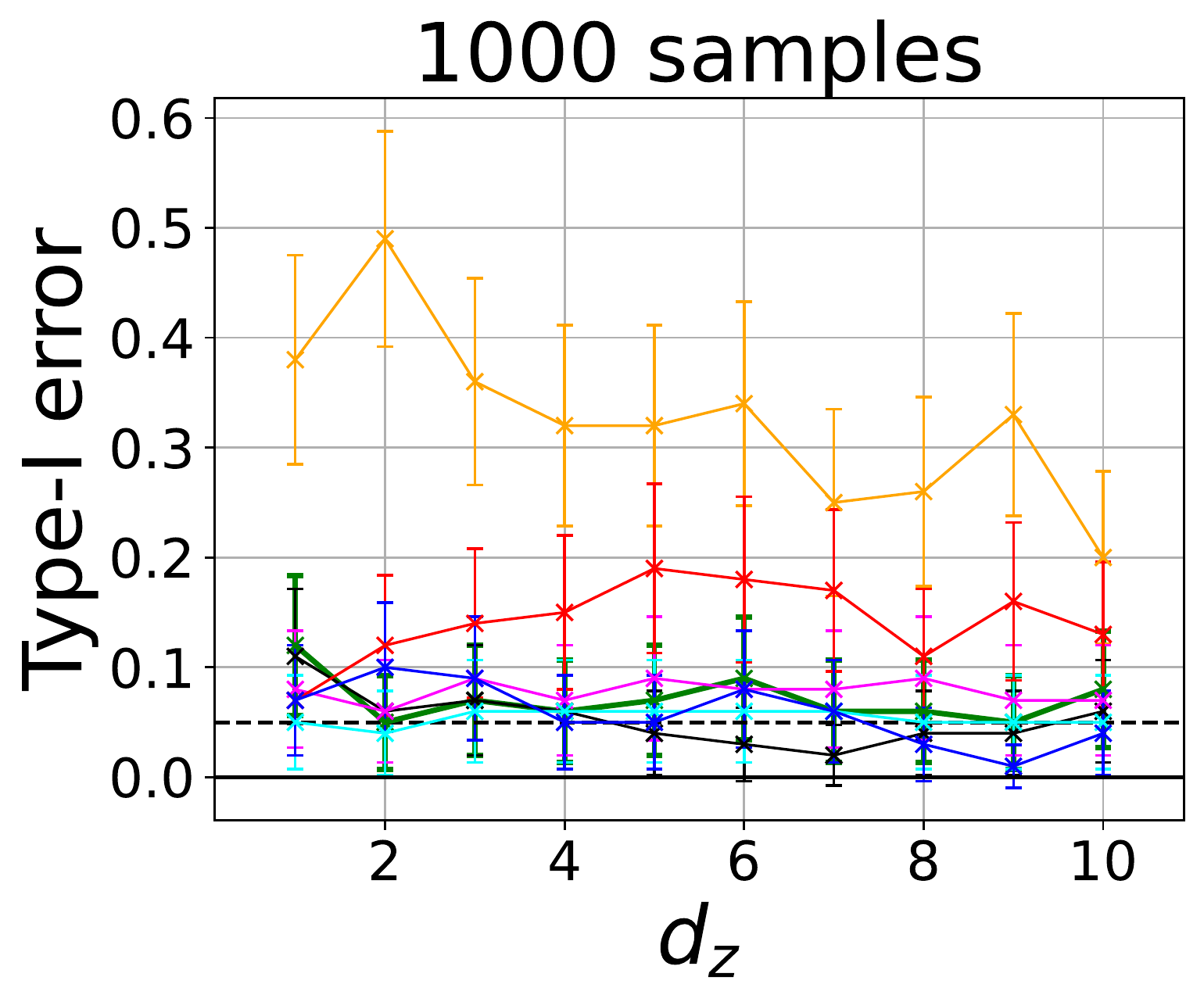}& \includegraphics[height=2.9cm]{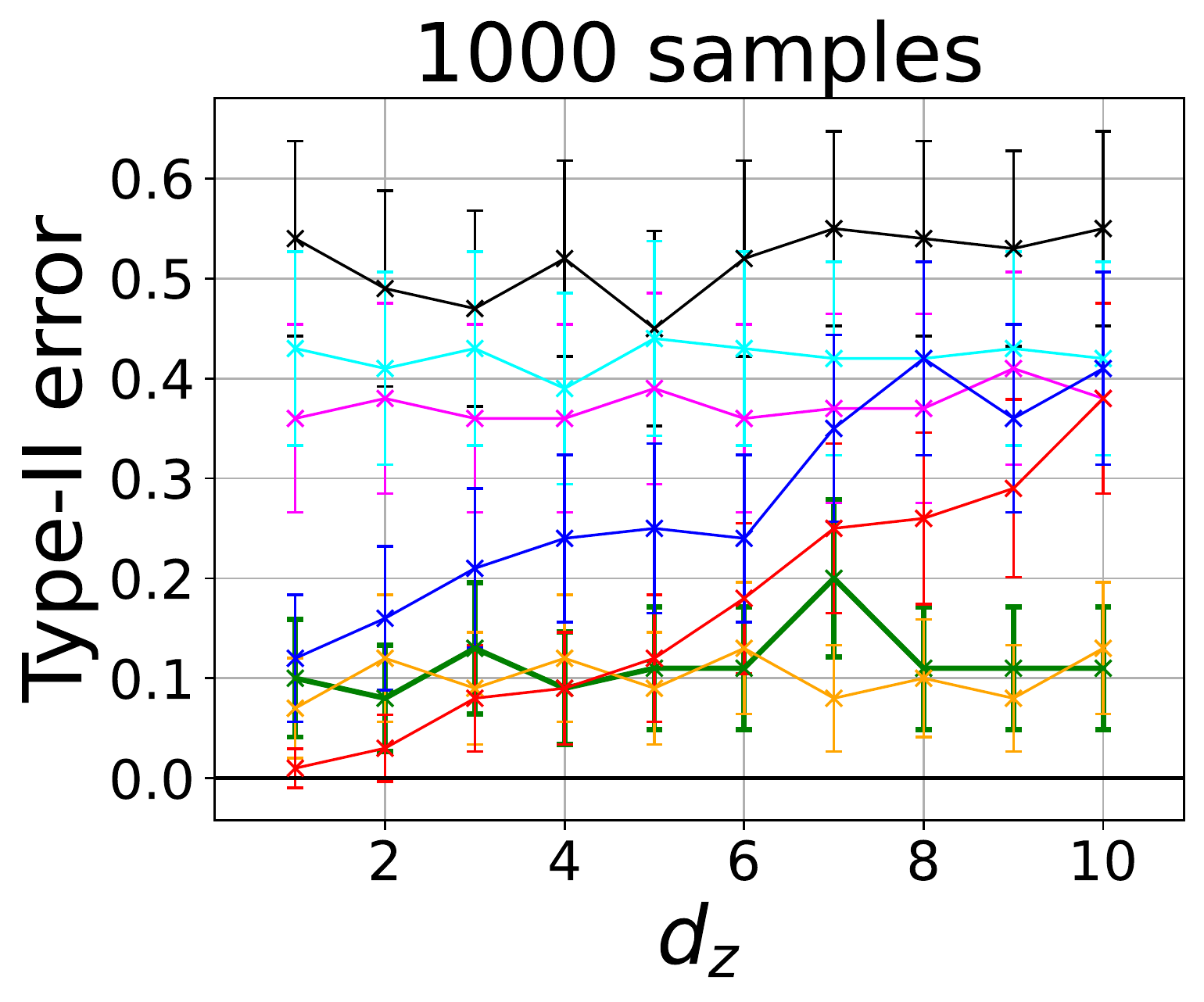} & 
\includegraphics[height=2.9cm]{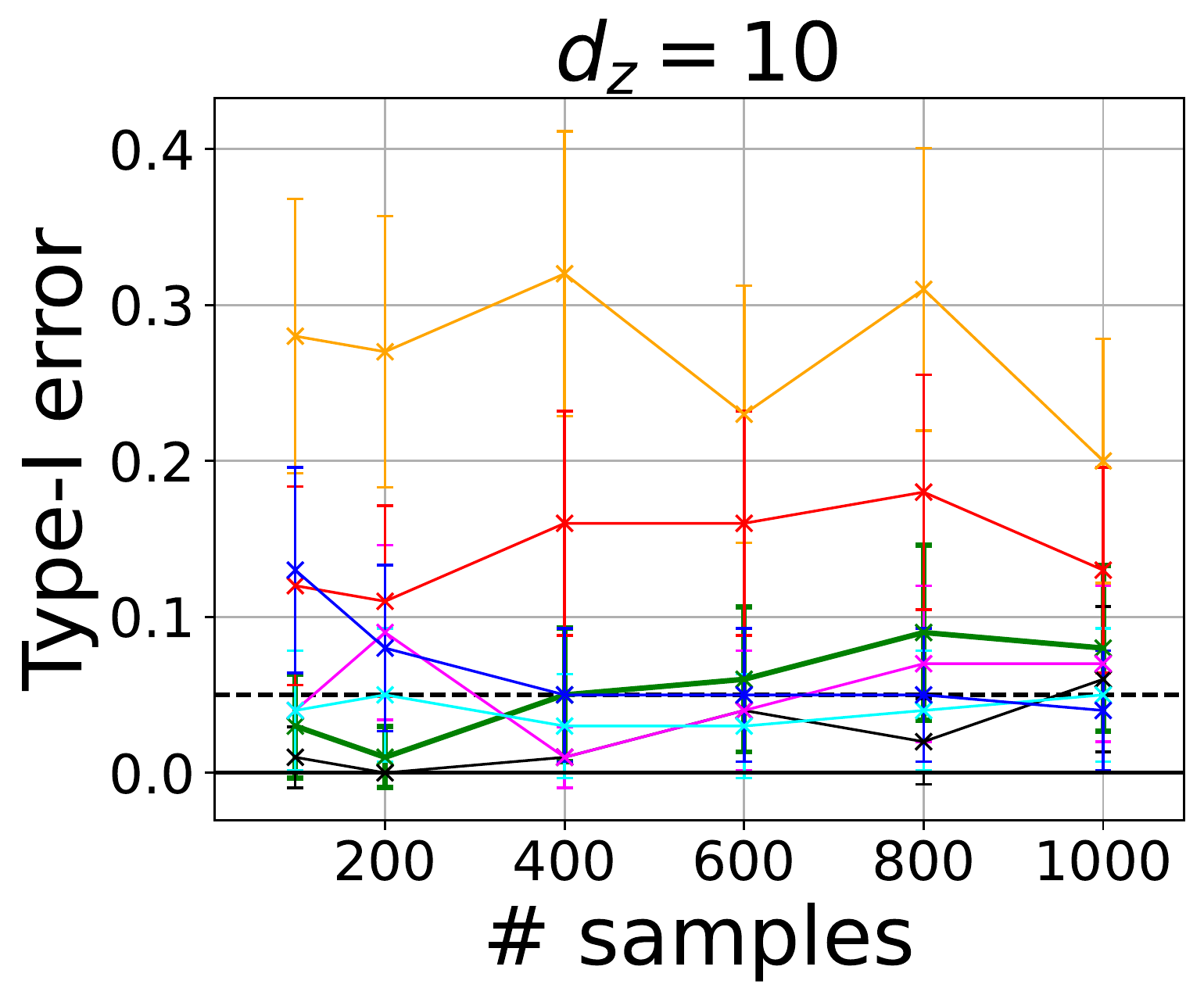}& \includegraphics[height=2.9cm]{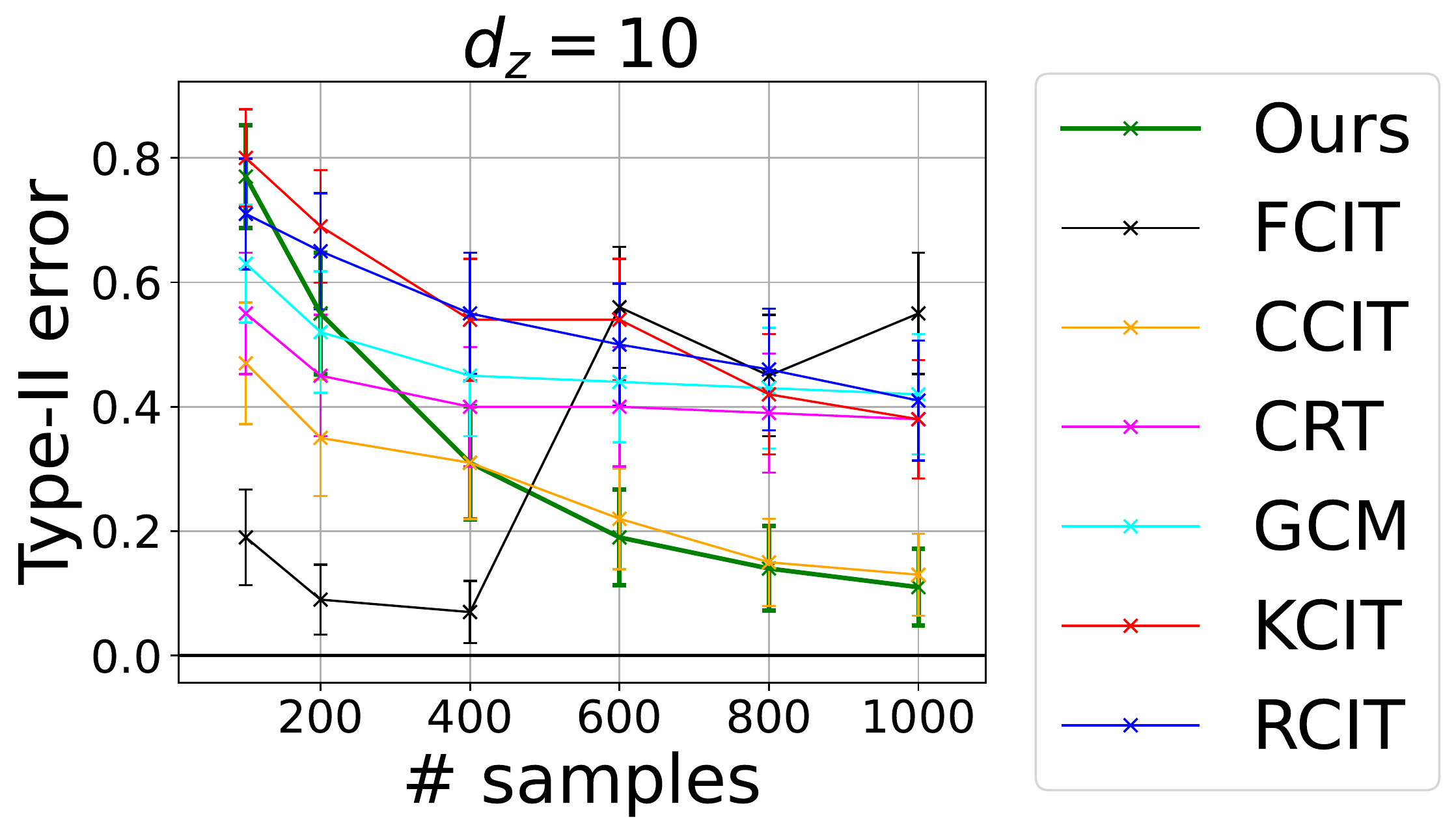} 
\end{tabular}
\caption{Comparison of the type-I error at level $\alpha=0.05$ (dashed line) and the type-II error (lower is better) of our test procedure with other SoTA tests on the two problems presented in~\eqref{exp-strobl-h0} and~\eqref{exp-strobl-h1} with Laplace noises. Each point in the figures is obtained by repeating the experiment for 100 independent trials. (\emph{Left, middle-left}): type-I and type-II errors obtained by each test when varying the dimension $d_z$ from 1 to 10; here, the number of samples $n$ is fixed and equals to $1000$. (\emph{Middle-right, right}): type-I and type-II errors obtained by each test when varying the number of samples $n$ from 100 to 1000; here, the dimension $d_z$ is fixed and equals to $10$.
\label{fig-exp-strobl-laplace-supp}}
\vspace{-0.3cm}
\end{figure*}

\begin{figure*}[htb]
\begin{tabular}{cccc} 
\includegraphics[height=2.9cm]{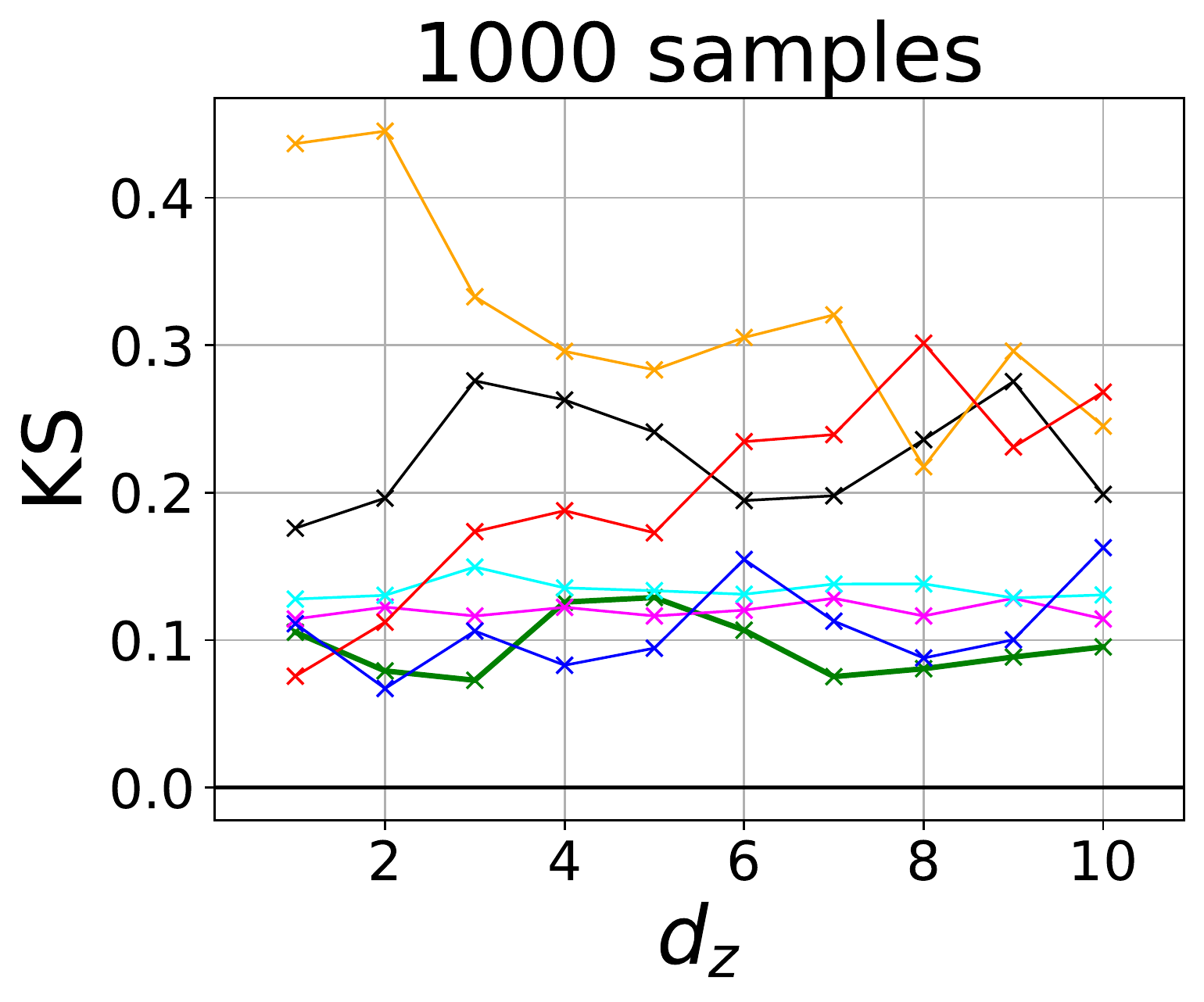}& \includegraphics[height=2.9cm]{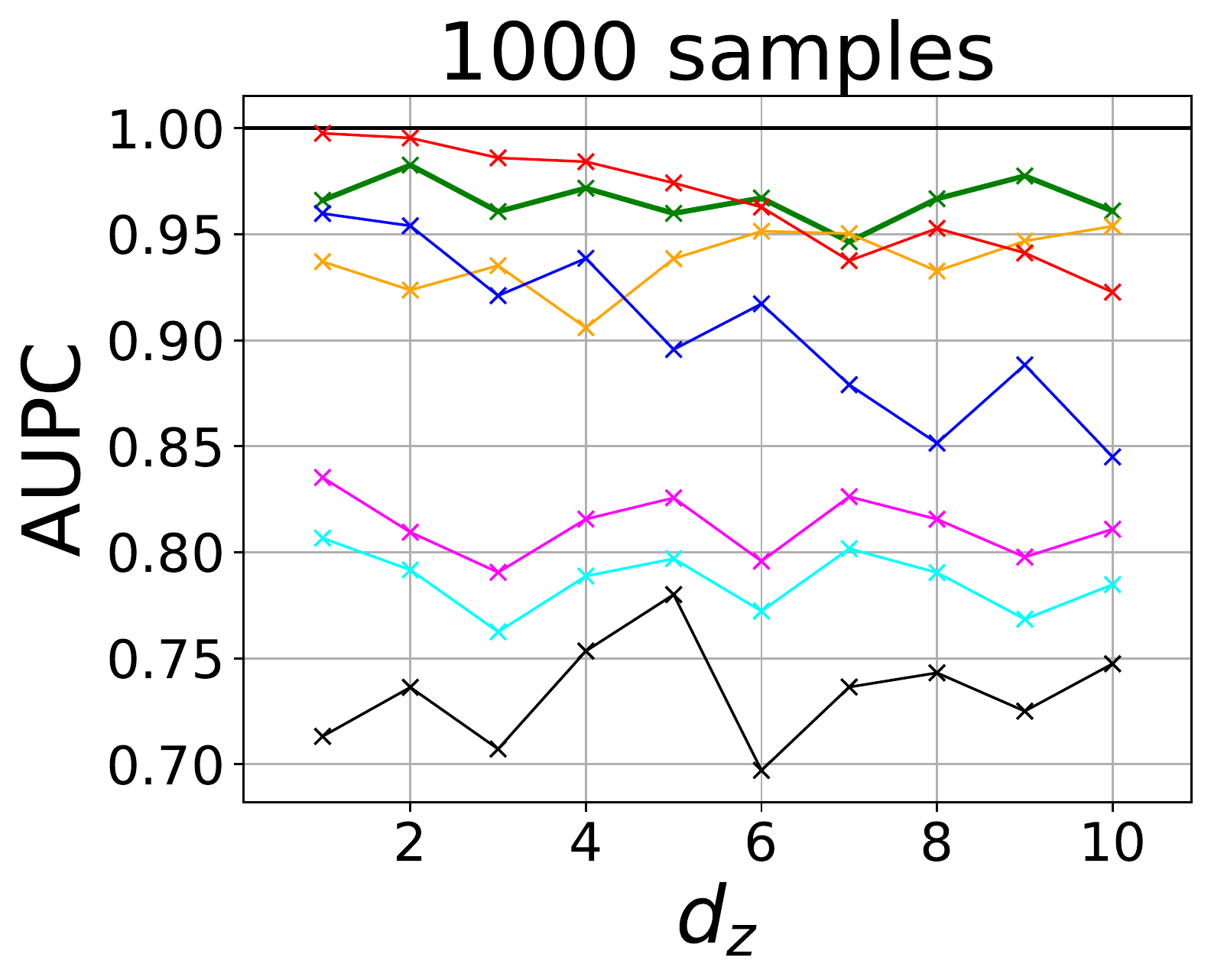} & 
\includegraphics[height=2.9cm]{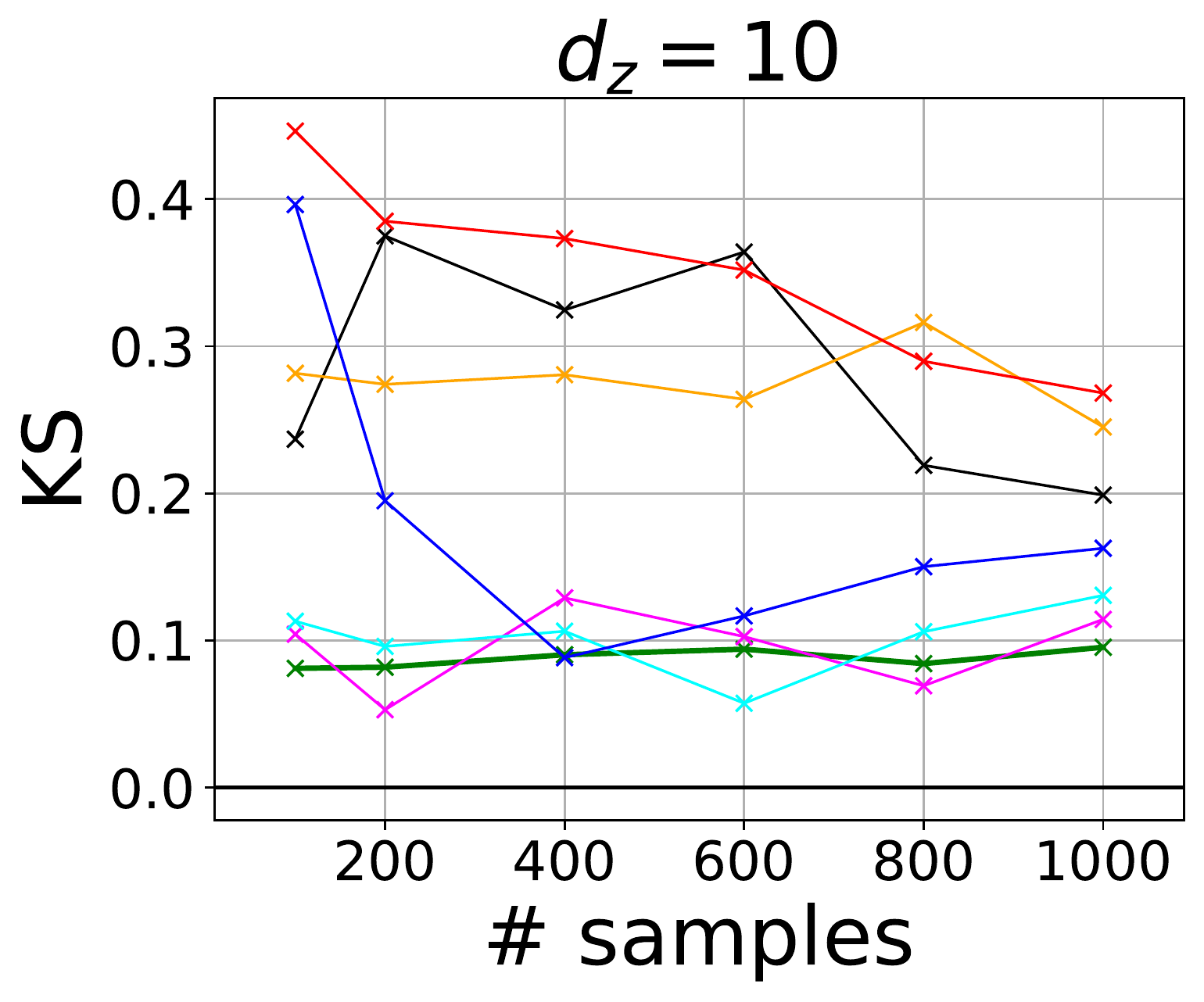}& \includegraphics[height=2.9cm]{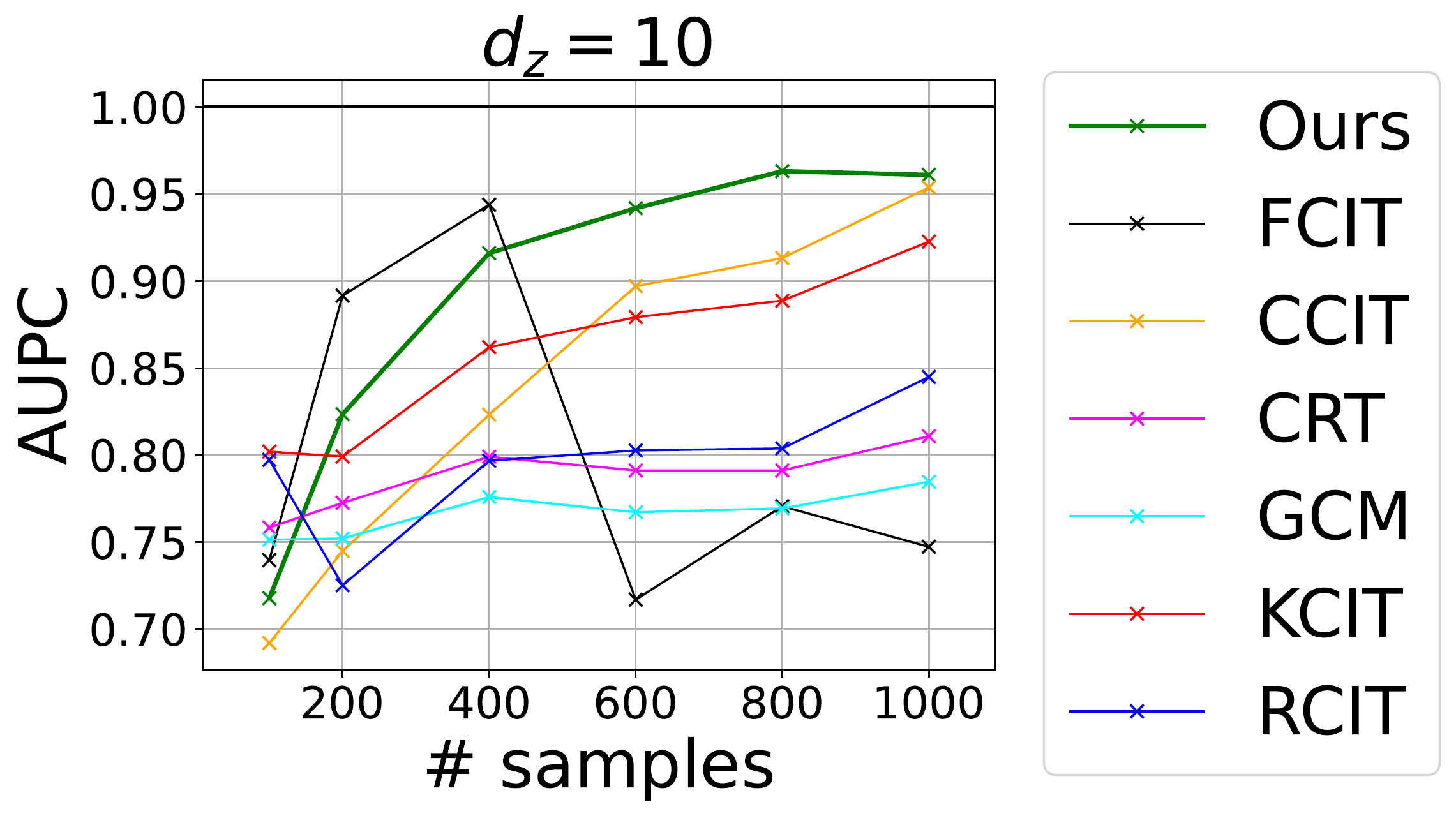} 
\end{tabular}
\caption{Comparison of the KS statistic and the AUPC of our testing procedure with other SoTA tests on the two problems presented in Eq.~\eqref{exp-strobl-h0} and Eq.~\eqref{exp-strobl-h1}  with Laplace noises. Each point in the figures is obtained by repeating the experiment for 100 independent trials. (\emph{Left, middle-left}): the KS statistic and AUPC (respectively) obtained by each test when varying the dimension $d_z$ from 1 to 10; here, the number of samples $n$ is fixed and equals to $1000$. (\emph{Middle-right, right}): the KS and AUPC (respectively), obtained by each test when varying the number of samples $n$ from 100 to 1000; here, the dimension $d_z$ is fixed and equals to $10$.
\label{fig-exp-strobl-ks-laplace-supp}}
\vspace{-0.3cm}
\end{figure*}

\begin{figure*}[ht]
\begin{tabular}{cccc} 
\includegraphics[height=2.9cm]{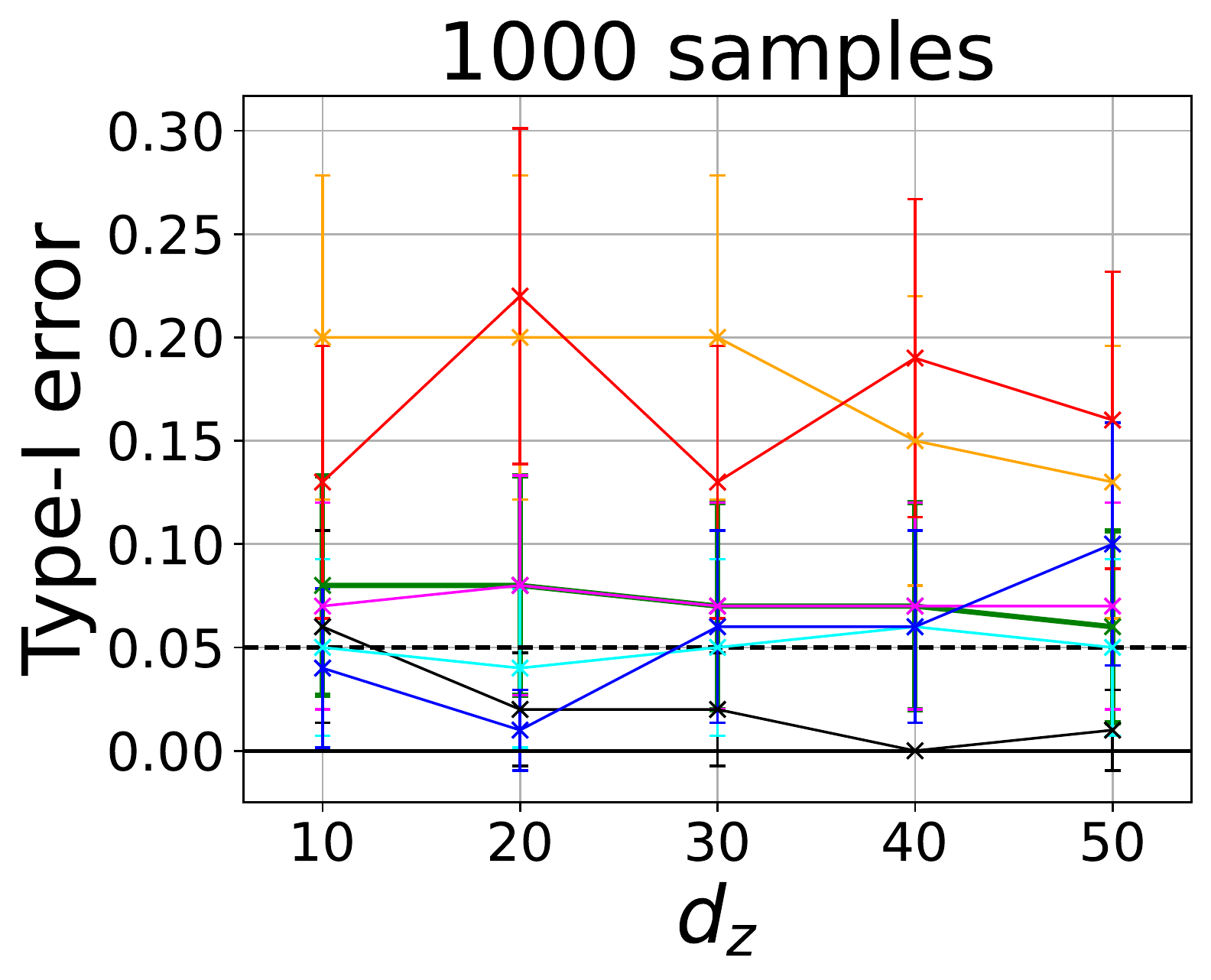}& \includegraphics[height=2.9cm]{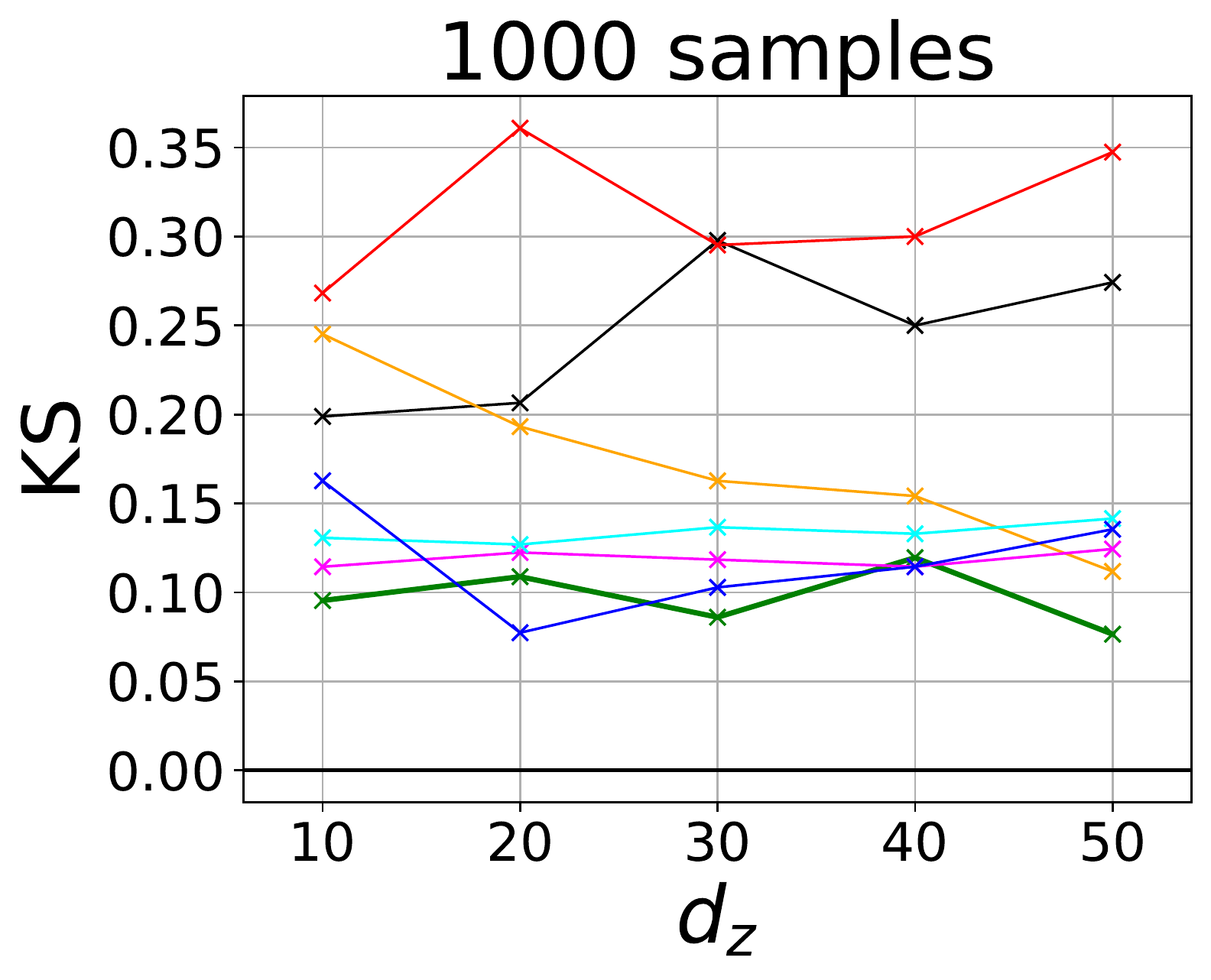} & 
\includegraphics[height=2.9cm]{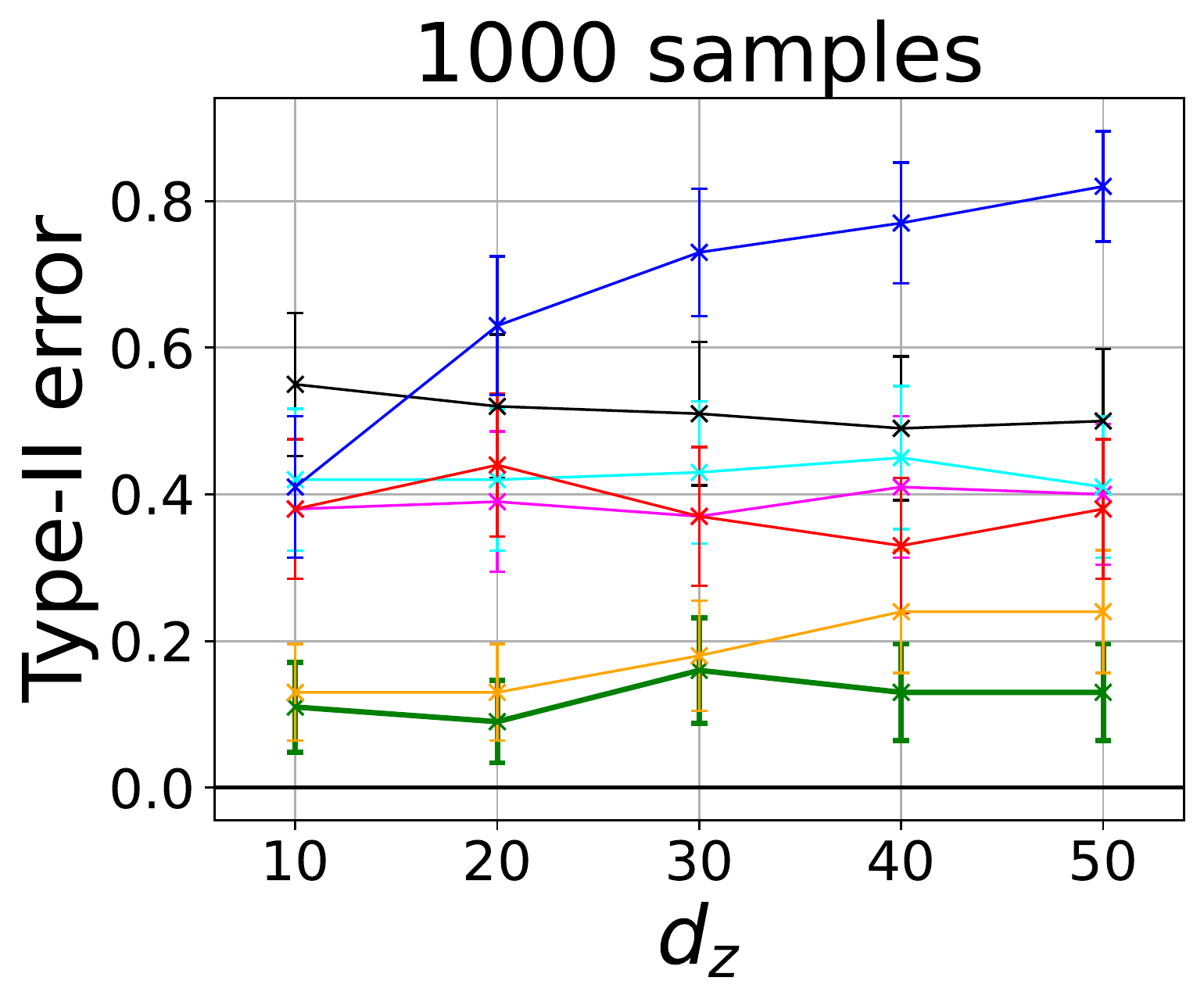}& \includegraphics[height=2.9cm]{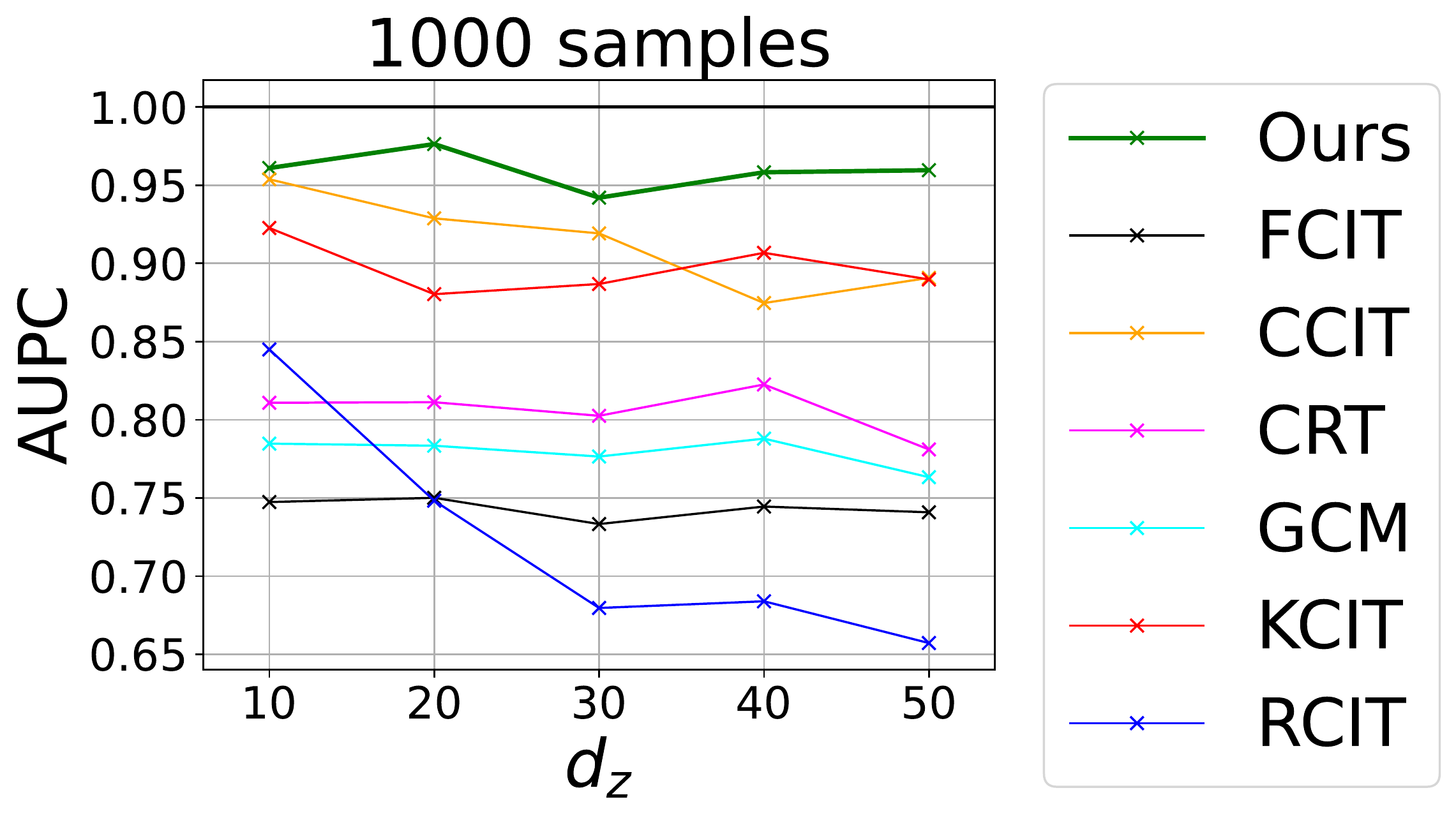}
\end{tabular}
\caption{Comparison of the type-I error at level $\alpha=0.05$ (dashed line), type-II error (lower is better), KS statistic and the AUPC of our testing procedure with other SoTA tests on the two problems presented in Eq.~\eqref{exp-strobl-h0} and Eq.~\eqref{exp-strobl-h1} with Laplace noises. Each point in the figures is obtained by repeating the experiment for 100 independent trials. In each plot the dimension $d_z$ is varying from 10 to 50; here, the number of samples $n$ is fixed and equals to $1000$.
\label{fig-exp-strobl-highdim-laplace-supp}}
\vspace{-0.3cm}
\end{figure*}

\newpage
\subsection{Additional experiments on Problems~\eqref{li-exp-h0} and~\eqref{li-exp-h1}}
\label{sec-exp-li}
\subsubsection{Gaussian Case}


\begin{figure*}[htb]
\begin{tabular}{cccc} 
\includegraphics[height=2.9cm]{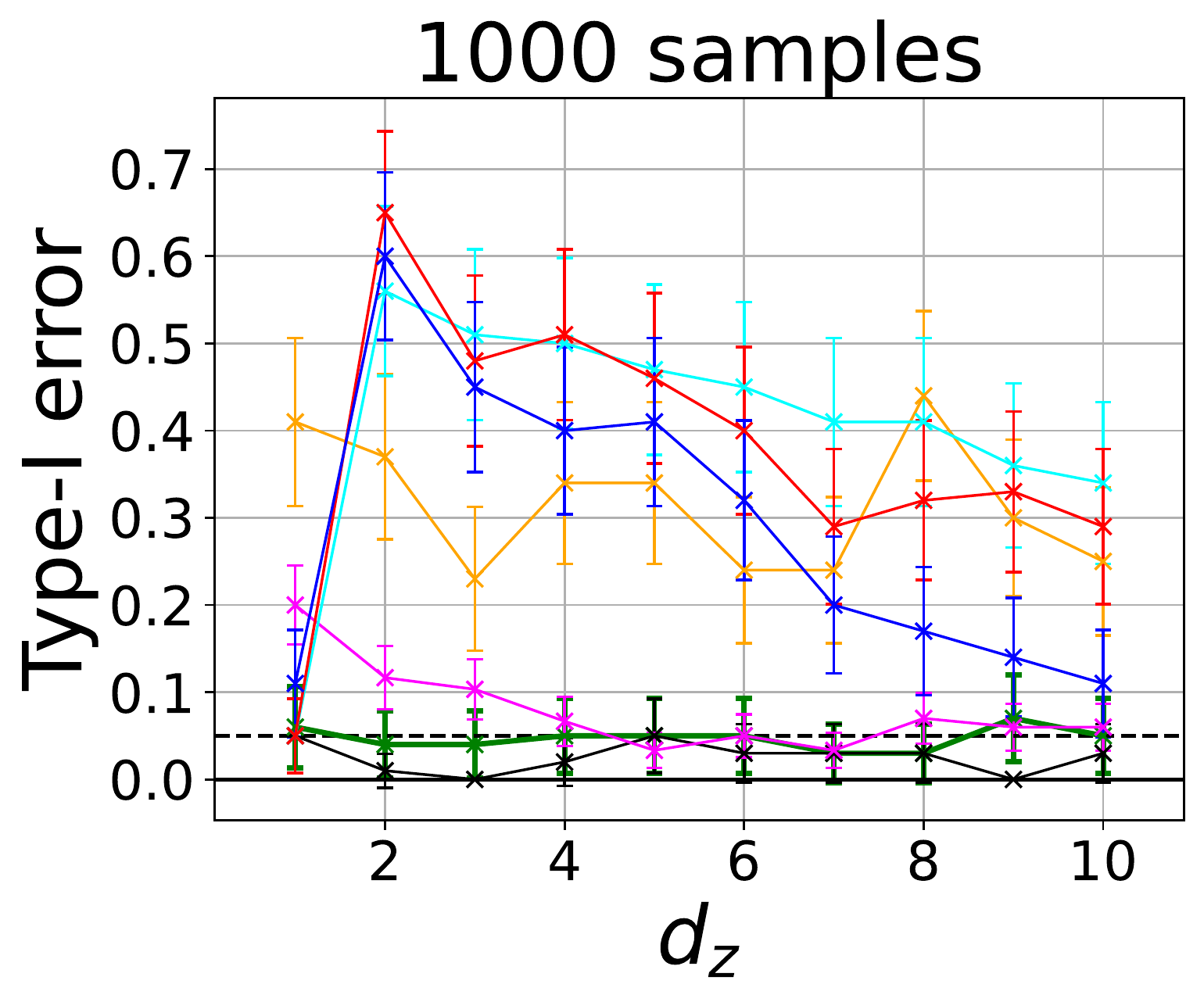}& \includegraphics[height=2.9cm]{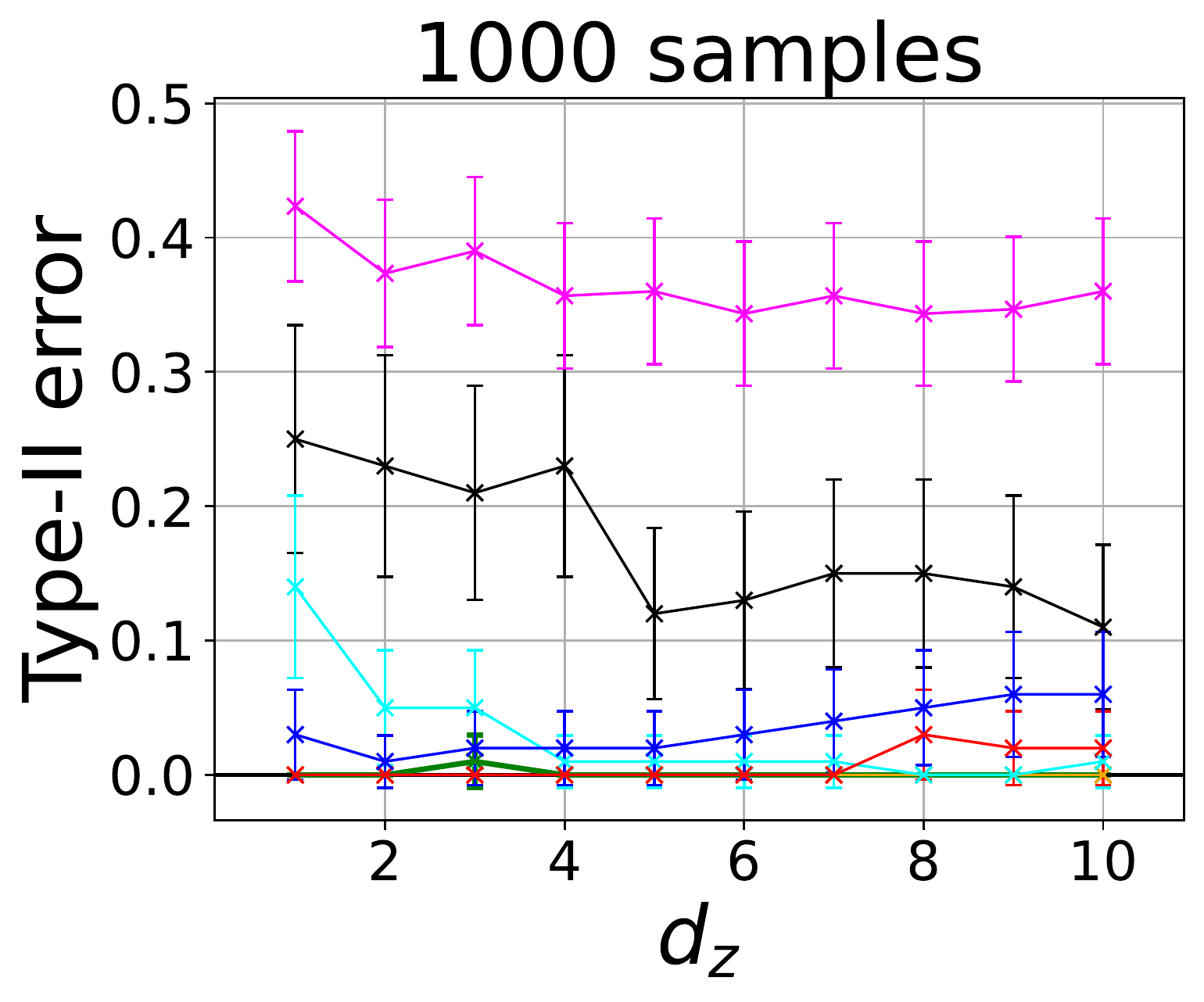} & 
\includegraphics[height=2.9cm]{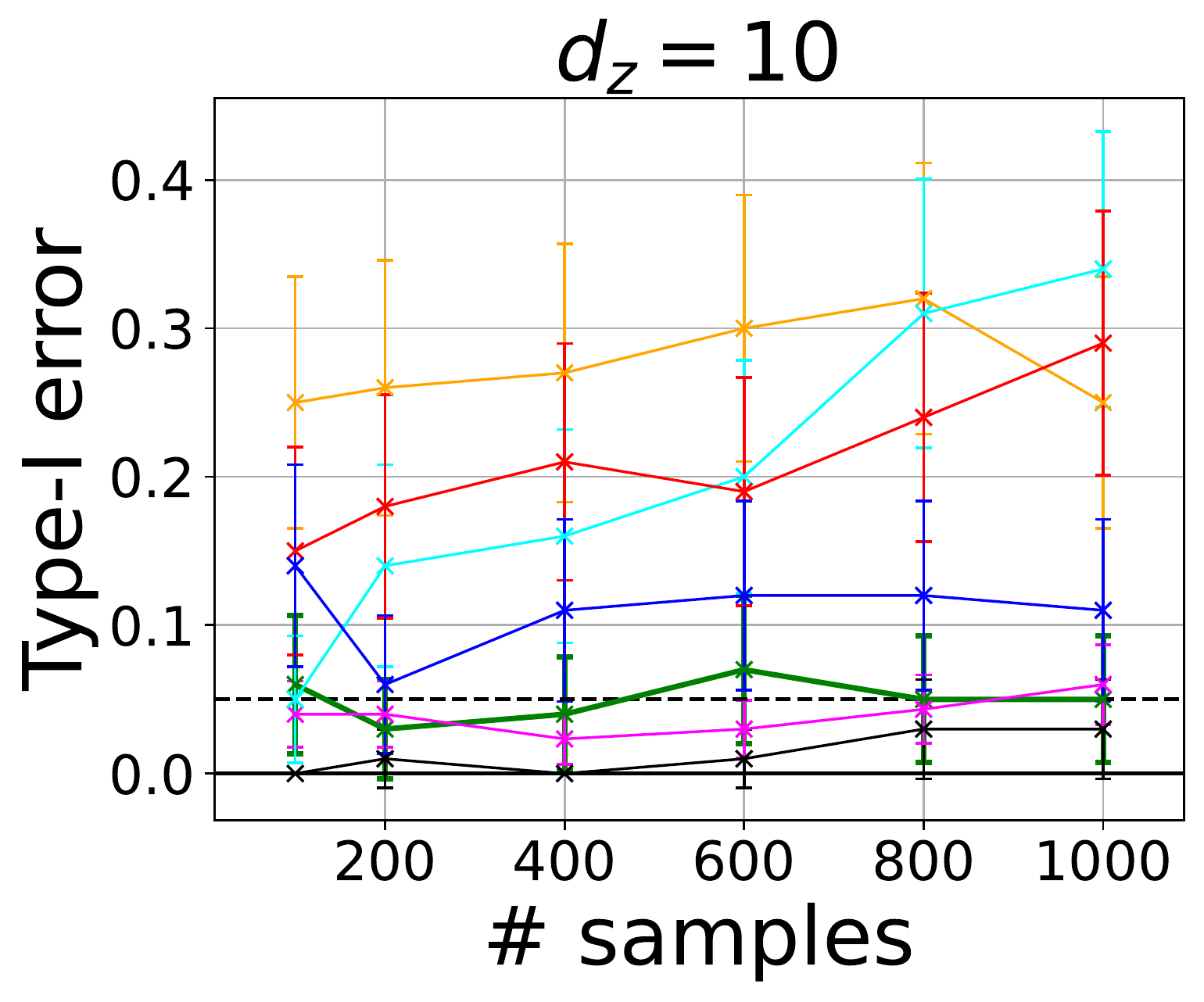}& \includegraphics[height=2.9cm]{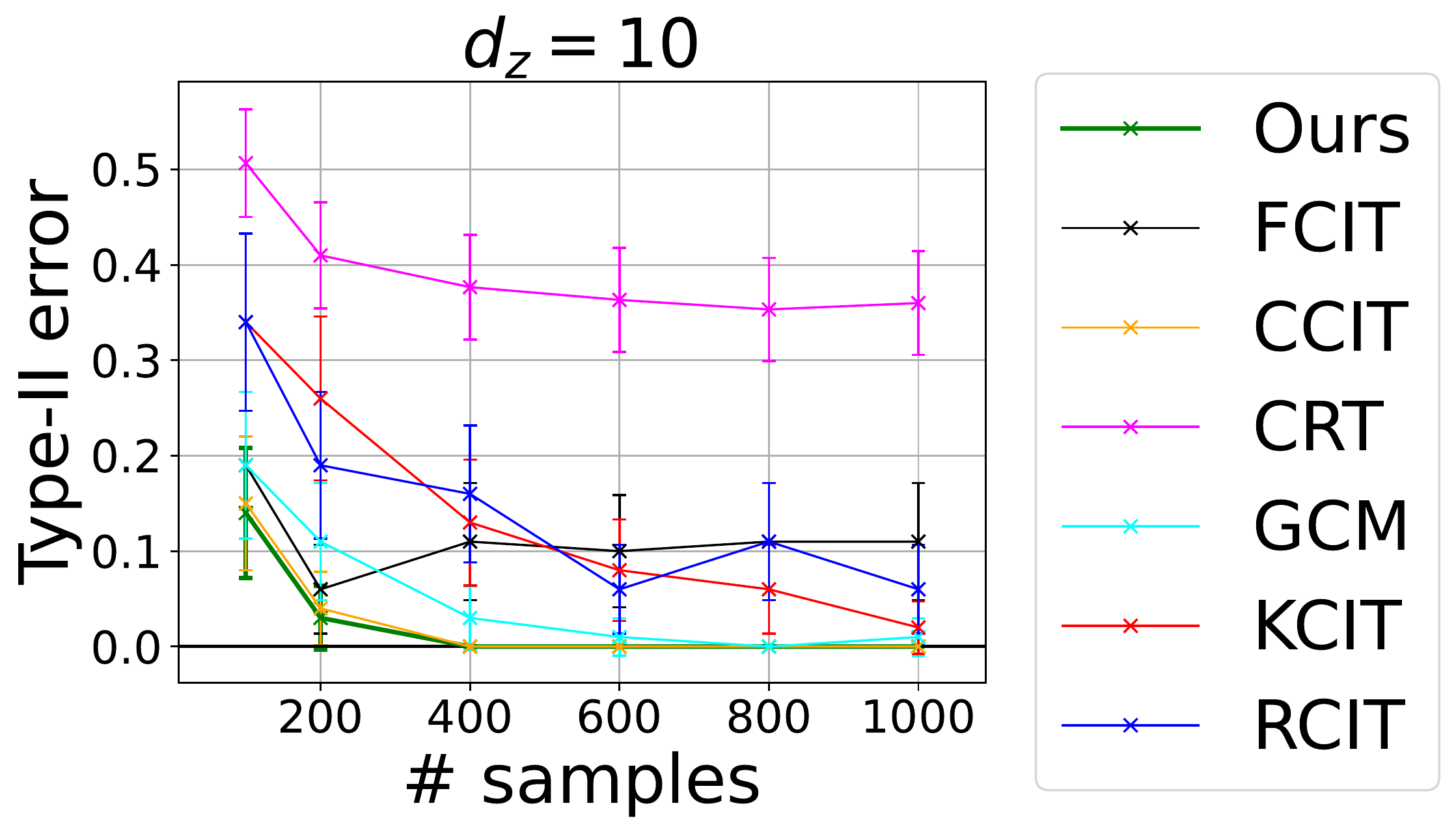} 
\end{tabular}
\caption{Comparison of the type-I error at level $\alpha=0.05$ (dashed line) and the type-II error (lower is better) of our test procedure with other SoTA tests on the two problems presented in~\eqref{li-exp-h0} and~\eqref{li-exp-h1}  with Gaussian noises. Each point in the figures is obtained by repeating the experiment for 100 independent trials. (\emph{Left, middle-left}): type-I and type-II errors obtained by each test when varying the dimension $d_z$ from 1 to 10; here, the number of samples $n$ is fixed and equals to $1000$. (\emph{Middle-right, right}): type-I and type-II errors obtained by each test when varying the number of samples $n$ from 100 to 1000; here, the dimension $d_z$ is fixed and equals to $10$. 
\label{fig-exp-li-type-supp}}
\vspace{-0.3cm}
\end{figure*}

\begin{figure*}[ht]
\begin{tabular}{cccc} 
\includegraphics[height=2.9cm]{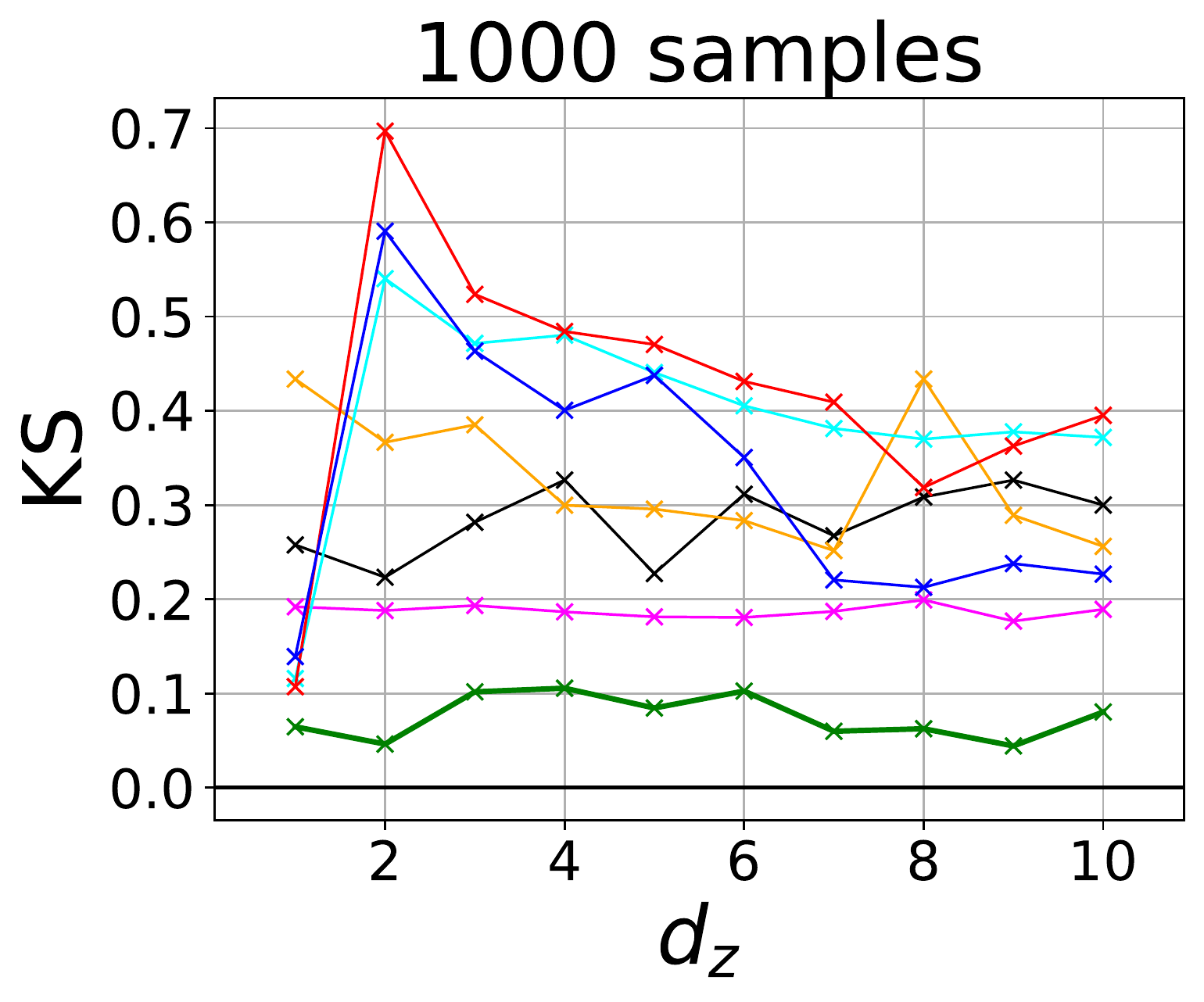}& \includegraphics[height=2.9cm]{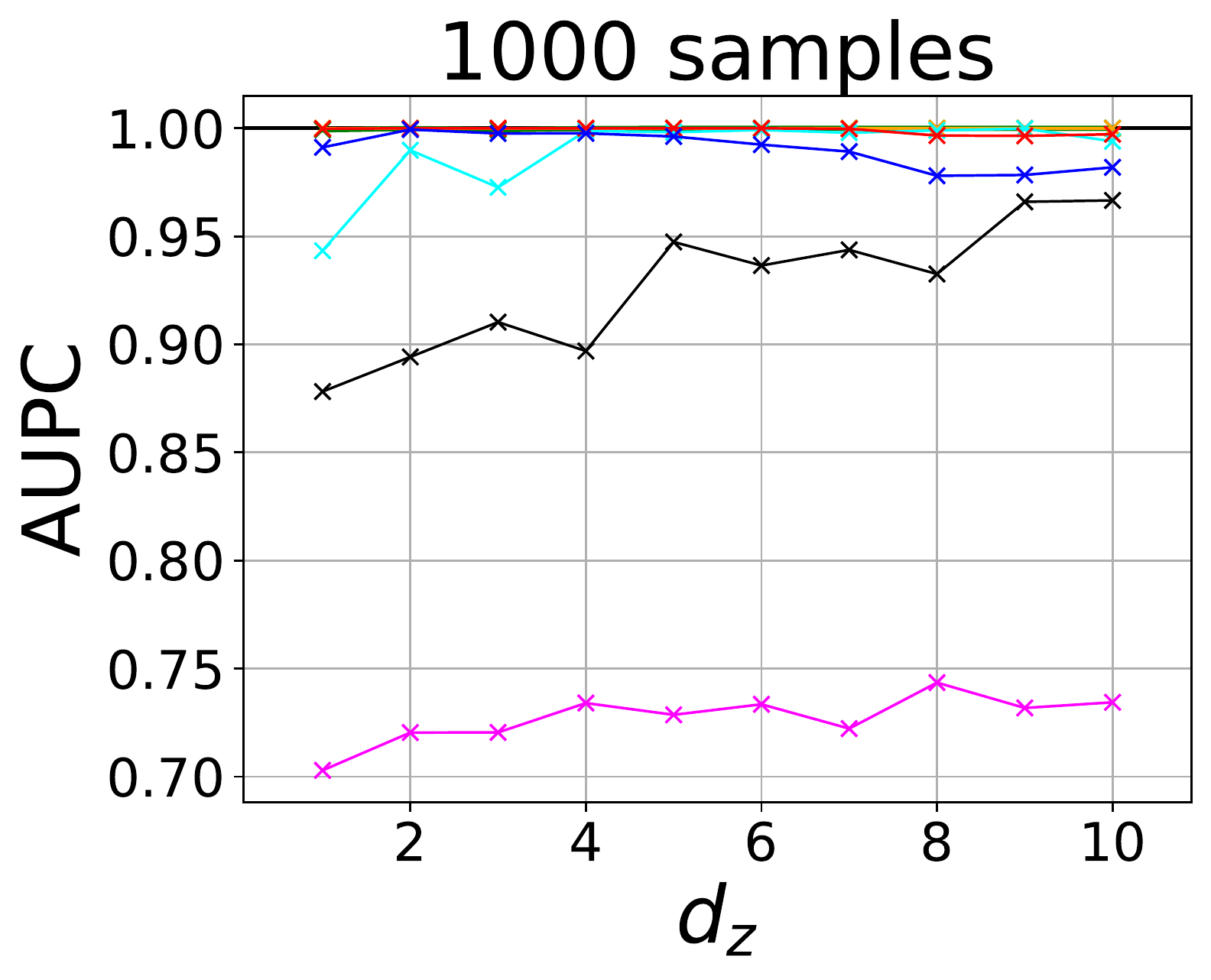} & 
\includegraphics[height=2.9cm]{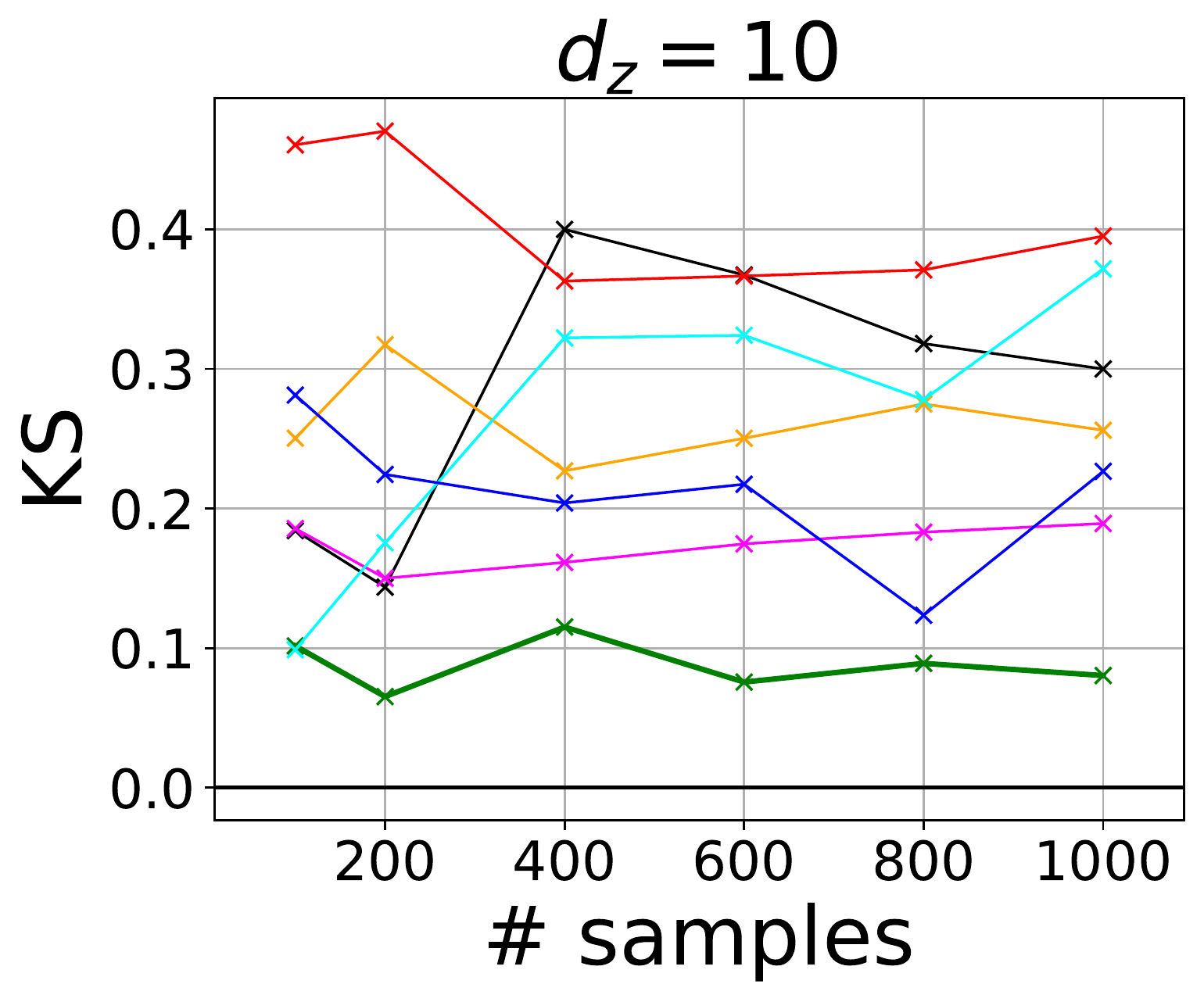}& \includegraphics[height=2.9cm]{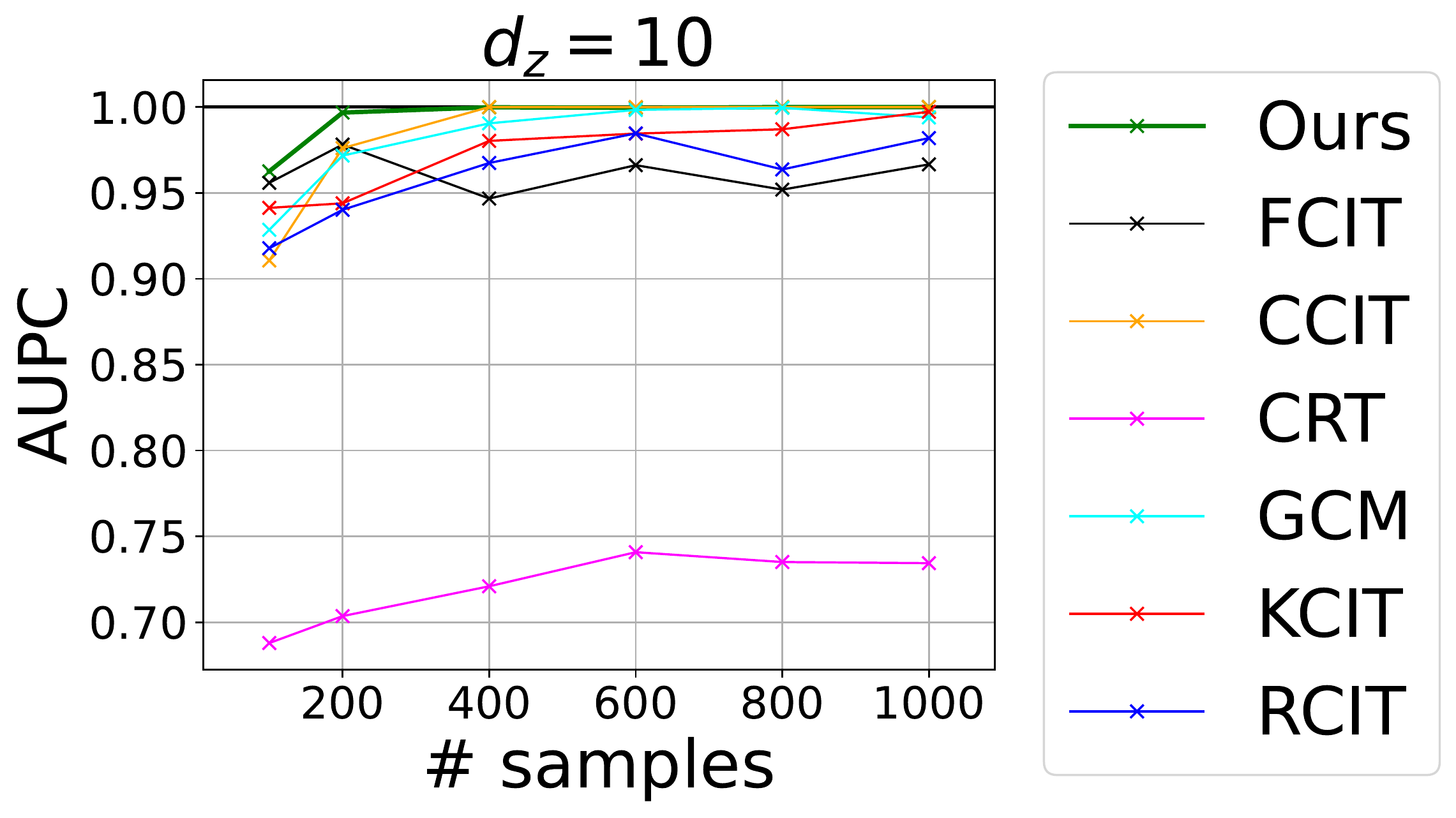}
\end{tabular}
\caption{Comparison of the KS statistic and the AUPC of our testing procedure with other SoTA tests on the two problems presented in Eq.~\eqref{li-exp-h0} and Eq.~\eqref{li-exp-h1} with Gaussian noises. Each point in the figures is obtained by repeating the experiment for 100 independent trials. (\emph{Left, middle-left}): the KS statistic and AUPC (respectively) obtained by each test when varying the dimension $d_z$ from 1 to 10; here, the number of samples $n$ is fixed and equals to $1000$. (\emph{Middle-right, right}): the KS and AUPC (respectively), obtained by each test when varying the number of samples $n$ from 100 to 1000; here, the dimension $d_z$ is fixed and equals to $10$.
\label{fig-exp-li-ks-gauss-supp}}
\vspace{-0.3cm}
\end{figure*}

\begin{figure*}[ht]
\begin{tabular}{cccc} 
\includegraphics[height=2.9cm]{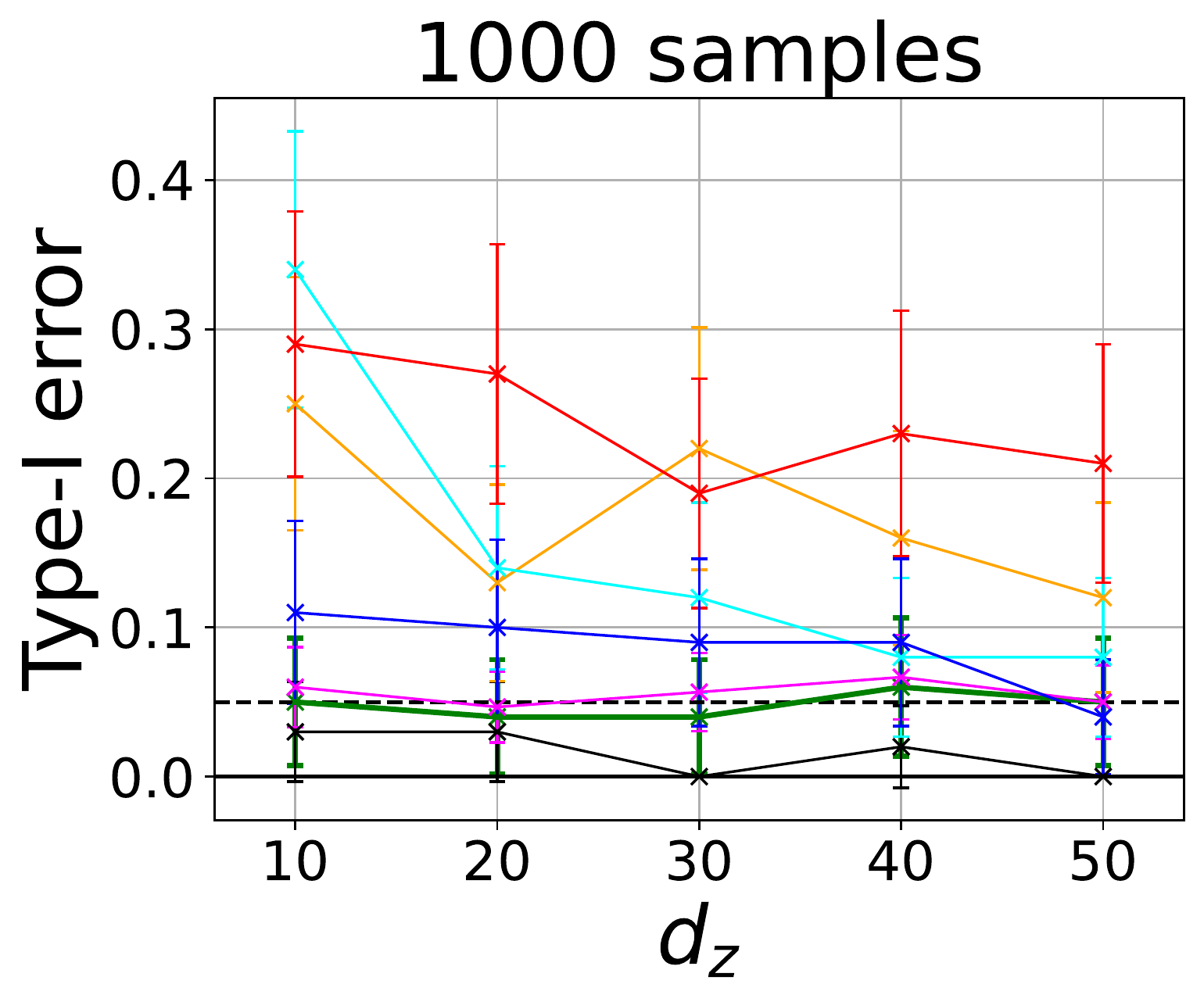}& \includegraphics[height=2.9cm]{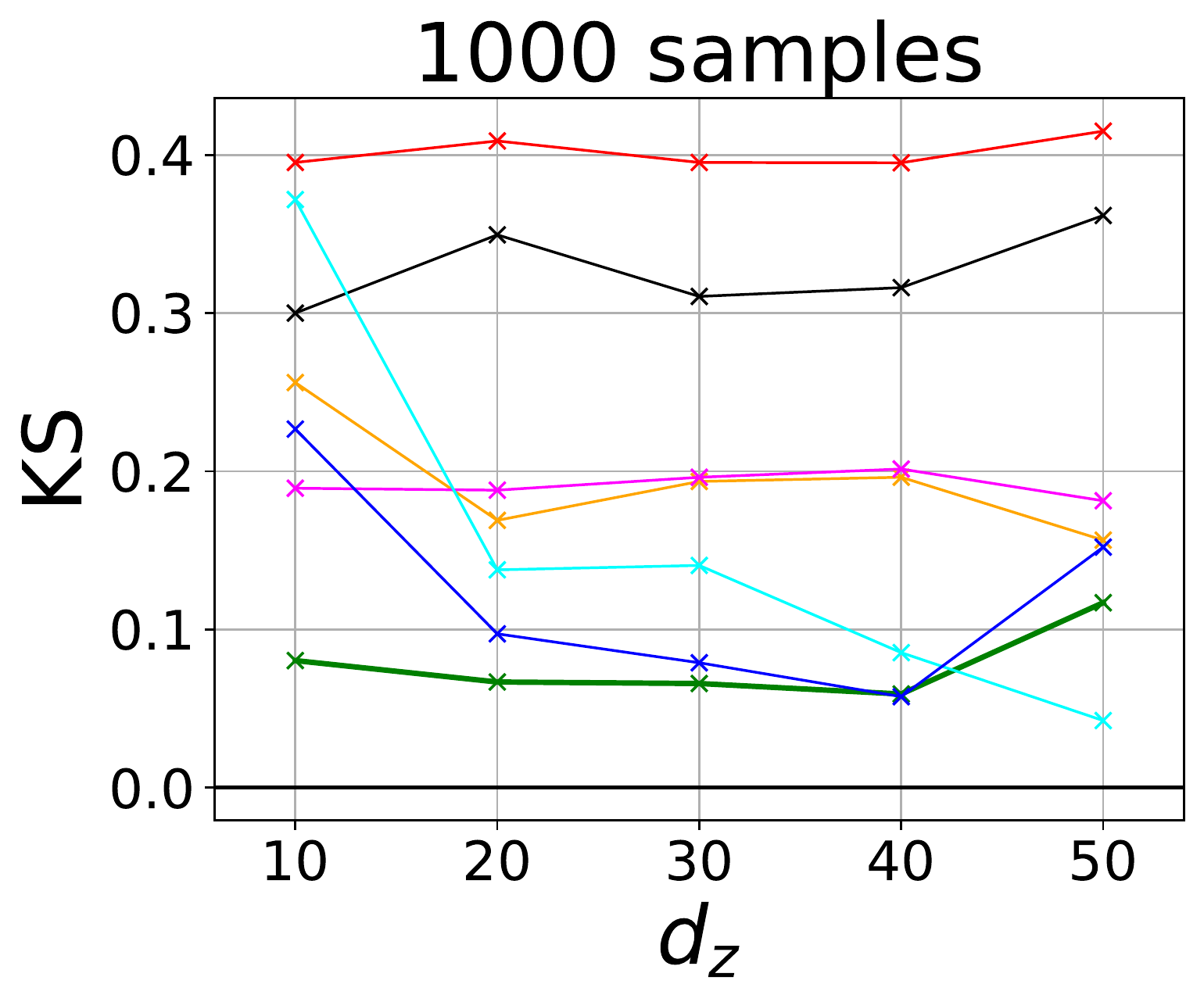} & 
\includegraphics[height=2.9cm]{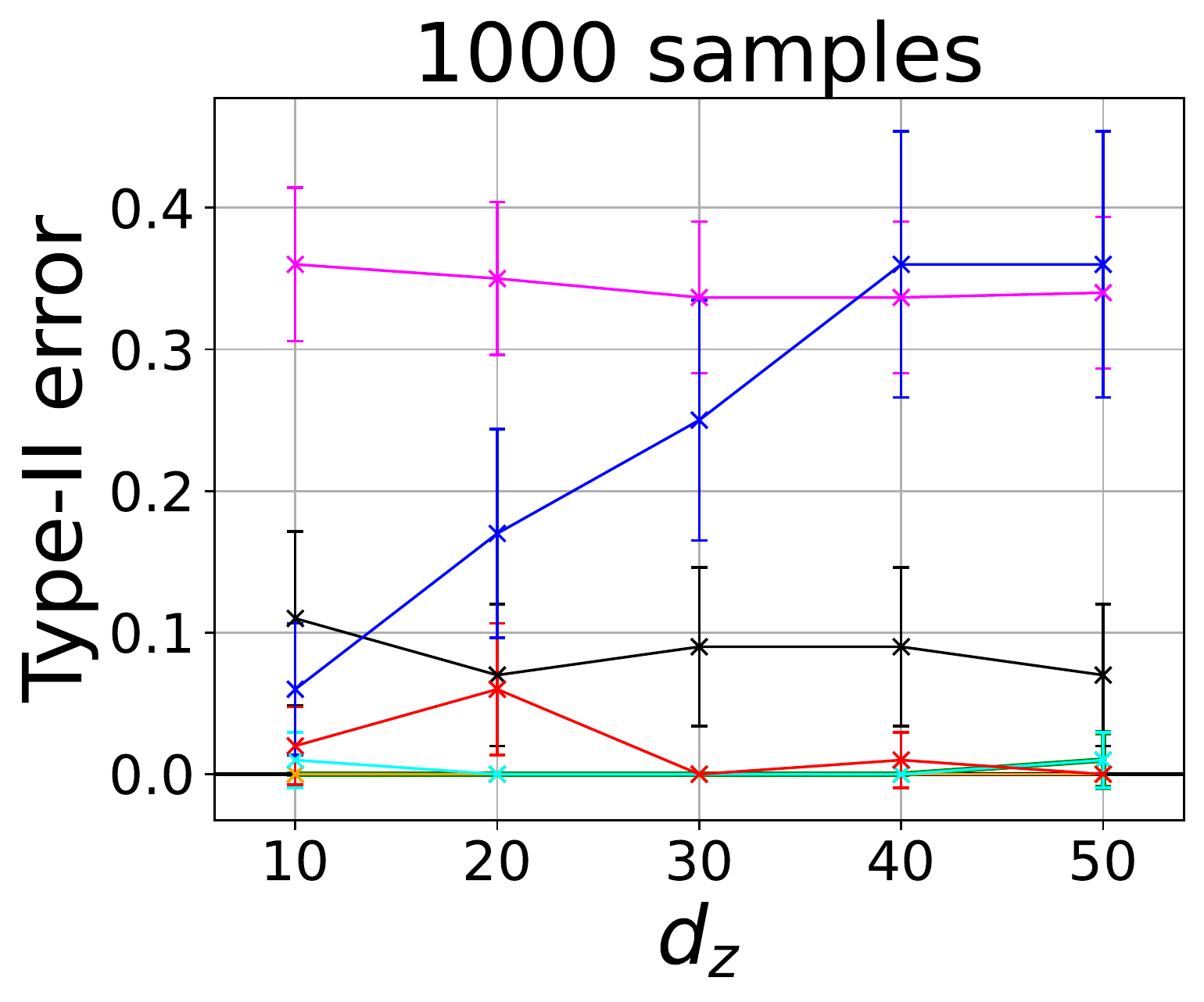}& \includegraphics[height=2.9cm]{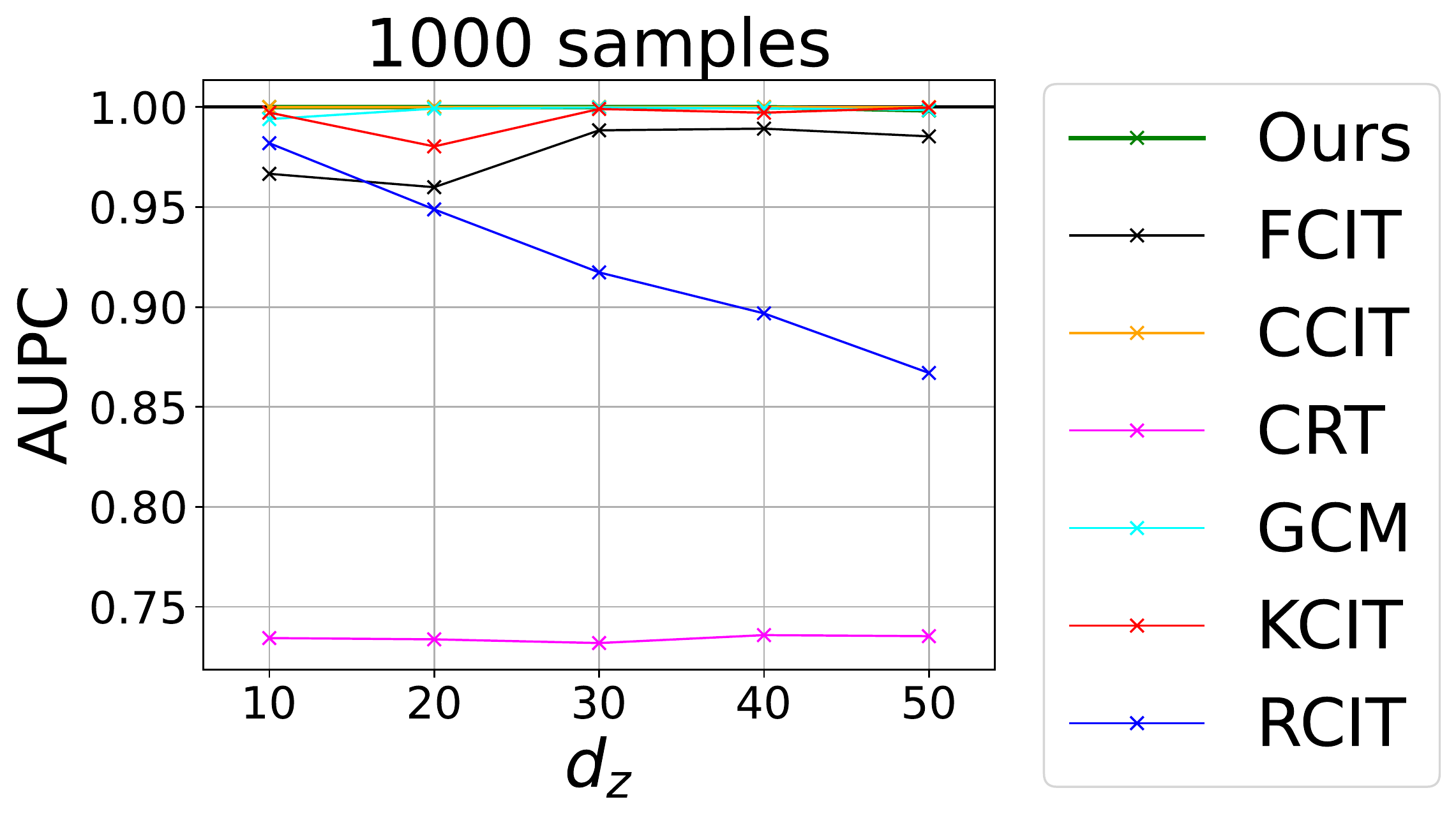}
\end{tabular}
\caption{Comparison of the type-I error at level $\alpha=0.05$ (dashed line), type-II error (lower is better), KS statistic and the AUPC of our testing procedure with other SoTA tests on the two problems presented in Eq.~\eqref{li-exp-h0} and Eq.~\eqref{li-exp-h1} with Gaussian noises. Each point in the figures is obtained by repeating the experiment for 100 independent trials. In each plot the dimension $d_z$ is varying from 10 to 50; here, the number of samples $n$ is fixed and equals to $1000$. 
\label{fig-exp-li-highdim-gauss-supp}}
\vspace{-0.3cm}
\end{figure*}

\newpage
\subsubsection{Laplace Case}

\begin{figure*}[h]
\begin{tabular}{cccc} 
\includegraphics[height=2.9cm]{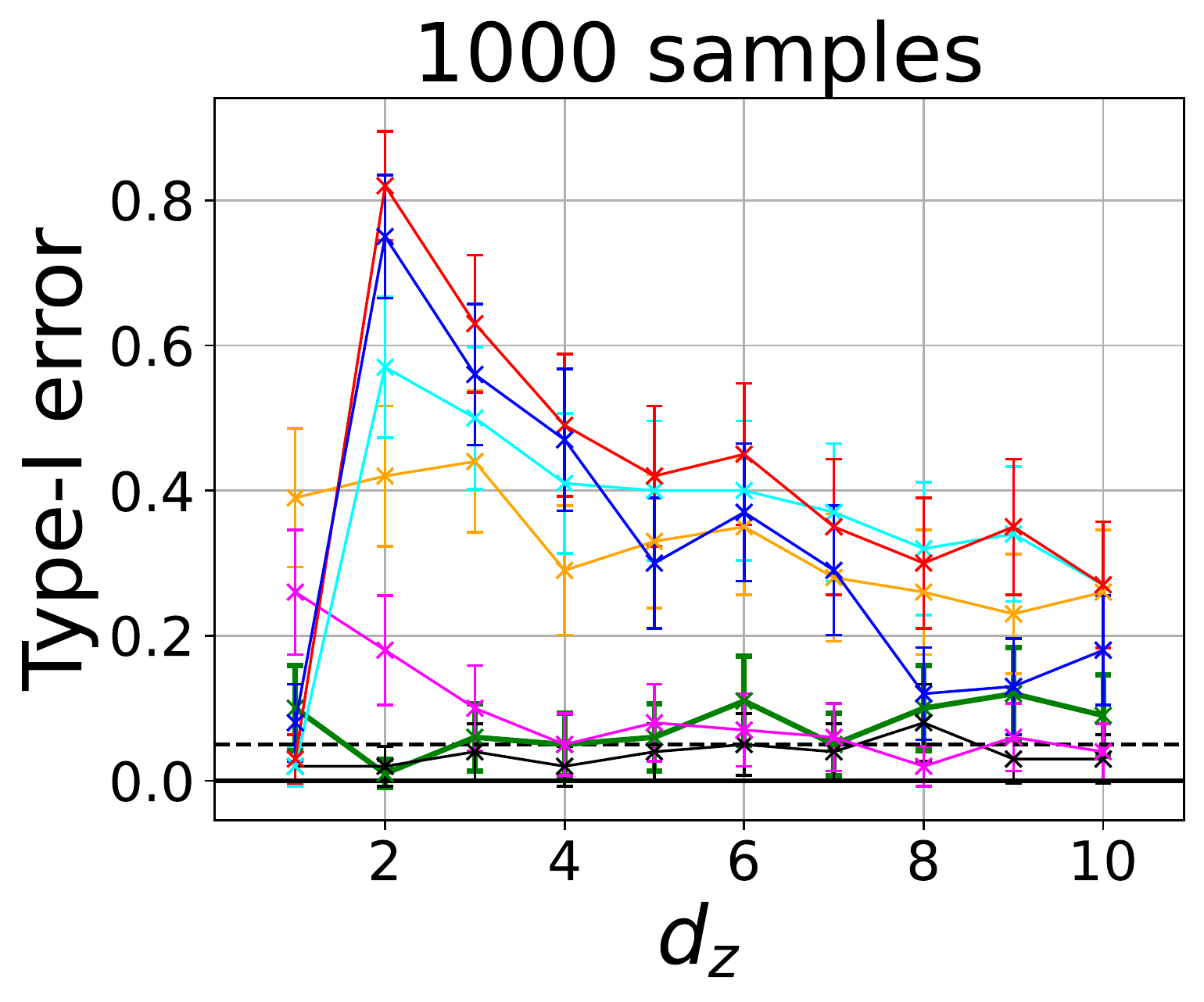}& \includegraphics[height=2.9cm]{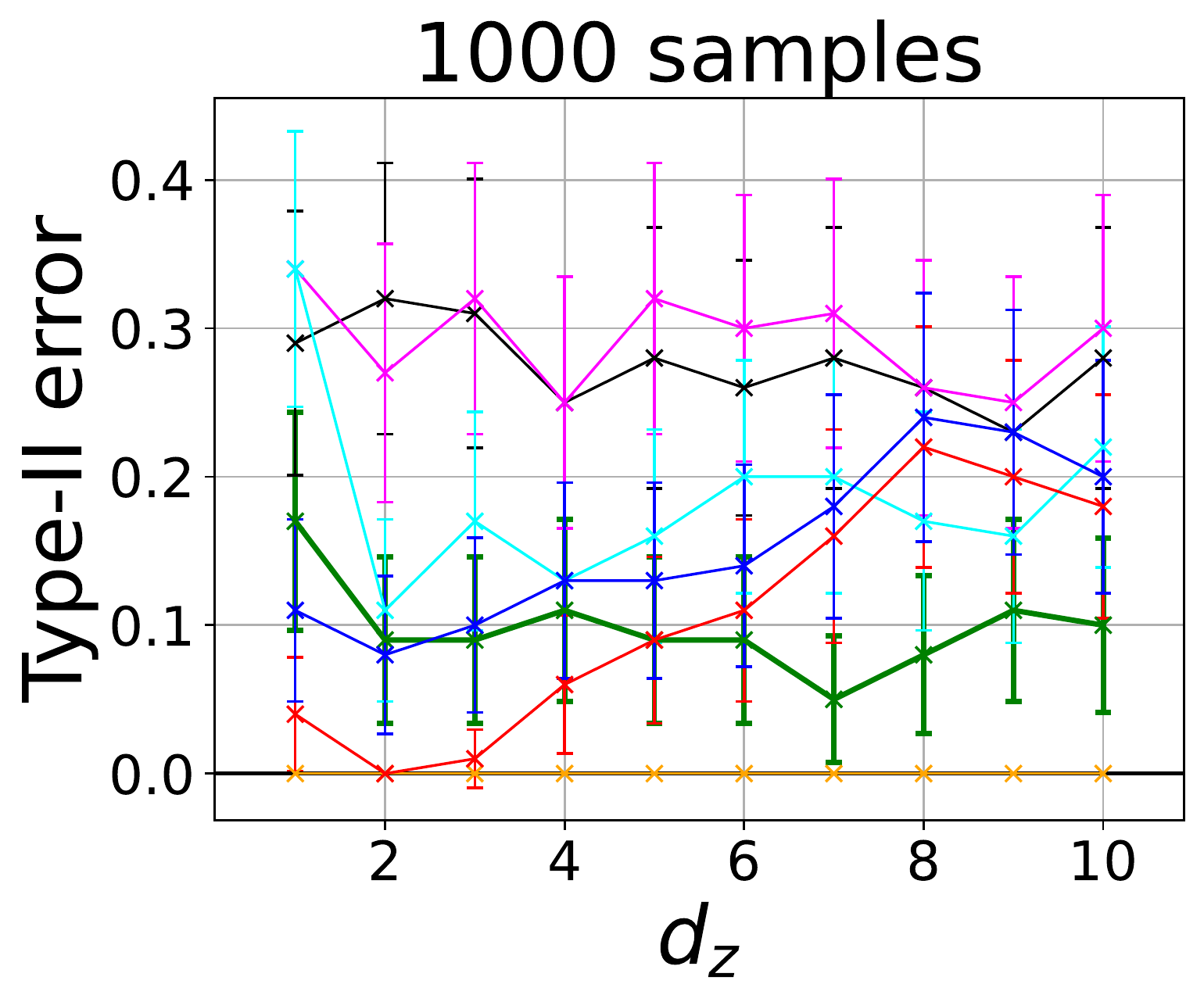} & 
\includegraphics[height=2.9cm]{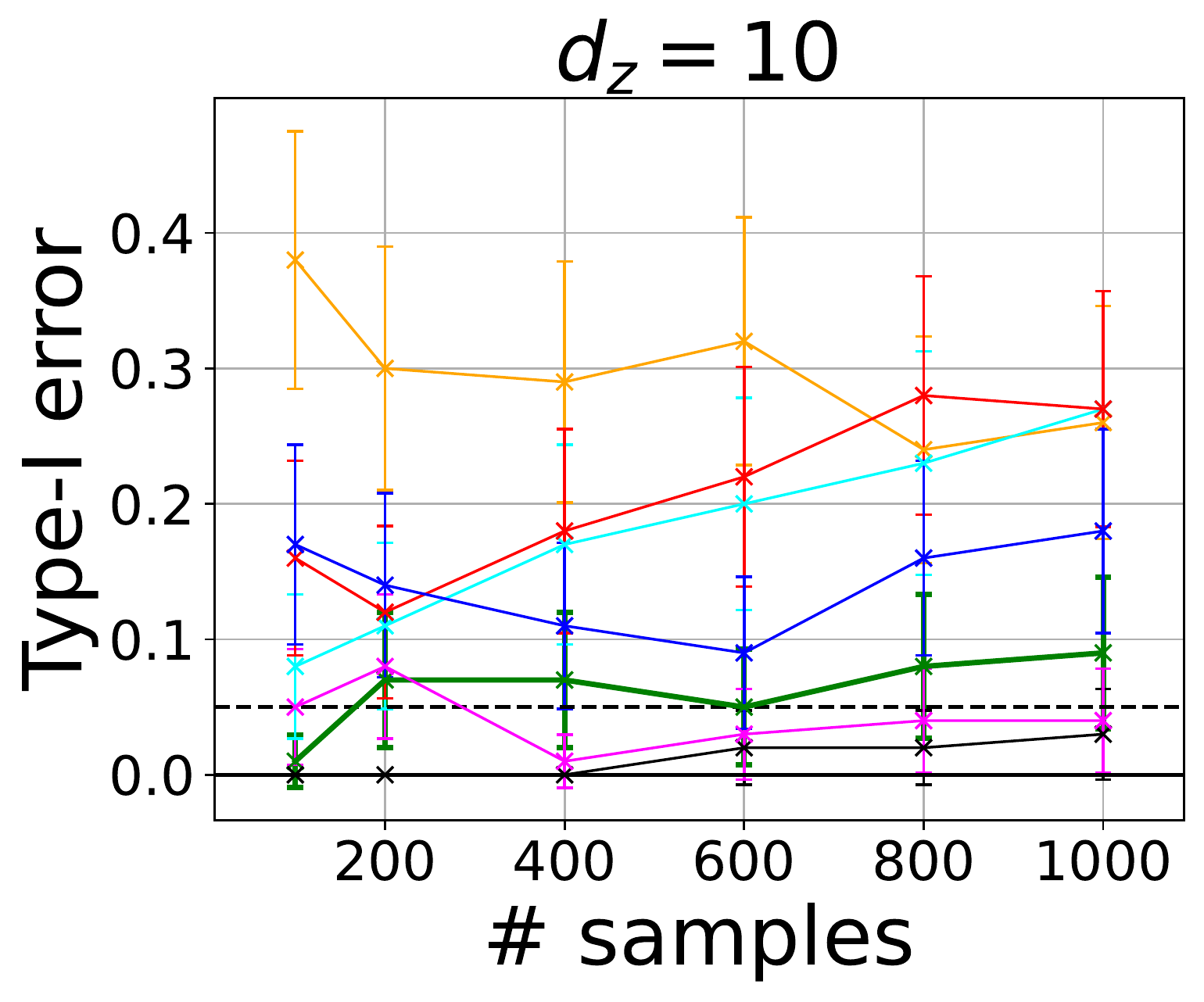}& \includegraphics[height=2.9cm]{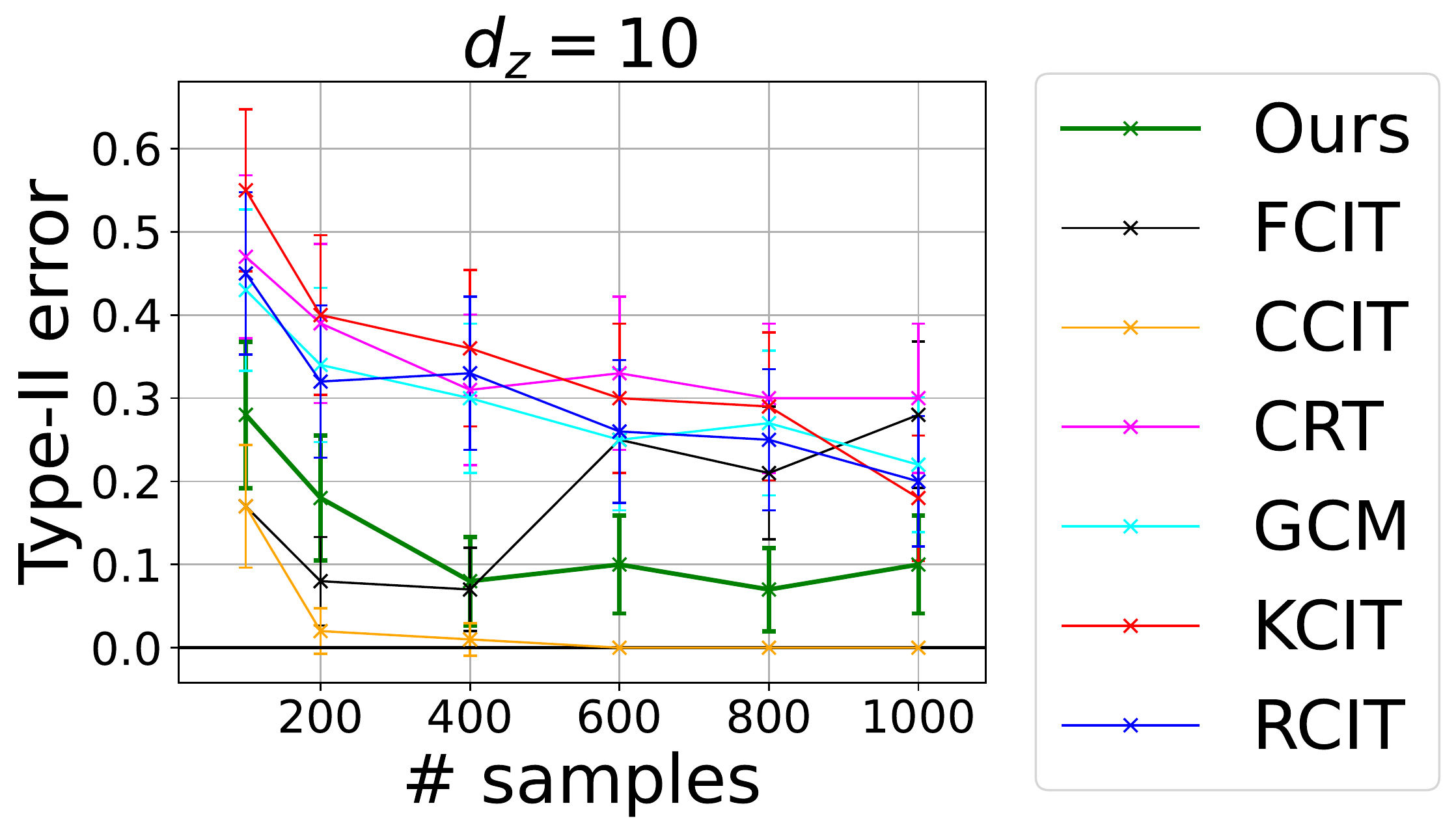} 
\end{tabular}
\caption{Comparison of the type-I error at level $\alpha=0.05$ (dashed line) and the type-II error (lower is better) of our test procedure with other SoTA tests on the two problems presented in~\eqref{li-exp-h0} and~\eqref{li-exp-h1}  with Laplace noises. Each point in the figures is obtained by repeating the experiment for 100 independent trials. (\emph{Left, middle-left}): type-I and type-II errors obtained by each test when varying the dimension $d_z$ from 1 to 10; here, the number of samples $n$ is fixed and equals to $1000$. (\emph{Middle-right, right}): type-I and type-II errors obtained by each test when varying the number of samples $n$ from 100 to 1000; here, the dimension $d_z$ is fixed and equals to $10$. 
\label{fig-exp-li-laplace-supp}}
\vspace{-0.3cm}
\end{figure*}

\begin{figure*}[h]
\begin{tabular}{cccc} 
\includegraphics[height=2.9cm]{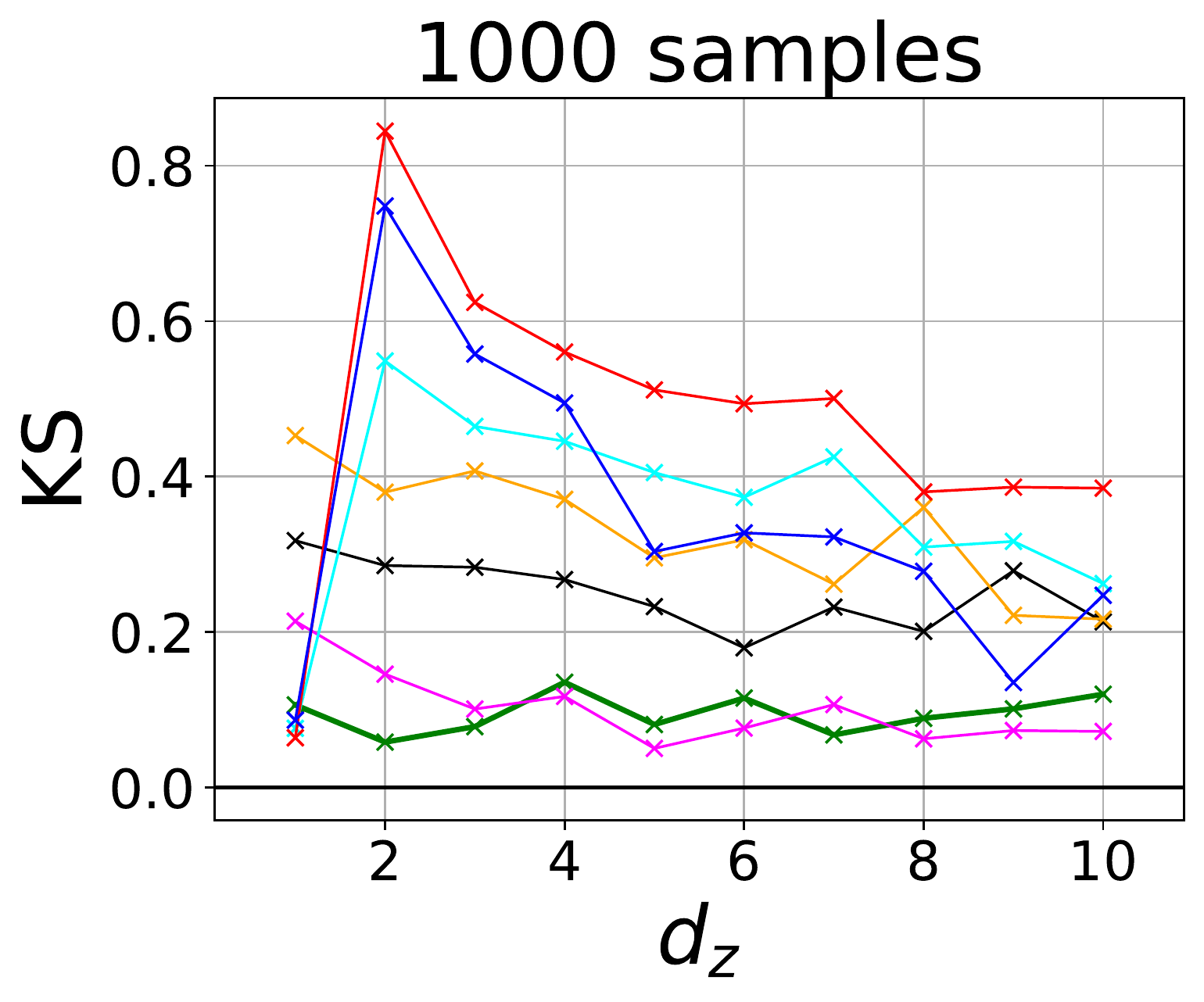}& \includegraphics[height=2.9cm]{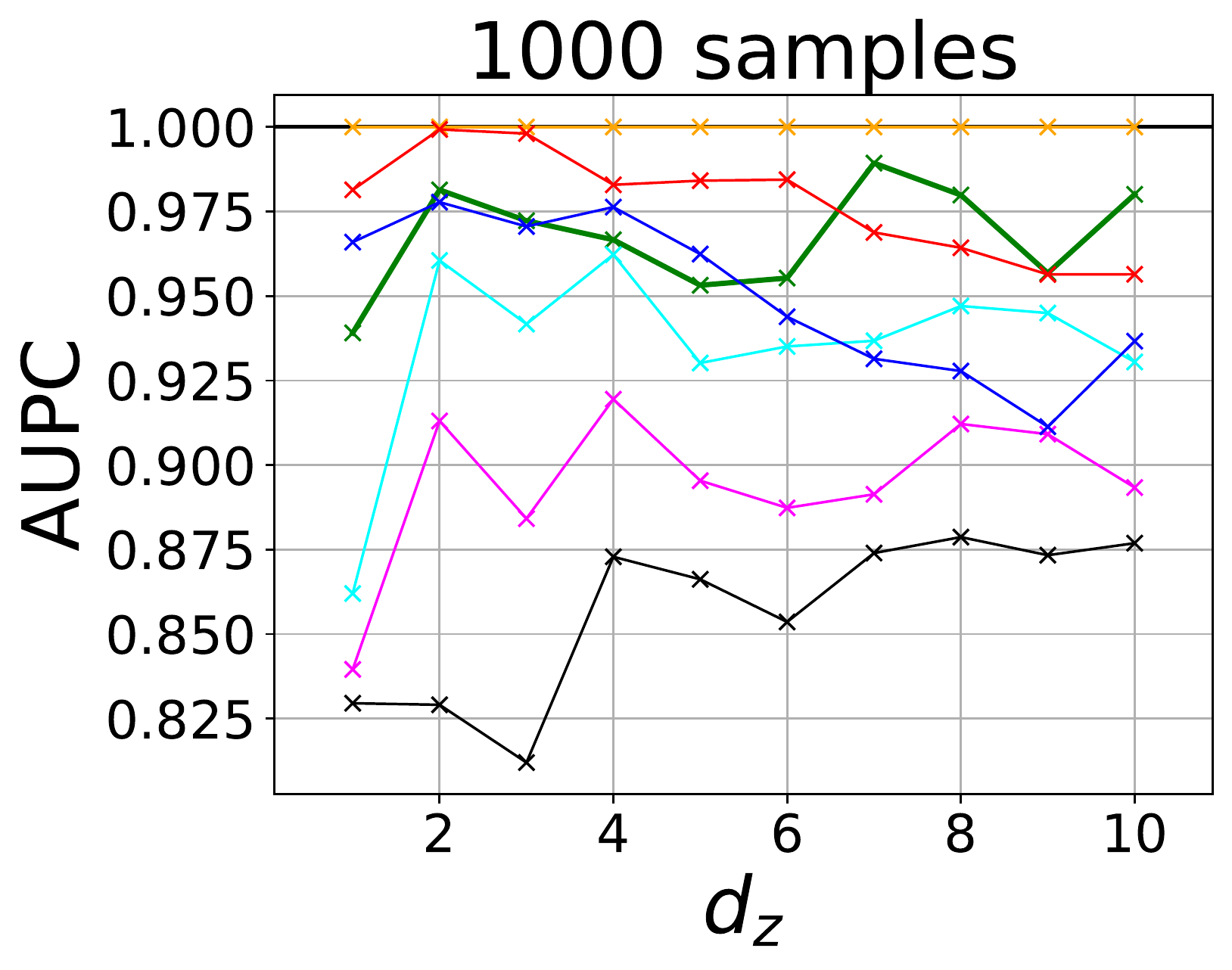} & 
\includegraphics[height=2.9cm]{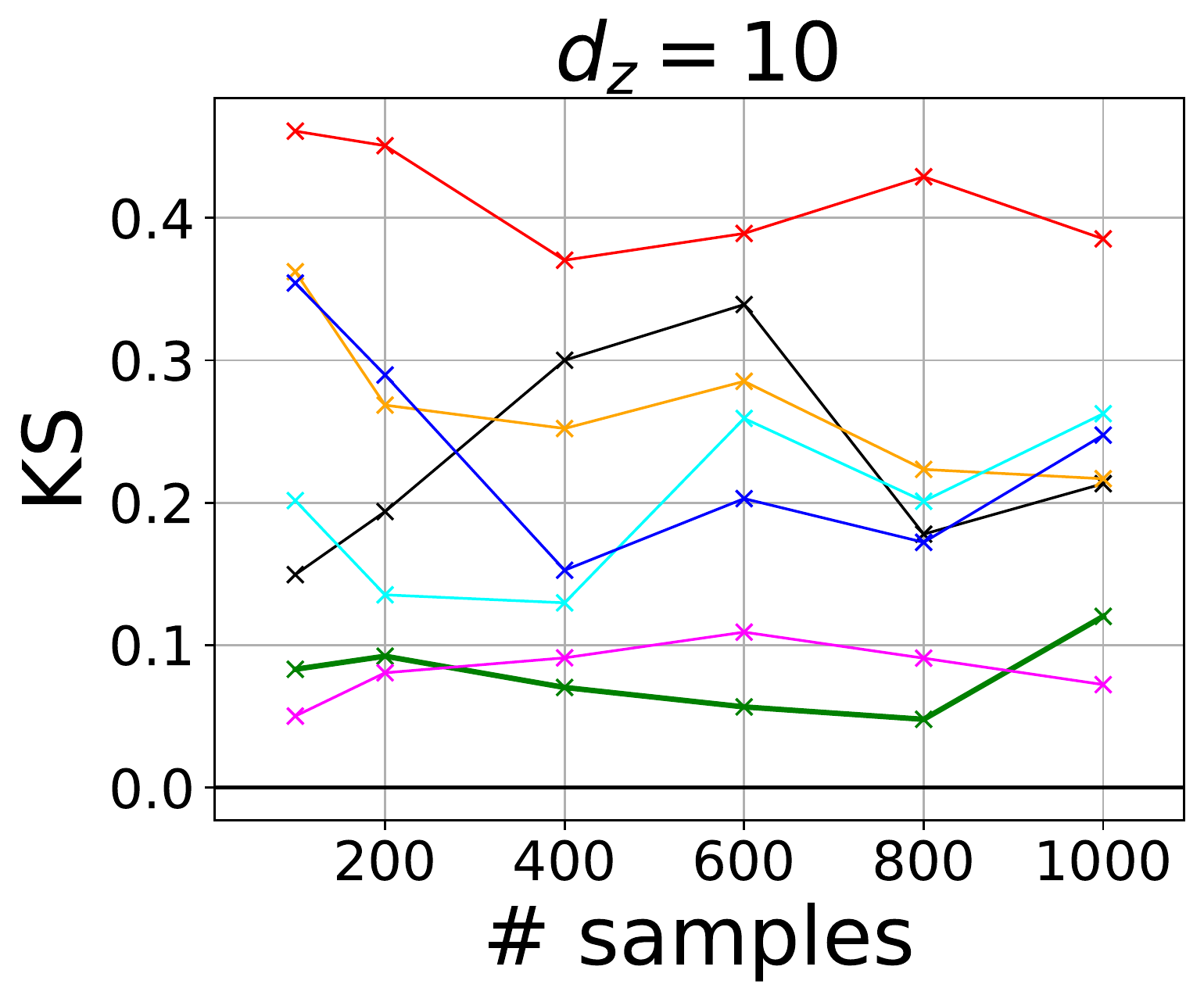}& \includegraphics[height=2.9cm]{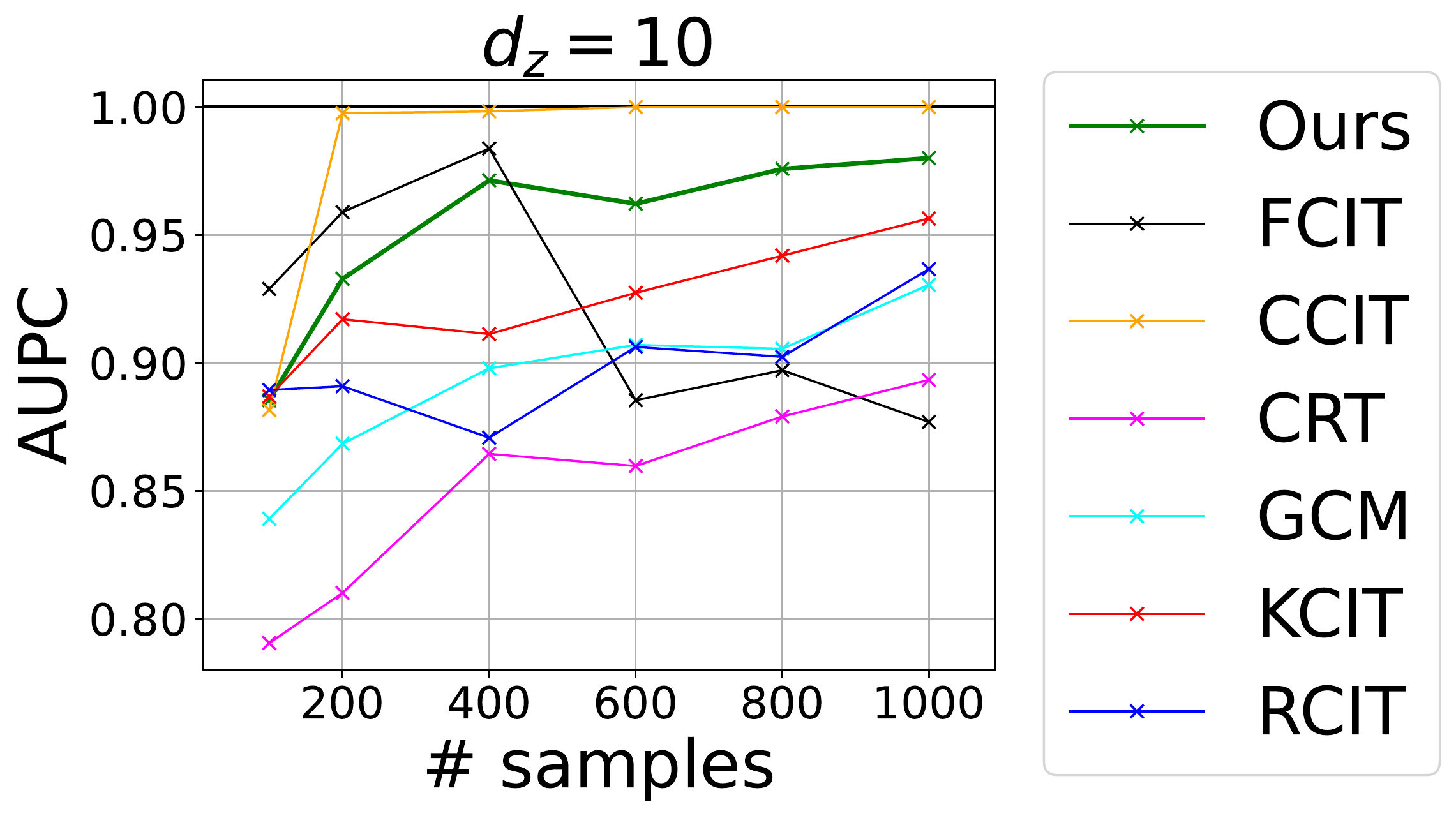} 
\end{tabular}
\caption{Comparison of the KS statistic and the AUPC of our testing procedure with other SoTA tests on the two problems presented in Eq.~\eqref{li-exp-h0} and Eq.~\eqref{li-exp-h1}  with Laplace noises. Each point in the figures is obtained by repeating the experiment for 100 independent trials. (\emph{Left, middle-left}): the KS statistic and AUPC (respectively) obtained by each test when varying the dimension $d_z$ from 1 to 10; here, the number of samples $n$ is fixed and equals to $1000$. (\emph{Middle-right, right}): the KS and AUPC (respectively), obtained by each test when varying the number of samples $n$ from 100 to 1000; here, the dimension $d_z$ is fixed and equals to $10$.
\label{fig-exp-li-ks-laplace-supp}}
\vspace{-0.3cm}
\end{figure*}

\begin{figure*}[h]
\begin{tabular}{cccc} 
\includegraphics[height=2.9cm]{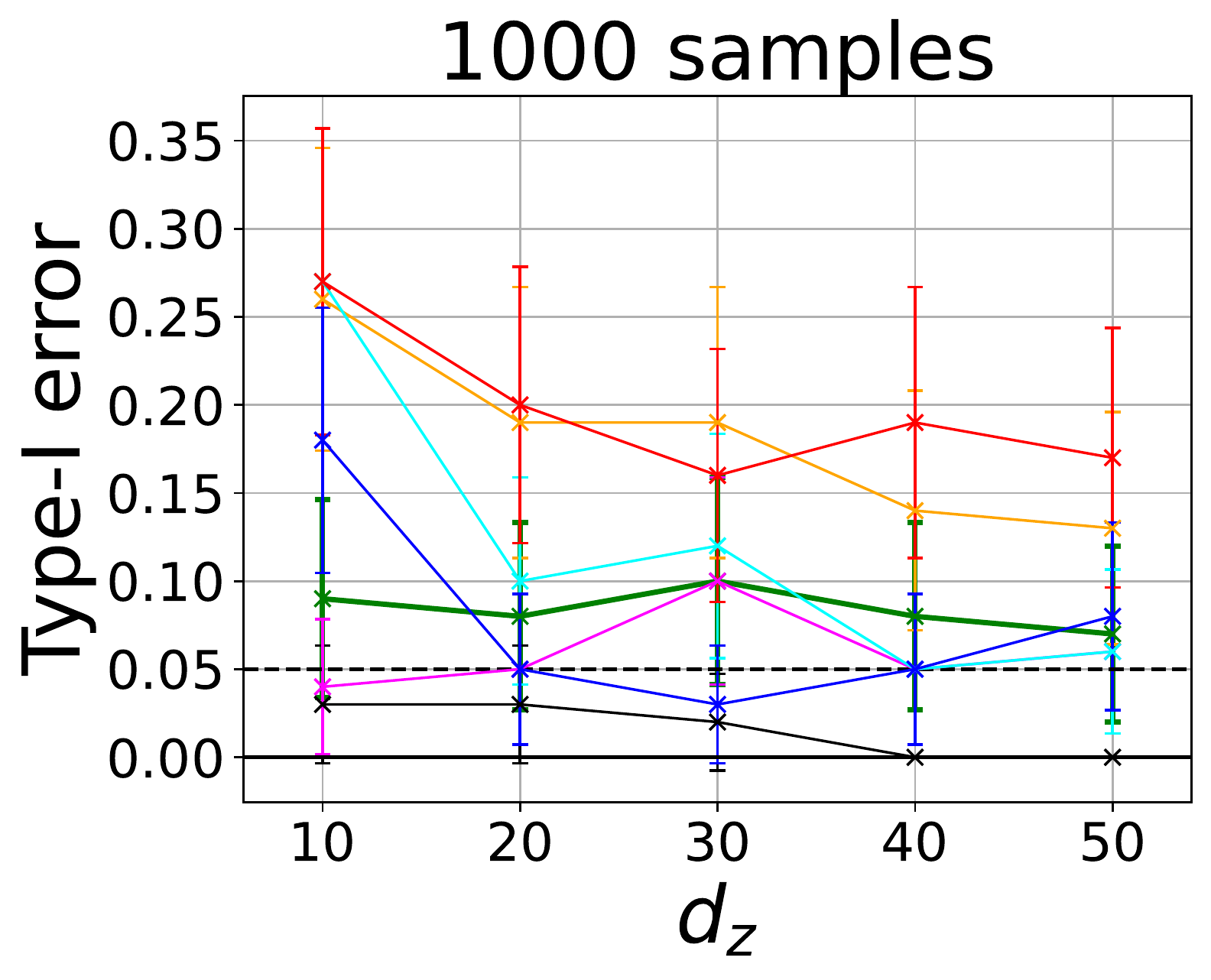}& \includegraphics[height=2.9cm]{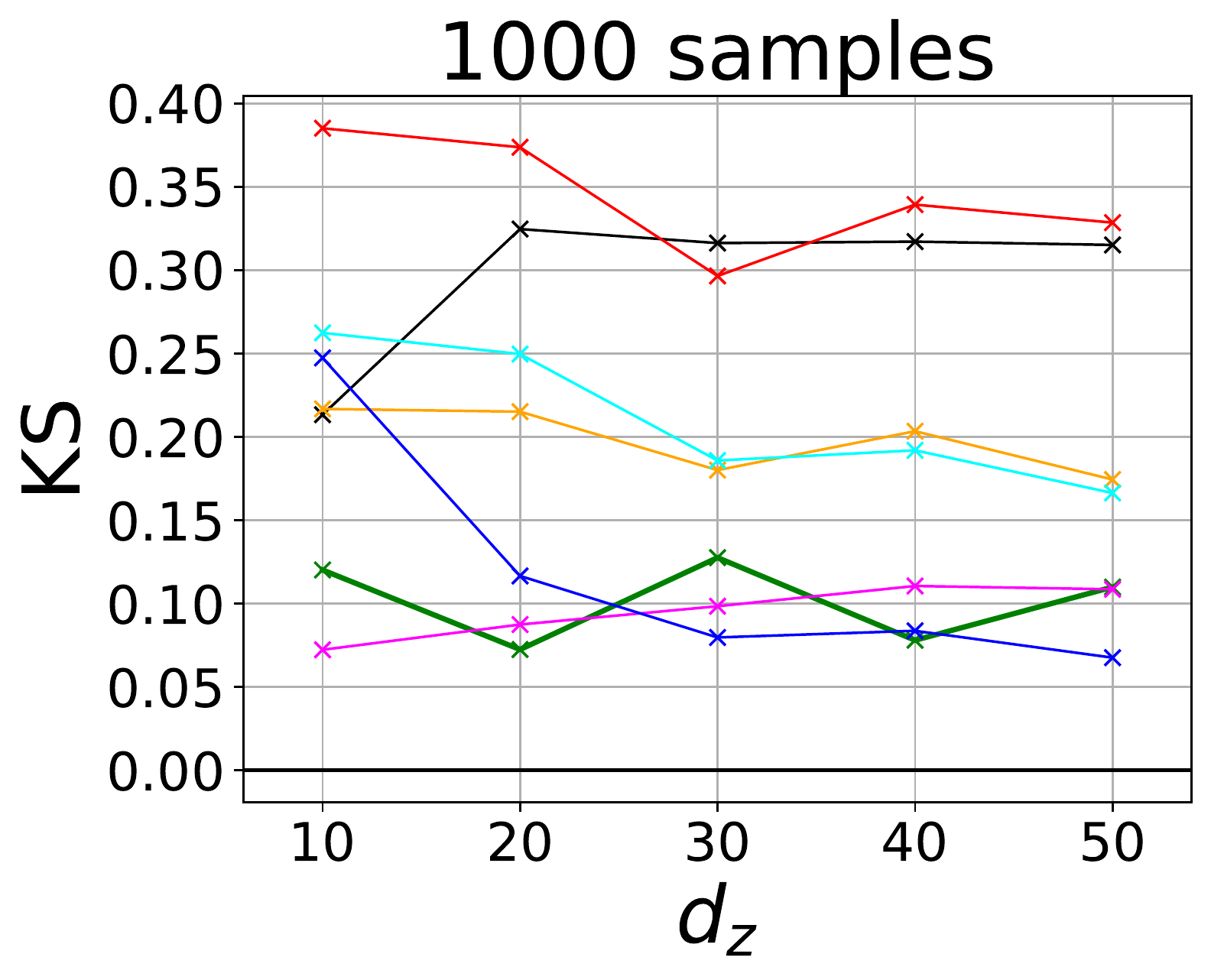} & 
\includegraphics[height=2.9cm]{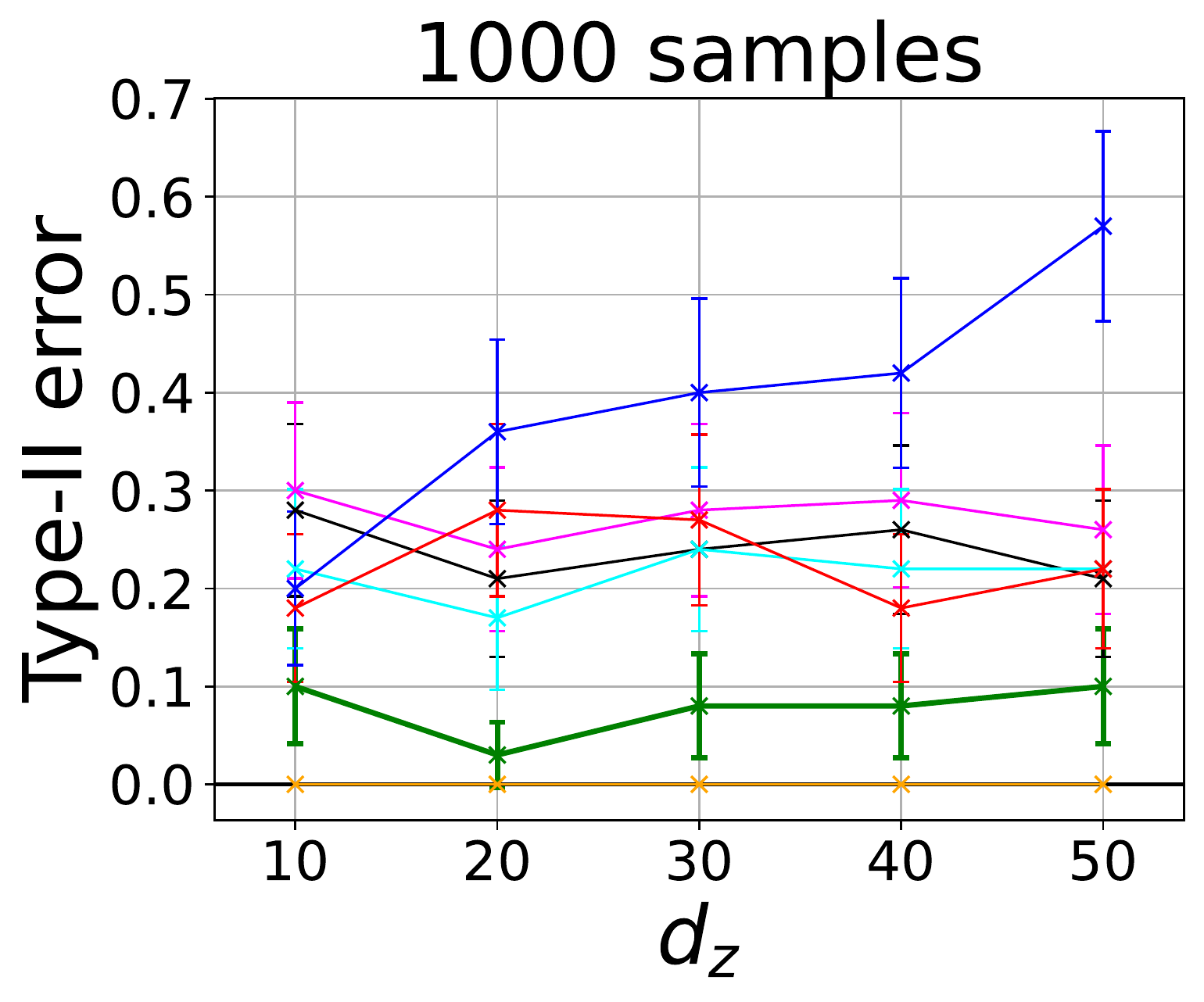}& \includegraphics[height=2.9cm]{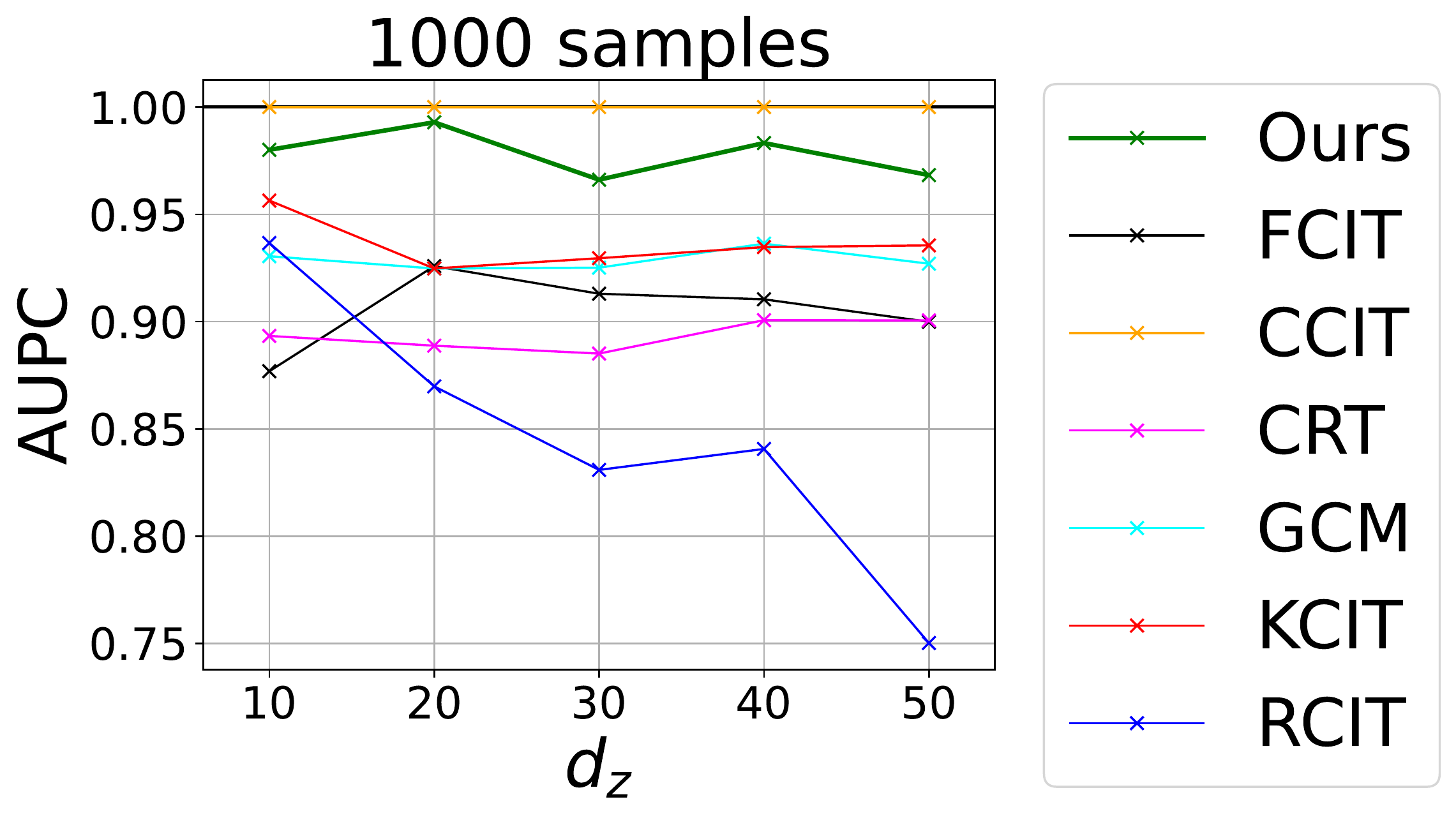}
\end{tabular}
\caption{Comparison of the type-I error at level $\alpha=0.05$ (dashed line), type-II error (lower is better), KS statistic and the AUPC of our testing procedure with other SoTA tests on the two problems presented in Eq.~\eqref{li-exp-h0} and Eq.~\eqref{li-exp-h1} with Laplace noises. Each point in the figures is obtained by repeating the experiment for 100 independent trials. In each plot the dimension $d_z$ is varying from 10 to 50; here, the number of samples $n$ is fixed and equals to $1000$. 
\label{fig-exp-li-highdim-laplace-supp}}
\vspace{-0.3cm}
\end{figure*}

\end{document}